\documentclass{article}

\usepackage[preprint]{corl_2026} %

\usepackage[utf8]{inputenc}

\usepackage{amsmath}
\usepackage{amsthm}
\usepackage{amsfonts}
\usepackage{amssymb}
\usepackage{graphicx}
\usepackage{wrapfig}
\usepackage{array}
\usepackage[dvipsnames]{xcolor}
\usepackage[table]{xcolor}
\usepackage{subcaption}
\usepackage{mathrsfs}
\usepackage{multicol}
\usepackage{multirow}
\usepackage{arydshln}
\usepackage{diagbox}
\usepackage{algorithm}
\usepackage{algpseudocode}
\usepackage{float}
\usepackage{caption}
\usepackage{cancel}
    
\usepackage{duckuments}
\usepackage{amsmath}
\usepackage{tabularx}
\usepackage{booktabs}
\usepackage{makecell}

\usepackage{colortbl}
\usepackage{tikz}
\usepackage{graphicx}
\usepackage{multicol}
\usepackage{rotating}
\usepackage{enumitem}

\definecolor{rowgray}{gray}{0.92}

        \newcommand{\reals}{\mathbb{R}}
        \newcommand{\R}{\reals}

        \newtheorem{theorem}{Theorem}
        
        \newtheorem{proposition}{Proposition}

    \usepackage{letltxmacro}
    \LetLtxMacro\orgvdots\vdots
    \LetLtxMacro\orgddots\ddots

    \makeatletter
    \DeclareRobustCommand\vdots{%
        \mathpalette\@vdots{}%
    }
    \newcommand*{\@vdots}[2]{%
        \sbox0{$#1\cdotp\cdotp\cdotp\m@th$}%
        \sbox2{$#1.\m@th$}%
        \vbox{%
            \dimen@=\wd0 %
            \advance\dimen@ -3\ht2 %
            \kern.5\dimen@
            \dimen@=\wd2 %
            \advance\dimen@ -\ht2 %
            \dimen2=\wd0 %
            \advance\dimen2 -\dimen@
            \vbox to \dimen2{%
                \offinterlineskip
                \copy2 \vfill\copy2 \vfill\copy2 %
            }%
        }%
    }
    \DeclareRobustCommand\ddots{%
        \mathinner{%
            \mathpalette\@ddots{}%
            \mkern\thinmuskip
        }%
    }
    \newcommand*{\@ddots}[2]{%
        \sbox0{$#1\cdotp\cdotp\cdotp\m@th$}%
        \sbox2{$#1.\m@th$}%
        \vbox{%
            \dimen@=\wd0 %
            \advance\dimen@ -3\ht2 %
            \kern.5\dimen@
            \dimen@=\wd2 %
            \advance\dimen@ -\ht2 %
            \dimen2=\wd0 %
            \advance\dimen2 -\dimen@
            \vbox to \dimen2{%
                \offinterlineskip
                \hbox{$#1\mathpunct{.}\m@th$}%
                \vfill
                \hbox{$#1\mathpunct{\kern\wd2}\mathpunct{.}\m@th$}%
                \vfill
                \hbox{$#1\mathpunct{\kern\wd2}\mathpunct{\kern\wd2}\mathpunct{.}\m@th$}%
            }%
        }%
    }
    \makeatother

\usepackage{cleveref}
\Crefname{figure}{Fig.}{Figs.}
\Crefname{equation}{Eq.}{Eqs.}
\Crefname{lemma}{Lemma}{Lemmata}
\Crefname{proposition}{Proposition}{Propositions}
\Crefname{assumption}{Assumption}{Assumptions}
\Crefname{theorem}{Theorem}{Theorems}
\Crefname{section}{Section}{Sections}
\Crefname{subsection}{Subsection}{Subsections}
\Crefname{appendix}{Appendix}{Appendices}
\Crefname{corollary}{Corollary}{Corollaries}

\usepackage{tcolorbox}
\tcbuselibrary{skins}

\DeclareMathOperator{\SE}{SE}

\newcolumntype{Y}{>{\centering\arraybackslash}X}

\newif\ifshowcomments
\showcommentstrue

\newcommand{\addauthor}[2]{%
  \expandafter\newcommand\csname #1\endcsname[1]{%
    \ifshowcomments
      {\color{#2}[#1: ##1]}%
    \fi
  }%
}

\usepackage{amsthm}
\theoremstyle{plain}

\theoremstyle{definition}
\newtheorem{definition}{Definition}
\theoremstyle{remark}
\newtheorem{remark}{Remark}

\addauthor{adam}{blue}
\addauthor{giannis}{orange}
\addauthor{costis}{purple}
\addauthor{russ}{red}
\addauthor{tommy}{cyan}
\addauthor{nicholas}{purple}
\addauthor{kerem}{pink}

\title{Ambient Diffusion Policy:\\Imitation Learning from Suboptimal Data in Robotics}

\author{
  Adam Wei\thanks{Corresponding author: \texttt{weiadam@mit.edu}. \quad $^{\dagger}$Equal contribution.} \quad
  Nicholas Pfaff$^{\dagger}$ \quad
  Thomas Cohn$^{\dagger}$ \quad
  Arif Kerem Day\i{}\\[0.25em]
  \textbf{Constantinos Daskalakis \quad Giannis Daras \quad Russ Tedrake}\\[0.5em]
  \url{https://ambient-diffusion-policy.github.io/} \\[0.5em]
  MIT
}

\begin{document}
\maketitle

\addtocontents{toc}{\protect\setcounter{tocdepth}{-1}}

\vspace{-1.5em} %
\begin{figure}[H]
    \centering
    \captionsetup{font=small}
    \includegraphics[width=\linewidth]{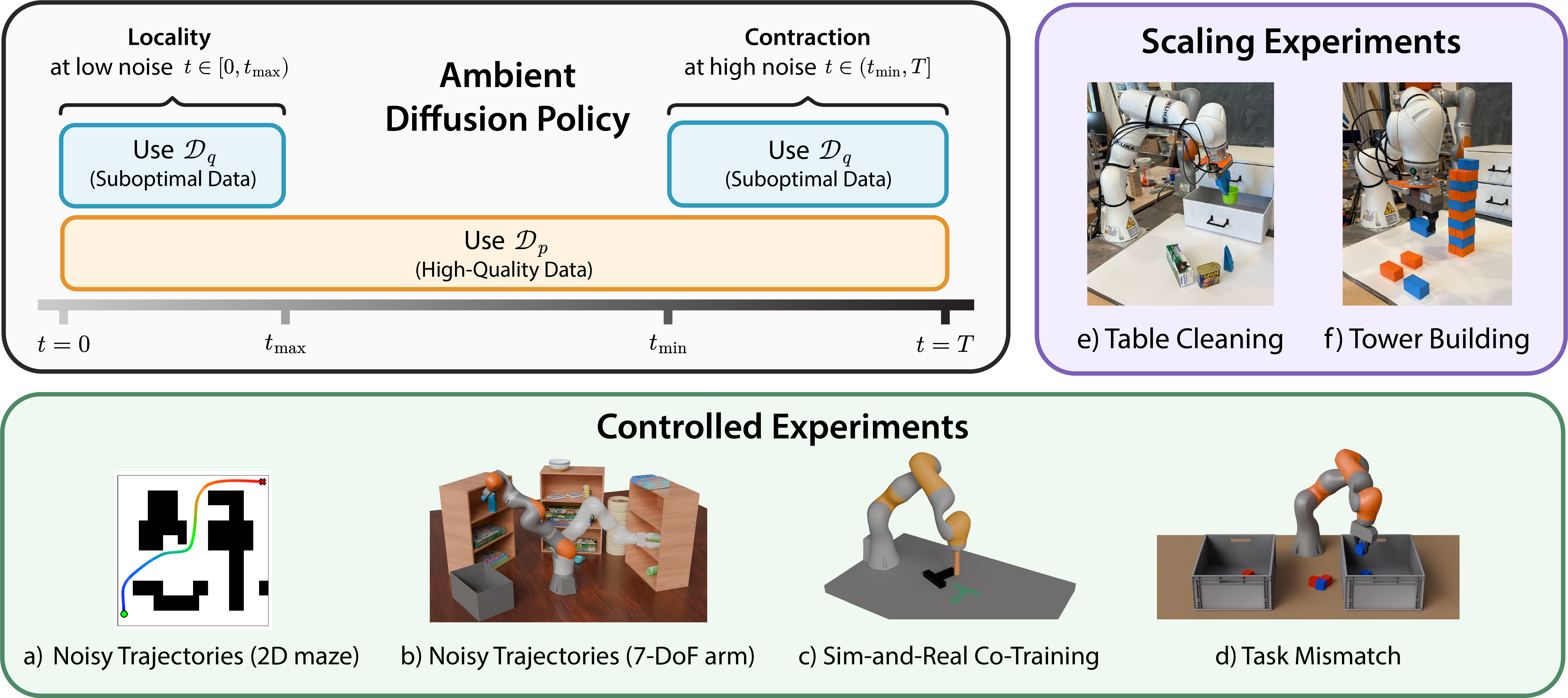}
    \caption{\textbf{Training Algorithm for Ambient Diffusion Policy.} High-quality data can be used during training at all diffusion times. 
    Suboptimal data is restricted to training at low or high diffusion times $t\in[0, t_{\max})\cup(t_{\min}, T]$. \textbf{Task visualizations.} We evaluate Ambient Diffusion Policy on four types of action suboptimality, six different tasks, and scale to the Open X-Embodiment (OXE)~\cite{open_x_embodiment_rt_x_2023}.}
    \label{fig:teaser}
\end{figure}

\vspace{-0.25em} %
\begin{abstract}
\looseness=-1
We propose Ambient Diffusion Policy, a simple and principled method for imitation learning from suboptimal data in robotics.
High-quality, task-specific robot data is expensive and time-consuming to collect, while suboptimal datasets with lower-quality or out-of-distribution demonstrations are abundant.
Existing methods that co-train on both data sources in robotics often fail to separate the meaningful and the harmful features in the suboptimal samples.
In contrast, our method extracts only the useful features by introducing a new axis to co-training in robotics: \emph{noise-dependent data usage}.
Ambient Diffusion Policy restricts the contribution of suboptimal data during training to only the high and low diffusion times.
To rigorously justify our approach, we first observe that robot action data exhibits a \textit{spectral power law}. This induces two important properties on the optimal Diffusion Policy that we exploit: a \textit{global-to-local hierarchy} and \textit{locality}.
We theoretically formalize this discussion using a simplified model.
Our experiments validate Ambient Diffusion Policy on four types of suboptimal action data (noisy trajectories, sim-to-real gap, task mismatch, and large-scale data mixtures) across six tasks.
The results show that it effectively learns from arbitrary sources of suboptimal data. Notably, it outperforms existing co-training baselines by up to 33\% when scaled to Open X-Embodiment---a large dataset with heterogeneous data quality and unstructured distribution shifts.
Overall, Ambient Diffusion Policy increases the utility of suboptimal demonstrations and expands the set of usable data sources in robotics.

\end{abstract}

\section{Introduction}
\label{sec:intro}
The training corpus of nearly every large-scale robot policy spans different tasks, real and simulated environments, embodiments, and even modalities~\cite{pi05, nvidia2025gr00tn1openfoundation, trilbm, kim2024openvla, open_x_embodiment_rt_x_2023, lin2026systematicstudydatamodalities, geminiroboticsteam2025geminiroboticsbringingai, agibotworldcontributors2025agibotworldcolosseolargescale, chi2024universalmanipulationinterfaceinthewild}. One reason for this data heterogeneity is that high-quality, task-specific robot data is expensive and time-consuming to collect; it requires skilled teleoperators and well-tuned low-level controllers. In contrast, suboptimal data is abundant. Any real-world data collection effort naturally produces failures and trajectories of differing quality~\cite{pi07, open_x_embodiment_rt_x_2023, agibotworldcontributors2025agibotworldcolosseolargescale, pi06star}. Out-of-distribution (OOD) data sources are also plentiful and widely available~\cite{open_x_embodiment_rt_x_2023, agibotworldcontributors2025agibotworldcolosseolargescale}. These include simulation~\cite{wei2025empirical, maddukuri2025sim}, cross-embodied data~\cite{open_x_embodiment_rt_x_2023}, and ego-centric video~\cite{grauman2024egoexo4dunderstandingskilledhuman, gao2026dreamdojogeneralistrobotworld}. 
Practitioners often draw from data sources of varying quality~\cite{open_x_embodiment_rt_x_2023, agibotworldcontributors2025agibotworldcolosseolargescale, kim2024openvla, hejna2024remixoptimizingdatamixtures, pi06star} to create massive pretraining sets, yet methods for learning from arbitrary suboptimal or shifted distributions are underexplored in robotics~\cite{kim2024openvla, pi05, wei2025empirical, hejna2024remixoptimizingdatamixtures}. We propose a simple and principled method for training robot policies that can leverage suboptimal and OOD datasets.

The simplest way to handle heterogeneous datasets with mixed quality is to discard the lowest-quality samples~\cite{kim2024openvla, pi05, zhang2026dexoraopensourcevlahighdof, wen2025dexvlavisionlanguagemodelplugin}, but this data filtering is wasteful: even suboptimal samples contain useful learning signal~\cite{daras2025ambient}. The most common alternative in robotics is to ``co-train'' on everything while down-weighting suboptimal datasets~\cite{pi05, nvidia2025gr00tn1openfoundation, trilbm, kim2024openvla, lin2026systematicstudydatamodalities, wei2025empirical, maddukuri2025sim}. However, training a model to sample from a mixture with low-quality distributions fundamentally biases the policy~\cite{tangkaratt2021robustimitationlearningnoisy, pi07}. Finetuning is a complementary method~\cite{pi05, pi06star, wen2025dexvlavisionlanguagemodelplugin, trilbm}, but can be insufficient alone (Section~\ref{sec:controlled_experiments:finetuning}).

Our key intuition is that suboptimal and high-quality samples differ in some features, but align in others. For example, non-expert teleoperators may share the same high-level plan as experts, but exhibit less precise manipulation skills. Conversely, pick-and-place data may be useful for grasping primitives, but encode the wrong task.

We turn this intuition into an algorithm by observing that robotic data exhibits a spectral power law~\cite{torralba2003statistics}. We show that this spectral structure induces a \emph{global-to-local hierarchy} in action diffusion and a locality property~\cite{lukoianov2025localityimagediffusionmodels} in the optimal denoisers. Concretely, Diffusion Policy~\cite{chi2024diffusionpolicyvisuomotorpolicy} learns high-level planning at high noise and motion primitives at low noise. Thus, a policy can selectively learn useful features from suboptimal data by using it only at noise levels where it aligns with the target distribution. This unlocks a new design axis for co-training in robotics: \emph{noise-dependent data usage}.

To this end, we propose \textbf{Ambient Diffusion Policy}, a principled algorithm for training Diffusion Policies on suboptimal robot data. Each suboptimal sample can only contribute to training at the high or low diffusion times where it aligns with the target data. Our method extends the Ambient Diffusion Omni framework~\cite{daras2024consistent, daras2025ambient} (or ``Ambient'' for short), which has been applied to computer vision and protein design~\cite{proteins}.
Our contributions are as follows: 

\begin{itemize}[leftmargin=2em, topsep=1pt, itemsep=8pt, parsep=0pt]
    \item \textbf{Properties of Robot Data.} We empirically demonstrate that robot action data exhibits a spectral power law, which induces a \textit{global-to-local hierarchy} and \textit{locality} in action diffusion. These properties make Ambient Diffusion Omni well-suited for robotics (Section~\ref{sec:robot_data}).
    \item \textbf{Ambient Diffusion Policy.} We propose a simple and principled method for learning from suboptimal data that requires just a \textit{single change} to Diffusion Policy's data sampler (Section~\ref{sec:method}).
    \item \textbf{Generality.} Ambient outperforms baselines when training on three common types of action suboptimality: noisy demonstrations, sim-to-real gap, and task mismatch (Section~\ref{sec:controlled_experiments}).
    \item \textbf{Scale.} When trained on Open X-Embodiment~\cite{open_x_embodiment_rt_x_2023}---a large dataset with mixed data quality and unstructured distribution shifts---Ambient outperforms the co-training baseline by up to \textbf{33\%} on two real-world tasks. Additionally, Ambient continues to improve as we scale the amount of suboptimal data in the training mixture, whereas co-training plateaus (Section~\ref{sec:oxe_experiments}).
    \item \textbf{Theory.} We prove that for a simple theoretical model, the spectral power law implies fast contraction through noise (Theorem~\ref{th:power_law_implies_contraction}) and locality (Theorem~\ref{th:locality}) of the optimal policies. These contributions justify our algorithm and advance the theoretical foundations of the broader Ambient Diffusion framework (Section~\ref{sec:theory}).
\end{itemize}

\section{Related Work}
\label{sec:related_work}
\textbf{Dataset Re-weighting.} The prevailing method for training on heterogeneous datasets in robotics is to re-weight each dataset's contribution to the loss \cite{pi05, nvidia2025gr00tn1openfoundation, trilbm, kim2024openvla, lin2026systematicstudydatamodalities, geminiroboticsteam2025geminiroboticsbringingai, wei2025empirical, maddukuri2025sim, nasiriany2024robocasalargescalesimulationeveryday}. In our problem setting, this method would optimize the loss
\begin{equation}
    \mathcal{L}_\text{co-training}(\theta) = \alpha\,\mathcal{L}_{\mathcal{D}_p}(\theta) + (1-\alpha)\,\mathcal{L}_{\mathcal{D}_q}(\theta),
    \label{eq:cotraining}
\end{equation}
where $\mathcal{L}_{\mathcal{D}_p}$ and $\mathcal{L}_{\mathcal{D}_q}$ denote the loss on the target and suboptimal datasets, and $\alpha\in[0,1]$ is the re-weighting parameter. In the remainder of the paper, we use ``co-training'' to refer to this standard approach. A core challenge is weight selection for each dataset~\cite{geminiroboticsteam2025geminiroboticsbringingai, kim2024openvla, theorydifferentdomains, hejna2025robotdatacurationmutual, xie2023doremioptimizingdatamixtures, hejna2024remixoptimizingdatamixtures, agia2025cupidcuratingdatarobot}. But even with the optimal weights, training a policy to sample from a mixture with suboptimal distributions biases the policy~\cite{tangkaratt2021robustimitationlearningnoisy, pi07}. Our method avoids this issue by only training on suboptimal samples in specific intervals of the diffusion process. Importantly, these intervals have precise interpretations and can be automatically annotated.

\textbf{Data Filtering.} A related method is data filtering, which rejects suboptimal samples entirely~\cite{kim2024openvla, pi05, zhang2026dexoraopensourcevlahighdof, Goyal_2024_CVPR, li2025datacomplmsearchgenerationtraining}. Our experiments and prior work~\cite{daras2025ambient, pi06star, pi07} show that data filtering is wasteful. For example, Dexora~\cite{zhang2026dexoraopensourcevlahighdof} improves policy performance by discarding 80\% of its training demonstrations using a kinematic smoothness heuristic. Section~\ref{sec:controlled_experiments:motion_planning} shows the opposite for Ambient Diffusion Policy: training on additional non-smooth demonstrations improves performance while retaining the policy's smoothness. Similarly, OpenVLA's Magic Soup++ mixture~\cite{kim2024openvla} includes only 27 of the datasets in OXE~\cite{open_x_embodiment_rt_x_2023}; Section~\ref{sec:oxe_experiments} shows that our method can still extract useful learning signal from the discarded datasets.

\textbf{Learning from Suboptimal or OOD Data.}
Prior works studied learning from corrupted data \cite{Aali_2023, daras2024consistent, chen2025denoisingscoredistillationnoisy, daras2024surveydiffusionmodelsinverse, shah2025does, dataloops, lu2025sfbd, sfbdomni, park2025measurementscorebaseddiffusionmodel, modi2026generativemodelingblackboxcorruptions, kawar2024gsurebaseddiffusionmodeltraining, matsuzaki2026scorecleanimagegeneration, bora2018ambientgan}; however, many require multiple training runs~\cite{bai2024expectationmaximizationalgorithmtrainingclean, hosseintabar2025diffemlearningcorrupteddata, rozet2025learningdiffusionpriorsobservations} or rely on knowing the form of the corruption. The latter is unrealistic in robotics. For example, the sim-to-real gap has no closed-form description. Recent works in robotics have also begun exploring training methods for suboptimal data. $\pi_{0.7}$~\cite{pi07} labels data quality and LDA-1B~\cite{lyu2026lda1bscalinglatentdynamics} relegates low-quality trajectories to an auxiliary dynamics objective. Other approaches learn invariant representations between the target and suboptimal data~\cite{cheng2026generalizabledomainadaptationsimandreal, punamiya2025egobridgedomainadaptationgeneralizable}, but are challenging to scale~\cite{wei2025empirical}. \citet{tangkaratt2021robustimitationlearningnoisy} and \citet{liu2022robustimitationlearningcorrupted} propose methods for robust imitation learning; however, they focus on ignoring corrupted samples instead of learning from them. They also require over half the data to be high-quality. Concurrent work explored a similar idea to ours for hand-tracking data~\cite{pace2025xdiffusiontrainingdiffusionpolicies}. In contrast, our method learns from arbitrary mixtures of suboptimal action data, scales to large real-world datasets~\cite{open_x_embodiment_rt_x_2023}, has theoretical grounding, and is simple to implement.

\section{Background}
\label{sec:background}
\subsection{Robot Imitation Learning and Diffusion Policy}
\label{sec:background:diffusion}

Given a dataset $\mathcal{D}_p=\{ (O^{(i)}, A^{(i)})\}_{i=1}^{n_p}$ of observation-action pairs, imitation learning aims to train a policy $\pi(A\mid O)$ that approximates the conditional distribution $p(A\mid O)$~\cite{chi2024diffusionpolicyvisuomotorpolicy}. Here, $A$ and $O$ represent chunks of future actions and recent observations~\cite{zhao2023learningfinegrainedbimanualmanipulation}.
Diffusion Policy~\cite{chi2024diffusionpolicyvisuomotorpolicy} samples from $\pi(A \mid O)$ with a conditional denoising process over the action space~\cite{ho2020denoisingdiffusionprobabilisticmodels}. Concretely, it learns
a family of denoisers $\{h_\theta(A_t, O, t)\}_{t\in [0,T]}$ that predict the clean action $A_0$ from a noisy version\footnote{Diffusion Policy~\cite{chi2024diffusionpolicyvisuomotorpolicy} uses a variance-preserving implementation. We present it from the variance-exploding perspective to lighten notation. The two perspectives are equivalent up to a change of variables~\cite{song2021scorebasedgenerativemodelingstochastic}; see Appendix~\ref{sec:appendix:vp_vs_ve}.}
\begin{equation}
    A_t = A_0 + \sigma(t)\, Z, \quad Z \sim \mathcal{N}(0, I),
    \label{eq:forward}
\end{equation}
where $T$ is a sufficiently large constant, $\sigma(t)$ is a strictly increasing function known as the noise schedule, and $\sigma(0)=0$. For brevity, we write $\sigma_t := \sigma(t)$.

These denoisers are trained by minimizing the denoising loss~\cite{ho2020denoisingdiffusionprobabilisticmodels} (or a reparametrization of it, as in flow matching~\cite{lipman2023flowmatchinggenerativemodeling}):
\begin{equation}
    \mathcal{L}(\theta) = \mathbb{E}_{t,\, (O, A_0),\, A_t \mid A_0}
    \left[ \| h_\theta(A_t, O, t) - A_0 \|^2 \right],
    \label{eq:diffusion_loss}
\end{equation}
where $t \sim \mathcal{U}[0, T]$, $(O, A_0)\sim \mathcal{D}_p$, and $A_t$ is sampled from Eq.~\eqref{eq:forward}. For a sufficiently expressive $h_\theta$, the minimizer of Eq.~\eqref{eq:diffusion_loss} is the conditional expectation $h_\theta^*(A_t, O, t) = \mathbb{E}[A_0 \mid A_t, O]$. Given access to the conditional expectations, one can sample from the target distribution $\pi(A \mid O)$ by running a reverse diffusion process~\cite{chi2024diffusionpolicyvisuomotorpolicy, ho2020denoisingdiffusionprobabilisticmodels, song2022denoisingdiffusionimplicitmodels}. Appendix~\ref{sec:appendix:diffusion_policy} provides implementation details for Diffusion Policy.

\subsection{Problem Setting}
\label{sec:background:datasets}
Suppose we have access to a (small) dataset, $\mathcal{D}_p = \{(O^{(i)}, A^{(i)})\}_{i=1}^{n_p}$, from a target distribution, $p$, and a much larger dataset, $\mathcal{D}_q = \{(O^{(i)}, A^{(i)})\}_{i=1}^{n_q}$, from a shifted or suboptimal distribution, $q$. As discussed in the Introduction, this is a common problem setting in robotics.
\definecolor{steelbluelight}{HTML}{E8F4F8}
\definecolor{steelbluedark}{HTML}{2E9BBF}
\begin{tcolorbox}[
    enhanced,
    colback=steelbluelight,
    colframe=steelbluedark,
    boxrule=0.5pt,
    arc=5pt,
    left=6pt, right=6pt, top=8pt, bottom=8pt
]
\textbf{Problem Statement:} Our goal is to leverage the vast
pool of \textit{suboptimal} samples in $\mathcal{D}_q$ to improve the estimation
of the denoisers, $\mathbb{E}_p[A_0 \mid A_t, O]$, on the user-defined \textit{target} distribution.
\end{tcolorbox}
The definitions of $p$ and $q$ can be arbitrary. In general, $\mathcal D_p$ is a dataset that a practitioner considers ``high-quality'' for a downstream application. In this paper, we define $\mathcal{D}_p$ as expert demonstrations on the target robot, environment, and task; however, our framework applies equally well to any other definition, including heuristic measures~\cite{belkhale2023dataqualityimitationlearning} or human labels~\cite{pi07}. In Section \ref{sec:controlled_experiments}, we construct $\mathcal{D}_q$ to contain controlled distribution shifts. In Section \ref{sec:oxe_experiments}, we scale $\mathcal{D}_q$ to Open X-Embodiment (OXE) \cite{open_x_embodiment_rt_x_2023}.

\subsection{Ambient Diffusion Omni}
This Problem Statement is not unique to robotics.
Ambient Diffusion Omni was proposed by~\citet{daras2025ambient} to address a similar problem in the vision domain. Their key idea is to allow suboptimal samples from $q$ to be used only at certain diffusion times during training. The algorithm assigns two parameters to samples from $\mathcal D_q$, denoted $t_{\min}$ and $t_{\max}$. Samples from $\mathcal D_q$ can only contribute to the learning at diffusion times $t \in [0, t_{\max}) \cup (t_{\min}, T]$.

The algorithm leverages \textit{noise as a contraction mechanism}: for any distributions $p$ and $q$ supported on a subset of $\mathbb{R}^d$ with diameter $D$, the authors show that
\newcommand{\dtv}{\mathrm{d}_{\mathrm{TV}}}
\begin{gather}
\dtv(p_t, q_t)\le \dtv(p,q)\cdot {D \over 2\sigma_t},
\label{eq:tv_contraction_bound}
\end{gather}
where $d_{\text{TV}}$ is the Total Variation distance, $p_t = p\circledast \mathcal N(0, \sigma_t I)$, and $q_t = q\circledast \mathcal N(0, \sigma_t I)$.

Simply put, noise erases distribution differences, i.e. $p$ and $q$ \textit{contract} towards one another. Hence, there exists a sufficiently high diffusion time, $t_{\min}$, for which noisy actions from both the high-quality and suboptimal distributions become effectively indistinguishable.  Consequently, $\mathcal D_q$ can safely contribute to training at high noise when $t\in(t_{\min}, T]$. The utility of $q$ is highest when $t_{\min}$ is small (i.e. when $p$ and $q$ contract quickly). We discuss sufficient conditions for fast contraction and justification for using $\mathcal D_q$ for $t\in[0, t_{\max})$ in subsequent sections.

Section~\ref{sec:robot_data} examines structural properties of action data that make this method well-suited for robotics.
Section~\ref{sec:theory} theoretically formalizes the discussion in both Section~\ref{sec:robot_data} and prior work~\cite{daras2025ambient, proteins}. Section~\ref{sec:method} outlines our algorithm. Despite the in-depth motivation, the resulting method is remarkably simple: it requires a \textbf{single change} to the Diffusion Policy~\cite{chi2024diffusionpolicyvisuomotorpolicy} data sampler. We conclude with experiments that demonstrate the generality and scalability of our method.

\section{Distributional Properties of Robot Data}
\label{sec:robot_data}
We empirically show that robot data exhibits a \textit{spectral power law}. This spectral structure is special for two reasons. First, it is not universal: it is absent in natural audio~\cite{audio, dieleman2024spectral}, stellar spectra~\cite{df_gray_observation_2005}, and Poisson point processes~\cite{papoulis2002probability}. Second, we show that it makes Ambient Diffusion effective: it implies a global-to-local hierarchy (Figure~\ref{fig:visualize_backward}), fast contraction through noise (Theorem~\ref{th:power_law_implies_contraction}), and a locality property (Theorem~\ref{th:locality}). Hence, Ambient Diffusion is well-suited for robotics.

\subsection{Spectral Power Law in Robotics}
\label{sec:robot_data:spectral_power_law}

\begin{figure*}
    \centering
    \includegraphics[width=\linewidth]{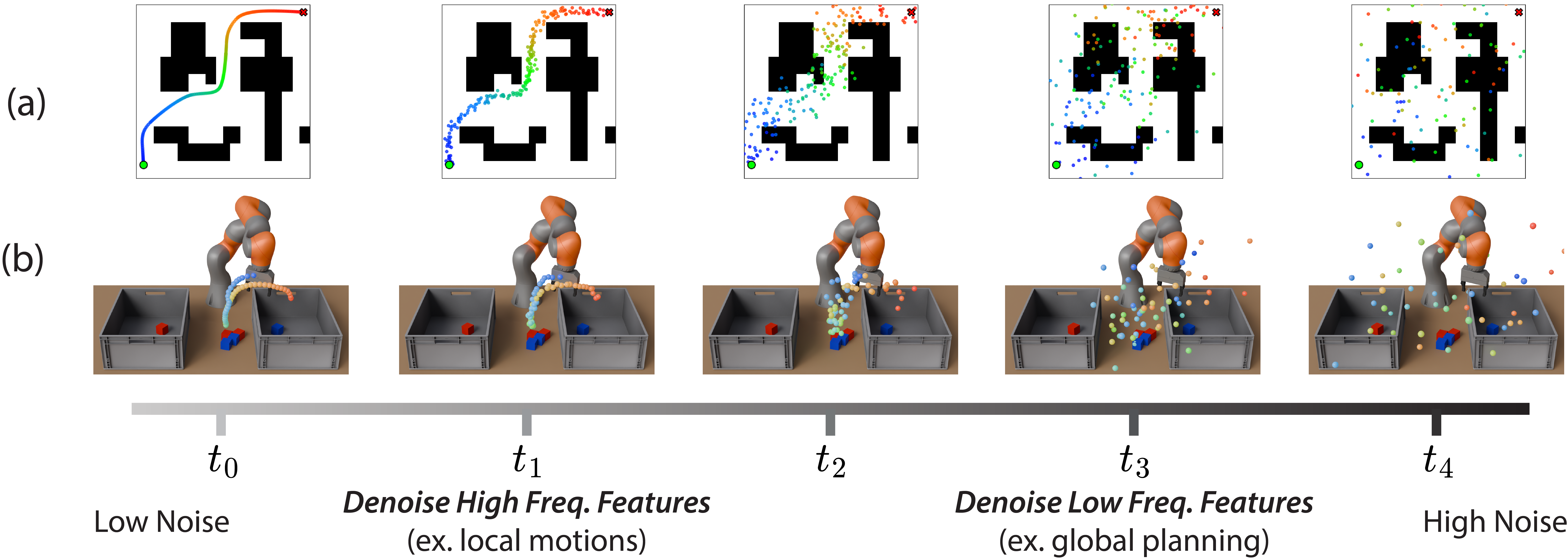}

    \caption{\textbf{The spectral power law induces a \textit{global}-to-\textit{local} hierarchy in action diffusion.} This is analogous to the ``coarse-to-fine'' hierarchy in image diffusion~\cite{dieleman2024spectral}. \textbf{(a)} \textit{$t \geq t_2$ (global planning):} the policy commits to a high-level path through the maze. \textit{$t < t_2$ (local refinement):} it refines the plan to be collision-free at $t_1$ and smooth at $t_0$. \textbf{(b)} $t_3$: the policy plans to move a block to the right bin. $t_2$: it plans which block to grasp. $t < t_2$: it refines the grasping motion.}
    \label{fig:visualize_backward}
\end{figure*}

The power spectral density (PSD) of a signal measures its energy at each frequency component, $f$. The signal exhibits a spectral power law if its PSD decays approximately as $|f|^{-\alpha}$, where $\alpha > 0$. In other words, its energy is concentrated in low-frequency components. A formal definition extended to random vectors is presented in Definition~\ref{def:spectral_power_law}. Figures~\ref{fig:spectral_power_law} and~\ref{fig:psd_all} show that \textbf{robot action data exhibits a spectral power law}.

\begin{figure}[!t]
    \centering
    \begin{subfigure}[b]{0.48\textwidth}
        \centering
        \includegraphics[width=\linewidth]{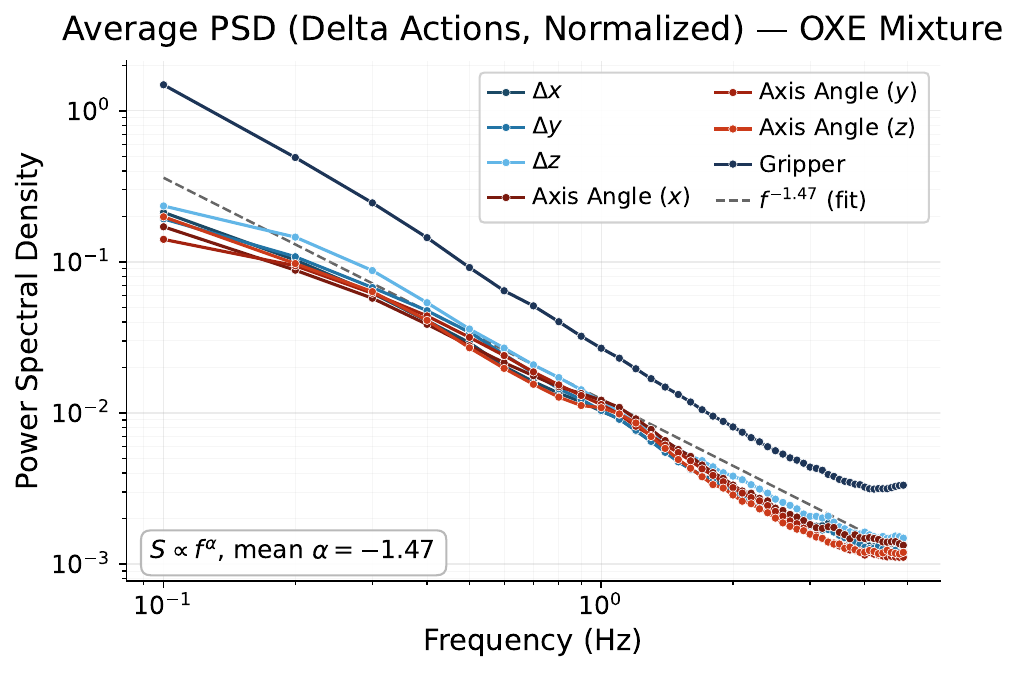}
        \caption{\textbf{Spectral power law in robot action data.} The power spectral density for OXE~\cite{open_x_embodiment_rt_x_2023}, a collection of 2.4M robot episodes across 70+ datasets, follows a power law. Figure~\ref{fig:psd_all} shows that this spectral power law is a more general property of robot data.}
        \label{fig:spectral_power_law}
    \end{subfigure}
    \hfill
    \begin{subfigure}[b]{0.48\textwidth}
        \centering
        \includegraphics[width=\linewidth]{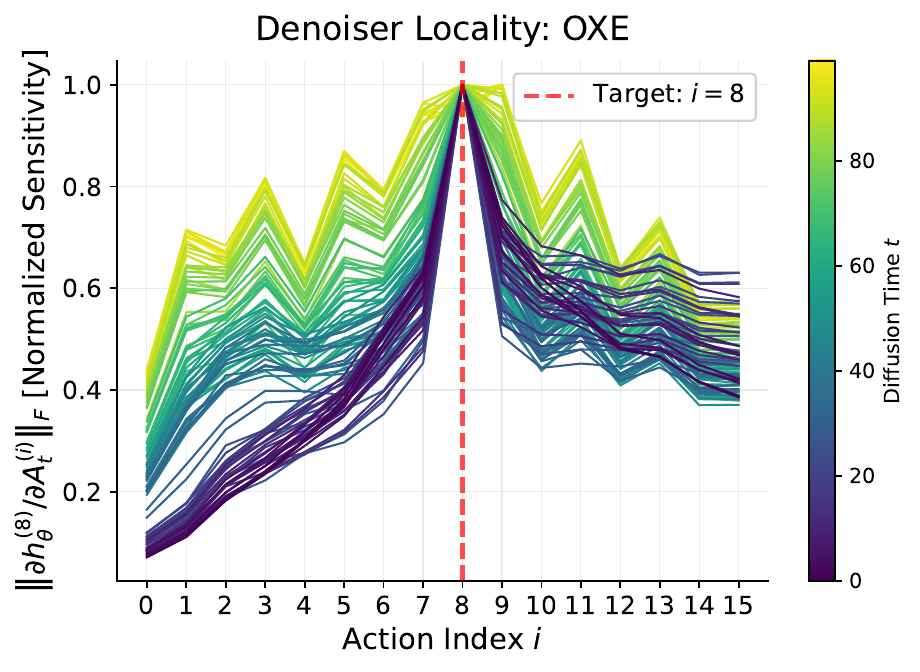}
        \caption{\textbf{Locality in robot action data.} Visualizing the sensitivity of the 8th action estimate to its neighboring actions (Eq.~\eqref{eq:sensitivity_field}). At low noise, the 8th action estimate is most sensitive to its temporal neighbors.}
        \label{fig:locality}
    \end{subfigure}
    \label{fig:robot_data_properties}
\end{figure}

\textbf{Implications for Diffusion Policy~\cite{chi2024diffusionpolicyvisuomotorpolicy}:} When a spectral power law exists, \citet{dieleman2024spectral} showed that denoisers generate low-frequency features at high noise, and high-frequency features at low noise. Figure~\ref{fig:visualize_PSD} illustrates that in robotics, low frequencies encode \textit{global} features (e.g. the high-level plan), while high frequencies encode \textit{local} features (e.g. grasping primitives and smoothness). Thus, Diffusion Policies exhibit a \textit{global-to-local} hierarchy: they learn \textit{global} planning at high noise and \textit{local} motion primitives at low noise. Appendix~\ref{sec:appendix:action_diffusion} provides further experimental evidence.

\begin{figure*}[!t]
    \centering
    \includegraphics[width=\linewidth]{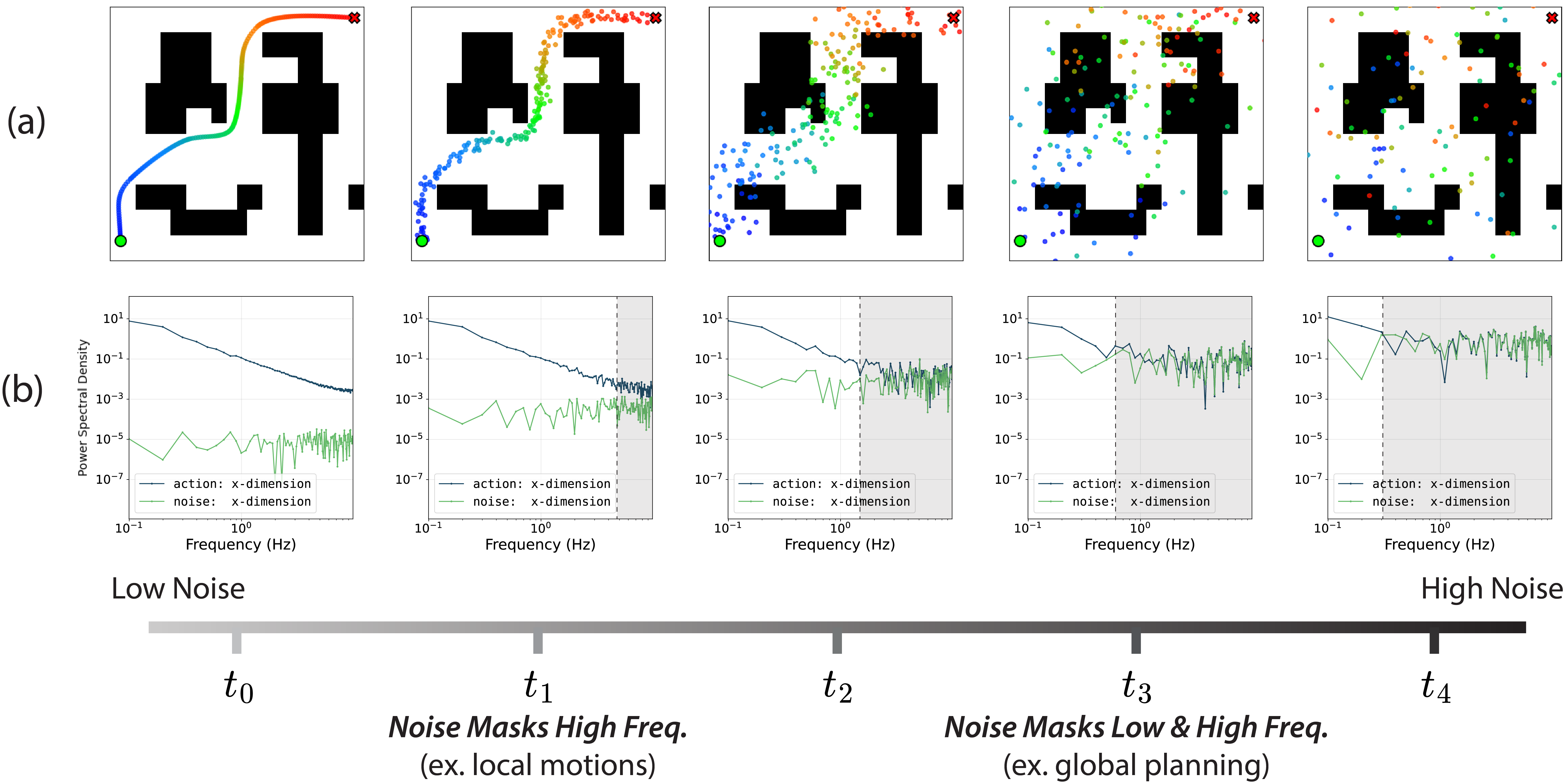}
    \caption{\textbf{The spectral power law makes Gaussian noise act as a high-frequency mask.} \textbf{(a)} Noisy trajectory visualizations. \textbf{(b)} PSD analysis for the noisy trajectories in the $x$-dimension at each diffusion time: the gray regions indicate \textit{masked} frequencies, which have low SNR. As $t$ increases, adding Gaussian noise destroys the signal in the trajectory, starting from the local (high-frequency) features first. This effect is visible in both the noisy trajectories and the PSD plots.}
    \label{fig:visualize_PSD}
\end{figure*}

This hierarchy unlocks a new design axis for co-training algorithms in robotics: \textit{noise-dependent data usage}. Diffusion Policies learn different features of the actions at each noise level; thus, suboptimal samples should only contribute to training at noise levels where $p$ and $q$ align.

\textbf{Implications for Ambient Diffusion Policy:} Due to the spectral power law, adding Gaussian noise to robot data acts as a high-frequency mask (Figure~\ref{fig:visualize_PSD}). Thus, if differences between $p$ and $q$ are concentrated in a high-frequency tail (i.e. low-level actions), then $t_{\min}$ will be small. We formalize this in Theorem~\ref{th:power_law_implies_contraction}.
By only learning from suboptimal samples when $t > t_{\min}$, policies can learn the useful \textit{global} features in $\mathcal D_q$ and ignore the \textit{local} suboptimality
(which is masked by noise).

Fortunately, the suboptimality in many robot datasets is in the local actions. For instance, non-expert and expert demonstrations may aim to achieve the same goal, but differ in the quality of the low-level motions. Similarly, in sim-to-real co-training or cross-embodied data, the high-level plan may be shared while the precise contact dynamics or feasible motions may differ.

\subsection{Locality in Robotics}
\label{sec:robot_data:locality}

A denoiser exhibits ``locality'' if its output at each coordinate depends primarily on a small receptive field in the noisy input~\cite{lukoianov2025localityimagediffusionmodels, daras2025ambient}.\footnote{As in~\citet{lukoianov2025localityimagediffusionmodels}, ``locality'' is a property of the data which is inherited by the optimal denoiser at low noise. We use the term ``locality'' loosely as both a property of the data and the resulting denoisers.} In robotics, this means that each action output from the denoiser is most sensitive to its temporal neighbors in the noisy input.
In Theorem~\ref{th:locality}, we prove that the spectral power law can imply locality in the optimal denoisers.
Indeed, Figure~\ref{fig:locality} empirically identifies this phenomenon in robotics. Intuitively, this is because these denoisers focus on resolving \textit{local} motion primitives, so they do not need to attend to distant actions in the input.

Now suppose that $p$ and $q$ share the same local motion primitives but differ globally. Due to locality, there exists a sufficiently small $t_{\max}$ such that for all $t < t_{\max}$, $p$ and $q$ agree within the receptive field of the optimal denoiser. Section~\ref{sec:theory} formalizes this statement. Section~\ref{sec:controlled_experiments:bin_sorting} leverages locality to learn \textit{local} grasping primitives from data at low noise, even when its task-level structure is incorrect.

\textbf{Remark:} A spectral power law and locality exists in all the examined robot datasets (Table~\ref{tab:psd_datasets}). This suggests that these are general properties of robot data. Section~\ref{sec:method} provides methods for estimating $t_{\min}$ and $t_{\max}$ to leverage these properties in Ambient Diffusion Policy.

\section{Theoretical Foundations}
\label{sec:theory}
We prove two theorems that extend the theoretical foundations of Ambient Diffusion and rigorously justify our method for robotics. Prior works~\cite{daras2025ambient} empirically identify the spectral power law and locality as important properties. We prove that for a simplified model, the former implies both \textit{fast contraction} through noise (Theorem~\ref{th:power_law_implies_contraction}) and \textit{locality} of the optimal denoiser (Theorem~\ref{th:locality}). This establishes the spectral power law as the single fundamental property underlying the framework.

Both theorems are proved for zero-mean stationary Gaussians. We adopt this setting for two reasons: \textbf{1)} it simplifies the bounds, and \textbf{2)} it is a natural model for latent-space diffusion, which has been proposed in robotics~\cite{bauer2026latentactiondiffusioncrossembodiment}, and is common in generative modeling more broadly~\cite{rombach2022highresolutionimagesynthesislatent}. Extending the analysis to general distributions is left to future work. Proofs are presented in Appendix~\ref{sec:appendix:theory}.

We begin by formally defining the spectral power law for general random vectors.

\begin{definition}[Spectral Power Law]
Let $X$ be a zero-mean, stationary random vector on $\mathbb{R}^N$, and let $S_X(f) := \mathbb{E}\big[|(\mathcal{F}X)(f)|^2\big]$ denote its power spectral density in the orthonormal Fourier basis. We say that the distribution of $X$ follows a \textit{spectral power law} if there exist a constant $C > 0$, a rolloff exponent $\alpha > 1$, and a cutoff frequency $f_0 \ge 1$ such that the high-frequency spectrum satisfies:
$$
S_X(f) = C |f|^{-\alpha} \quad \text{for all } |f| \ge f_0.
$$
The exact spectral behavior below the cutoff $f_0$ is permitted to be arbitrary but bounded.
\label{def:spectral_power_law}
\end{definition}

\subsection{Contraction Through Noise}
Theorem \ref{th:power_law_implies_contraction} states that a spectral power law and low-frequency alignment between $p$ and $q$ imply \textit{fast contraction} through noise. Moreover, this contraction is \textit{independent of the original distance}.

\begin{theorem}[Spectral Power Law implies fast contraction through noise]
    Let $p_0 = \mathcal{N}(0, \Sigma_p)$ and $q_0 = \mathcal{N}(0, \Sigma_q)$ be zero-mean stationary Gaussians on $\mathbb{R}^N$ with power spectra $S_p, S_q$. Fix a cutoff frequency $f^\star \ge 1$, an exponent $\alpha > \tfrac{1}{2}$, and a constant $C > 0$. Assume:

\begin{enumerate}
    \item \textbf{(Low-frequency agreement.)} $S_p(f) = S_q(f)$ for every $f \le f^\star$.
    \item \textbf{(Power-law tail.)} $\max(S_p(f), S_q(f)) \le C f^{-\alpha}$ for every $f$ such that $f^\star < f \leq \lfloor N/2 \rfloor$.
\end{enumerate}

Let $p_t := p_0 * \mathcal{N}(0, \sigma_t^2 I)$ and $q_t := q_0 * \mathcal{N}(0, \sigma_t^2 I)$. If
\begin{equation}
\sigma_t^2 \ge C (f^\star)^{-\alpha},
\tag{$\star$}
\end{equation}
then
\[
\boxed{ \mathrm{d}_{\mathrm{KL}}(p_t \| q_t) \le \frac{4 C^2}{(2\alpha - 1)\sigma_t^4 (f^\star)^{2\alpha - 1}} }
\qquad
\boxed{ \mathrm{d}_{\mathrm{TV}}(p_t, q_t) \le \frac{\sqrt{2}C}{\sigma_t^2 \sqrt{(2\alpha - 1)(f^\star)^{2\alpha - 1}}} }.
\]
\label{th:power_law_implies_contraction}
\end{theorem}

In contrast, Eq.~\eqref{eq:tv_contraction_bound} depends on the original TV distance (which can be arbitrarily close to 1) and decays only as $1/\sigma$. For our simple theoretical model, the spectral power law enables fast contraction and, consequently, high utilization of the suboptimal data in Ambient Diffusion.

\subsection{Spectral Power Law Implies Locality}
The next theorem shows that the spectral power law implies \textit{locality} for the optimal denoiser at low noise levels.

\begin{theorem}[Spectral Power Law implies locality of the optimal denoiser]
    Let $X_0 \sim \mathcal{N}(0, \Sigma)$ be a zero-mean stationary Gaussian on $\mathbb{R}^N$ whose power spectrum $S(f)$ follows a spectral power law with constant $C > 0$ and exponent $\alpha > 1$. Let $X_t = X_0 + \sigma_t Z$ be the variance-exploding noisy observation, and $h^\star_t(x) := \mathbb{E}[X_0 \mid X_t = x]$ denote the MMSE optimal denoiser.

    Let $M_i$ be a circulant mask that sets to zero all coordinates of the input $x$ at a circular distance strictly greater than $L$ from index $i$. Then, for any mask size $1 \le L \le \lfloor N/2 \rfloor$, the absolute error in the $i$-th coordinate between the full optimal denoiser and the masked denoiser is bounded by:
    \[
        \boxed{ \bigg| (h^\star_t(x))_i - (h^\star_t(M_i x))_i \bigg| \le \frac{\alpha N \|x\|_\infty}{8 L}. }
    \]
\label{th:locality}
\end{theorem}

In short, at low noise levels the optimal denoiser is functionally myopic: it can reconstruct the signal by observing \textit{only} a local spatial neighborhood, safely ignoring the global structure of the input. This is precisely the property that Ambient Diffusion Policy exploits to learn \textit{local} action primitives from suboptimal data, even when its global structure is incorrect (Section~\ref{sec:controlled_experiments:bin_sorting}).

\section{Ambient Diffusion Policy: Method}
\label{sec:method}

Since robot data exhibits a spectral power law, learning from suboptimal data in robotics with Ambient is empirically and theoretically justified. We now present our algorithm: Ambient Diffusion Policy. It has two phases: a) data annotation, and b) training.

\subsection{Phase 1: Dataset annotation}

This phase annotates $t_{\max}$ and $t_{\min}$ for the suboptimal samples from $\mathcal D_q$. The simplest method is to perform a hyperparameter sweep. A more principled solution is to use a classification network. In particular, $\sigma_{t_{\min}}$ corresponds to the minimal amount of noise required to make the distributions $p_t$ and $q_t$ indistinguishable~\cite{daras2025ambient}. Hence, one way to find this noise level is to train a classifier, $c_{\phi}(A_t, t)$ that predicts the probability that the noisy action $A_t$ came from $p_t$ (as opposed to $q_t$). Then, $t_{\min}$ is the minimum noise level required to fool the classifier into classifying samples from $q_t$ as samples from $p_t$ (i.e. the classifier cannot reliably distinguish $p_t$ and $q_t$). Formally, for some small $\tau > 0$, we assign
\begin{equation}
    t_{\min} = \inf \{t : \mathbb E_{(O, A_0)\sim q,\, A_t|A_0}[c_{\phi^*}(A_t, t)] > 0.5 - \tau \}.
    \label{eq:tmin}
\end{equation}
Figure~\ref{fig:classifier_accuracy} visualizes this method for $\tau=0.05$; details are in the Appendix. If the classifier is well-trained, we show that this annotation method guarantees distributional closeness between $p_t$ and $q_t$ for $t > t_{\min}$ (Theorem~\ref{th:approximate-classifier-threshold}).

This classifier can also be used to annotate samples at different granularities (e.g. one $t_{\min}$ per dataset vs one $t_{\min}$ per sample); Appendix~\ref{sec:appendix:sim_and_real} and Figure~\ref{fig:planar_pushing_ablations}\subref{fig:planar_pushing_partitions} provide additional details. A similar classifier-based approach can be used to annotate $t_{\max}$~\cite{daras2025ambient}, although we do not use it in this paper.

\subsection{Phase 2: Training}

The training loop for Ambient Diffusion Policy is the same as Diffusion Policy~\cite{chi2024diffusionpolicyvisuomotorpolicy} with one required change: a custom data sampler. The custom sampler first draws a batch of diffusion times $t^{(i)}\sim \mathcal{U}[0, T]$. Afterward, it draws admissible samples from $\mathcal{D}_p \cup \mathcal{D}_q$. Samples from $\mathcal{D}_p$ are always admissible; suboptimal samples from $\mathcal{D}_q$ are admissible only if $t^{(i)} \in [0, t_{\max}) \cup(t_{\min}, T]$. Figure~\ref{fig:teaser} visualizes the intervals where each dataset can contribute to training. Inference is identical to Diffusion Policy. We defer the full training pseudocode to Appendix~\ref{sec:appendix:pseudocode} to prevent implementation details from obscuring the essence of the method.

\textbf{Variants of Ambient Diffusion Policy.} Ambient Diffusion Policy provides an (optional) alternative loss function~\cite{daras2024consistent}, which is discussed in Appendix~\ref{sec:appendix:ambient_loss}.
Additionally, there is a variation of the algorithm that uses the classifier from Phase 1 for rejection sampling at training-time. We provide pseudocode and theoretical justification in Appendix~\ref{sec:appendix:rejection_sampling} and leave experimental verification of this variant to future work.

\section{Controlled Experiments}
\label{sec:controlled_experiments}
We evaluate Ambient Diffusion Policy on three distribution shifts in robot action data to illustrate its generality: noisy trajectories, sim-to-real gap, and task mismatch. This section compares our method against data filtering, co-training, and finetuning. Appendix~\ref{sec:appendix:experiments_overview} provides experimental details.

\subsection{Noisy Trajectories}
\label{sec:controlled_experiments:motion_planning}
\begin{wrapfigure}[13]{r}{0.5\linewidth}
    \centering
    \includegraphics[width=\linewidth]{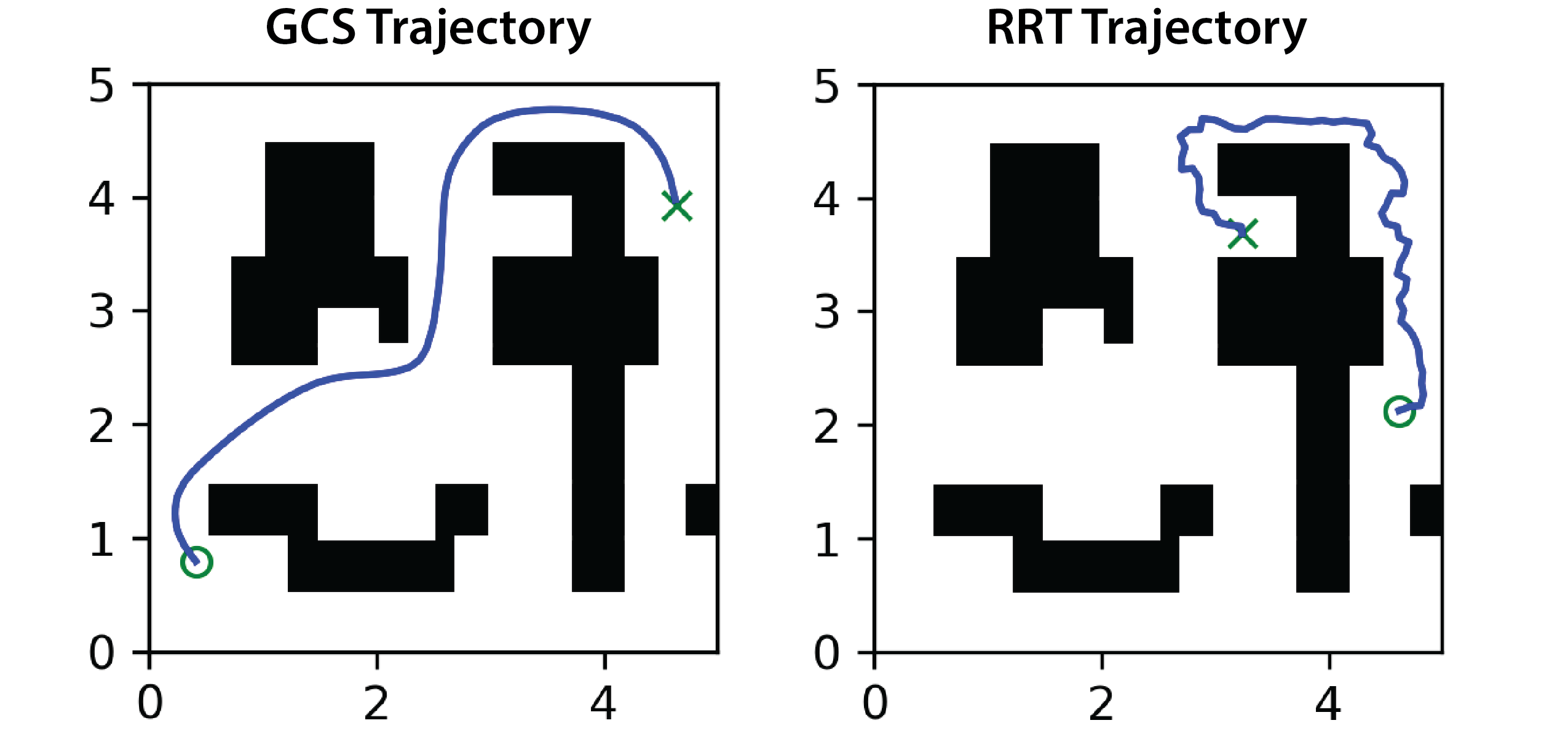}
    \caption{\textbf{GCS (left) vs RRT (right) data.} The high-quality GCS \cite{marcucci2022motionplanningobstaclesconvex} trajectories are smoother and shorter than their RRT \cite{LaValle1998RapidlyexploringRT} counterparts.}
    \label{fig:gcs_rrt}
\end{wrapfigure}
Consider a 2D point robot that must navigate between random start and goal positions in a maze environment (Figure~\ref{fig:teaser}a). A policy rollout is successful if it reaches an $\epsilon$-ball around the goal within a time limit without collision. We measure a policy's \textit{global} understanding of the maze using success rate, and its \textit{local} smoothness using average squared acceleration (Eq.~\eqref{eq:smoothness}).

$\mathcal D_p$ contains 50 smooth trajectories from GCS
\cite{marcucci2022motionplanningobstaclesconvex}; $\mathcal D_q$ contains 5,000
cheaper but jittery trajectories from RRT
\cite{LaValle1998RapidlyexploringRT}. Figure~\ref{fig:gcs_rrt} visualizes the datasets. This mimics real-world sources of suboptimality, such as teleoperation with low-quality hardware or noisy control stacks \cite{foland2025diffusionpolicieslearnkinematic}.

\begin{table}[!t]
    \centering
    \caption{\textbf{Maze and Neural Motion Planning results (1000 trials each).} In both settings, Ambient Diffusion Policy achieves the highest success rate and is substantially smoother than co-training.}
    \label{tab:maze_nmp_results}
    \renewcommand{\arraystretch}{1.7}
    \begin{subtable}[t]{0.48\textwidth}
        \centering
        \scriptsize
        \setlength{\tabcolsep}{3pt}
        \begin{tabular}{l|cc}
            \toprule
            Training Algo. & Success Rate $\uparrow$ & \makecell{Smoothness $\downarrow$\\(Eq.~\eqref{eq:smoothness})} \\
            \midrule
            Data Filtering & $57.5^{+3.1}_{-3.1}$ \% & $31.9$ \\
            \makecell[l]{Co-train\\($\alpha^*=0.019$)} & $99.4^{+0.4}_{-0.7}$ \% & $62.2$ \\
            \makecell[l]{Ambient\\($\sigma_{t_{\min}^*}=0.074$)} & $\mathbf{99.5^{+0.3}_{-0.7}}$ \% & $\mathbf{31.0}$ \\
            \bottomrule
        \end{tabular}
        \caption{2D maze.}
        \label{tab:maze_results}
    \end{subtable}
    \hfill
    \begin{subtable}[t]{0.48\textwidth}
        \centering
        \scriptsize
        \setlength{\tabcolsep}{3pt}
        \begin{tabular}{l|cc}
            \toprule
            Training Algo. & Success Rate $\uparrow$ & \makecell{Smoothness $\downarrow$\\(Eq.~\eqref{eq:smoothness})} \\
            \midrule
            Data Filtering & $46.0^{+3.1}_{-3.1}$ \% & $\mathbf{3.9}$ \\
            \makecell[l]{Co-train\\($\alpha^*=0.091$)} & $59.9^{+3.1}_{-3.1}$ \% & $42.7$ \\
            \makecell[l]{Ambient\\($\sigma_{t_{\min}^*}=0.025$)} & $\mathbf{65.9^{+2.9}_{-3.0}}$ \% & $31.4$ \\
            \bottomrule
        \end{tabular}
        \caption{7-DoF neural motion planning.}
        \label{tab:nmp_results}
    \end{subtable}
\end{table}

\begin{figure}[!t]
    \centering
    \begin{subfigure}[t]{0.48\textwidth}
        \centering
        \includegraphics[width=\linewidth]{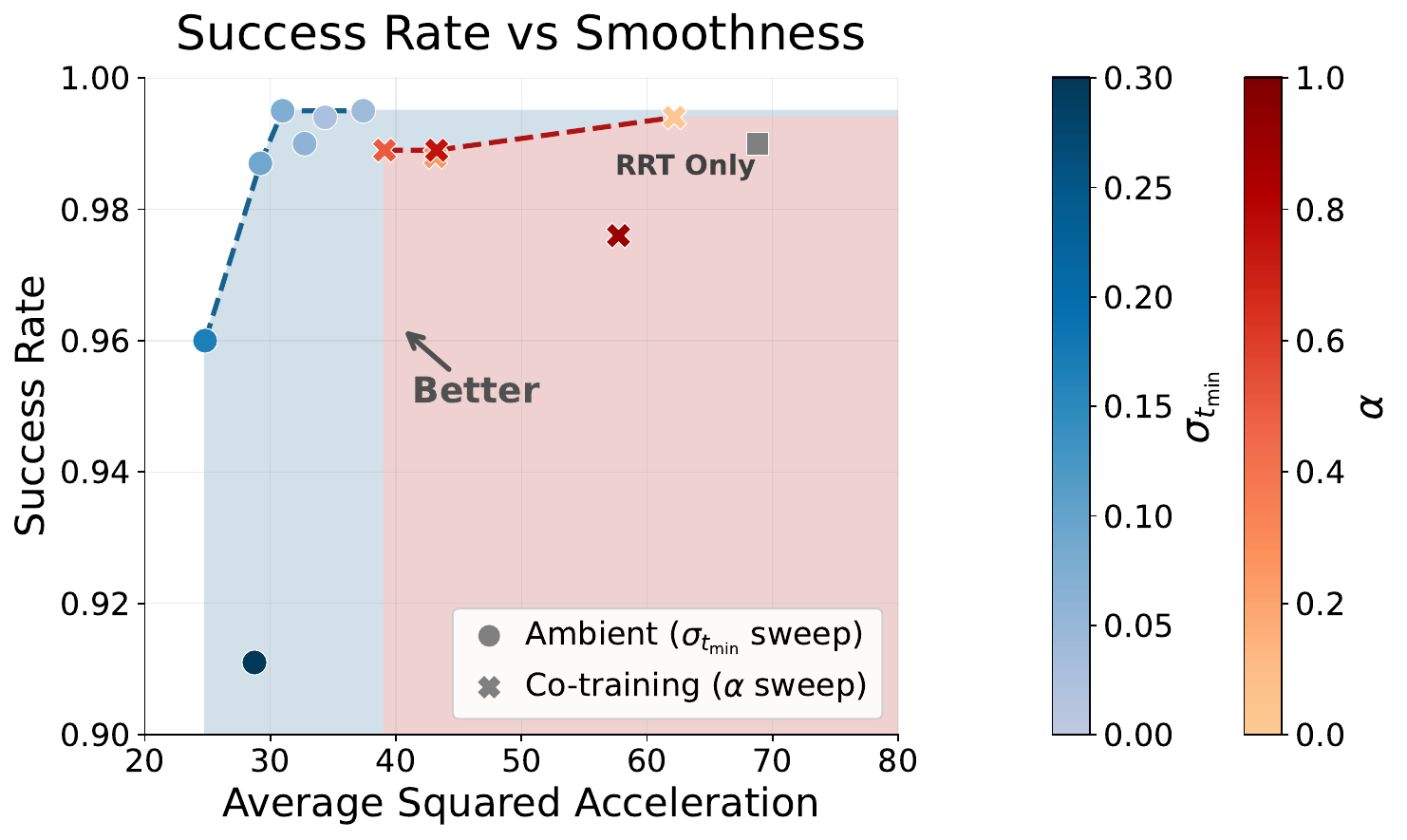}
        \caption{2D maze.}
        \label{fig:maze_pareto}
    \end{subfigure}
    \hfill
    \begin{subfigure}[t]{0.48\textwidth}
        \centering
        \includegraphics[width=\linewidth]{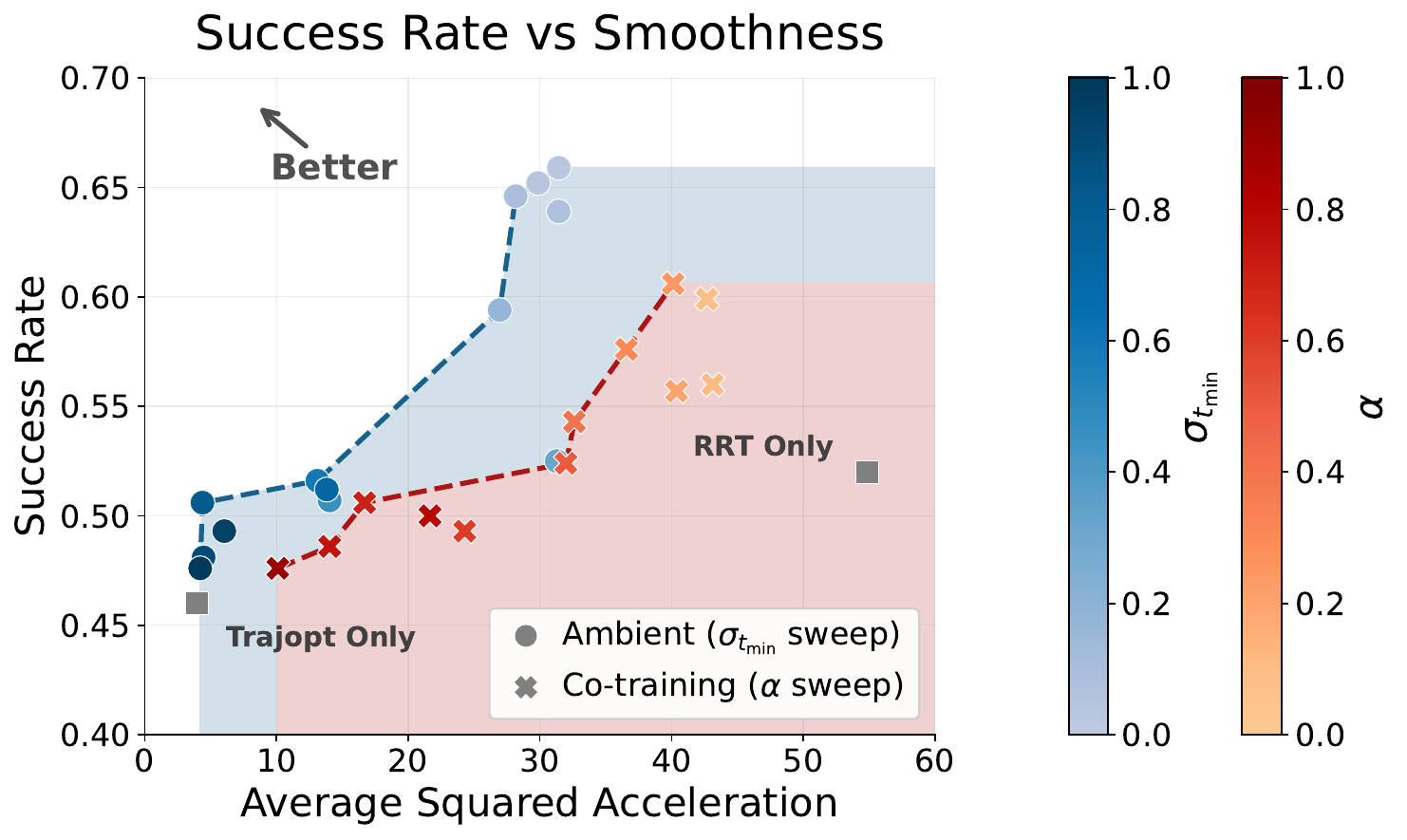}
        \caption{7-DoF neural motion planning.}
        \label{fig:nmp_pareto}
    \end{subfigure}
    \caption{\textbf{The ``smoothness vs success rate'' Pareto frontier for Ambient (blue) dominates co-training (red)} in both the (a) maze and (b) neural motion planning experiments. This trade-off is controlled by $\sigma_{t_{\min}}$ for Ambient and $\alpha$ for co-training.}
    \label{fig:pareto_combined}
\end{figure}

The policy trained with data filtering is smooth, but performs poorly (Table \ref{tab:maze_results}). Co-training and Ambient both perform well ($99.0\%+$ success), but the Ambient Diffusion Policy is 2$\times$ smoother. This is because RRT's jittery behavior is a local property; thus, excluding it from training at low diffusion times prevents the policy from learning its non-smooth behavior. Figure~\ref{fig:more_maze_rollouts} (Appendix~\ref{sec:appendix:maze}) visualizes the rollouts. Lastly, training on $\mathcal D_q$ trades smoothness for better data coverage and success rates. Ambient's Pareto frontier for this trade-off strictly dominates the co-training baseline (Figure~\ref{fig:maze_pareto}).

\textbf{Ablation:} We claim that $p_t\approx q_t$ for $t > t_{\min}$. If this were true, then we should be able to train an effective policy without using $\mathcal{D}_p$ for $t>t_{\min}$. To test this, we train a policy that exclusively trains on $\mathcal{D}_q$ for $t \in (t_{\min}, T]$ and $\mathcal{D}_p$ for $t \in [0, t_{\min}]$. Remarkably, it matches the performance of the Ambient Diffusion Policy, with a success rate of $99.0\%$ and a squared acceleration of $29.9$.

\subsection{Neural Motion Planning}
\label{sec:controlled_experiments:nmp}

We extend the maze experiment to neural motion planning~\cite{dalal2023imitatingtaskmotionplanning, qureshi2019motionplanningnetworks, carvalho2024motionplanningdiffusionlearning} with a 7-DoF robot arm in a cluttered environment (Figure~\ref{fig:teaser}b). The setup and metrics are analogous to the 2D maze, except that start and goal configurations are sampled around objects and shelves, and the planner predicts the entire plan open-loop. Appendix Table~\ref{tab:maze_vs_nmp} compares the two setups. The main results in Table~\ref{tab:nmp_results} and Figure~\ref{fig:nmp_pareto} mirror the 2D maze experiments: Ambient achieves the highest success rate while remaining substantially smoother than co-training. The full experimental details and ablations are in Appendix~\ref{sec:appendix:nmp}.

\textbf{Implications for Neural Motion Planning (NMP).} Most neural motion planners~\cite{dalal2024neuralmpgeneralistneural, qureshi2019motionplanningnetworks} are trained on large datasets of sampling-based trajectories (akin to $\mathcal D_q$) since they are inexpensive to generate. Unfortunately, sampling-based trajectories are non-smooth~\cite{LaValle1998RapidlyexploringRT}; thus, these approaches choose to prioritize data scale over quality. Our results suggest that an NMP trained with Ambient can enjoy the best of both worlds: augmenting the existing sampling-based datasets with a small amount of high-quality data could yield planners that are both general and smooth.

\subsection{Sim-to-Real Distribution Shift}
\label{sec:controlled_experiments:sim2real}

We evaluate Ambient Diffusion Policy on sim-and-real co-training using the planar pushing setup from \cite{wei2025empirical} (Figure~\ref{fig:teaser}c), a canonical task in robot imitation learning. Our experiments use the \textit{sim-and-target} problem from \citet{wei2025empirical} as a controlled proxy for the sim-and-real problem.

$\mathcal{D}_p$ consists of 50 teleoperated demonstrations in the \textit{target} environment; $\mathcal{D}_q$ consists of 2000 trajectories generated by a motion planner \cite{graesdal2024tightconvexrelaxationscontactrich} in the \textit{simulated} environment. We use the same datasets from \citet{wei2025empirical} to enable direct comparisons with their method. A policy rollout is successful if it pushes the T to the target pose within a time limit.

\begin{table}[!t]
    \centering
    \caption{\textbf{Main planar pushing results (200 trials).} All Ambient variations outperform co-training~\cite{wei2025empirical}.}
    \label{tab:sim_and_real_results}
    \scriptsize
    \setlength{\tabcolsep}{4pt}
    \renewcommand{\arraystretch}{1.7}
    \resizebox{\textwidth}{!}{%
    \begin{tabular}{l|ccccc}
        \toprule
        Training Algo. & Data Filtering & Co-training~\cite{wei2025empirical} & \makecell{Ambient\\($t_{\min}$ per dataset)} & \makecell{Ambient\\($t_{\min}$ per datapoint)} & \makecell{Ambient + Locality\\($\sigma_{t_{\max}}=0.025$)} \\
        \midrule
        Success Rate $\uparrow$ & $56.5^{+7.0}_{-7.2}$ \% & $84.5^{+4.7}_{-5.8}$ \% & $87.0^{+4.3}_{-5.5}$ \% & $\mathbf{93.5^{+3.0}_{-4.4}}$ \% & $92.0^{+3.4}_{-4.7}$ \% \\
        \bottomrule
    \end{tabular}}
\end{table}

Table~\ref{tab:sim_and_real_results} reports the success rate for three variations of Ambient Diffusion Policy that annotate $t_{\min}$ at two granularities: \textit{per dataset} (with a parameter sweep) and \textit{per datapoint} (with the classifier approach). We explore intermediate granularities, scaling laws, and more ablations in Appendix~\ref{sec:appendix:sim_and_real}. All Ambient variations outperform the best co-trained policy~\cite{wei2025empirical}. Annotating $t_{\min}$ per datapoint performs the best because even within a single dataset, the utility of each datapoint can vary. As before, Ambient learns from suboptimal simulation data by leveraging the diffusion hierarchy. $\mathcal {D}_p$ and $\mathcal{D}_q$ share the same high-level strategy; thus, $\mathcal D_q$ is safe to use at high noise. Locality ($t_{\max}>0$) does not help since $\mathcal {D}_p$ and $\mathcal{D}_q$ exhibit different low-level contact dynamics.

\subsection{Task Mismatch}
\label{sec:controlled_experiments:bin_sorting}

\begin{table}[!t]
    \centering
    \caption{\textbf{Main block sorting results ($\sim$800 blocks).} Only Ambient Diffusion Policy (Locality) performs well on all metrics.}
    \label{tab:bin_sorting_results}
    \scriptsize
    \newcommand{\rowstrut}{\rule[-5pt]{0pt}{14pt}}%
    \begin{tabular*}{\textwidth}{@{\extracolsep{\fill}}l|cccc@{}}
        \toprule
        \makecell{Training Algo.} & \makecell{Data Filtering} & \makecell{$\mathcal{D}_q$ Only} & \makecell{Co-train\\($\alpha^*=0.9$)} & \makecell{Locality\\($\sigma_{t_{\max}^*}=0.46$)} \\
        \midrule
        \rowstrut Logic $\uparrow$        & $\mathbf{98.6^{+0.8}_{-1.5}}$ \% & $3.0^{+1.6}_{-1.2}$ \% & $26.0^{+3.4}_{-3.2}$ \% & $98.2^{+0.8}_{-1.2}$ \% \\
        \rowstrut Motion $\uparrow$       & $61.9^{+3.4}_{-3.5}$ \% & $83.0^{+2.5}_{-2.8}$ \% & $87.2^{+2.2}_{-2.5}$ \% & $\mathbf{95.0^{+1.4}_{-1.7}}$ \% \\
        \midrule
        \rowstrut Success Rate $\uparrow$ & $61.0^{+3.4}_{-3.5}$ \% & $2.5^{+1.3}_{-1.0}$ \% & $22.7^{+3.1}_{-2.9}$ \% & $\mathbf{93.3^{+1.6}_{-2.0}}$ \% \\
        \bottomrule
    \end{tabular*}
\end{table}

Ambient can leverage ``locality'' to learn local action primitives (e.g., grasping) from data collected on a different task. Consider the block sorting task in Figure~\ref{fig:teaser}d. The robot must sort red and blue blocks into the left and right bins, respectively. This requires two distinct skills: \textit{motion} (grasping and placing blocks) and \textit{logic} (selecting the correct bin). We design a metric to evaluate each skill:

\begin{itemize}[leftmargin=2em, topsep=1pt, itemsep=1pt, parsep=0pt]
    \item \textbf{Motion metric:} $\frac{\text{\# blocks placed}}{\text{\# total blocks}}$. This metric is independent of logical reasoning because the numerator counts any placed block, regardless of logical correctness.
    \item \textbf{Logic metric:} $\frac{\text{\# blocks placed correctly}}{\text{\# blocks placed}}$. This metric is independent of pick-and-place ability because the denominator only counts successfully placed blocks.
\end{itemize}
The overall success rate is $\frac{\text{\# blocks placed correctly}}{\text{\# total blocks}}$, the product of the two metrics. $\mathcal{D}_p$ contains 50 demonstrations with the correct sorting. $\mathcal{D}_q$ contains 200 demonstrations with the opposite sorting. To isolate the effect of locality, we set $t_{\min}=0$ and sweep $t_{\max}$ for $\mathcal D_q$.

Table \ref{tab:bin_sorting_results} presents the results. Co-training performs poorly because training on $\mathcal{D}_q$ at all diffusion times teaches the policy both the useful motion primitives and the incorrect logic. Ambient Diffusion Policy solves this problem by restricting $\mathcal{D}_q$ to $t\in[0, t_{\max})$. As illustrated in Figures~\ref{fig:visualize_backward} and~\ref{fig:bin_sorting_combined}, the policy learns motion primitives at low diffusion times without attending to the (incorrect) task-level structure of the demonstrations. This results in near-perfect scores on both metrics simultaneously.

\begin{figure}[!t]
    \centering
    \includegraphics[width=0.49\linewidth]{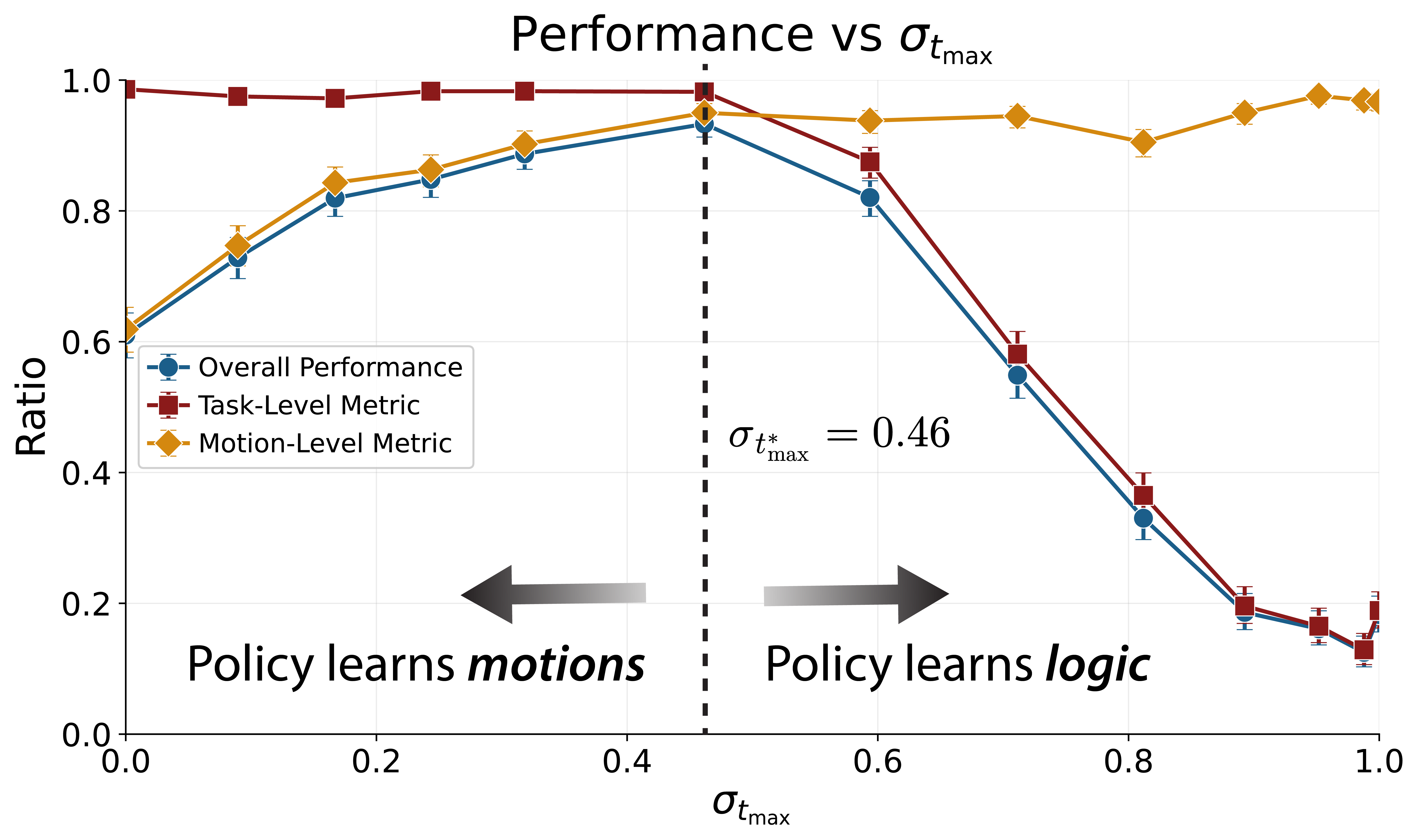}\hfill
    \includegraphics[width=0.49\linewidth]{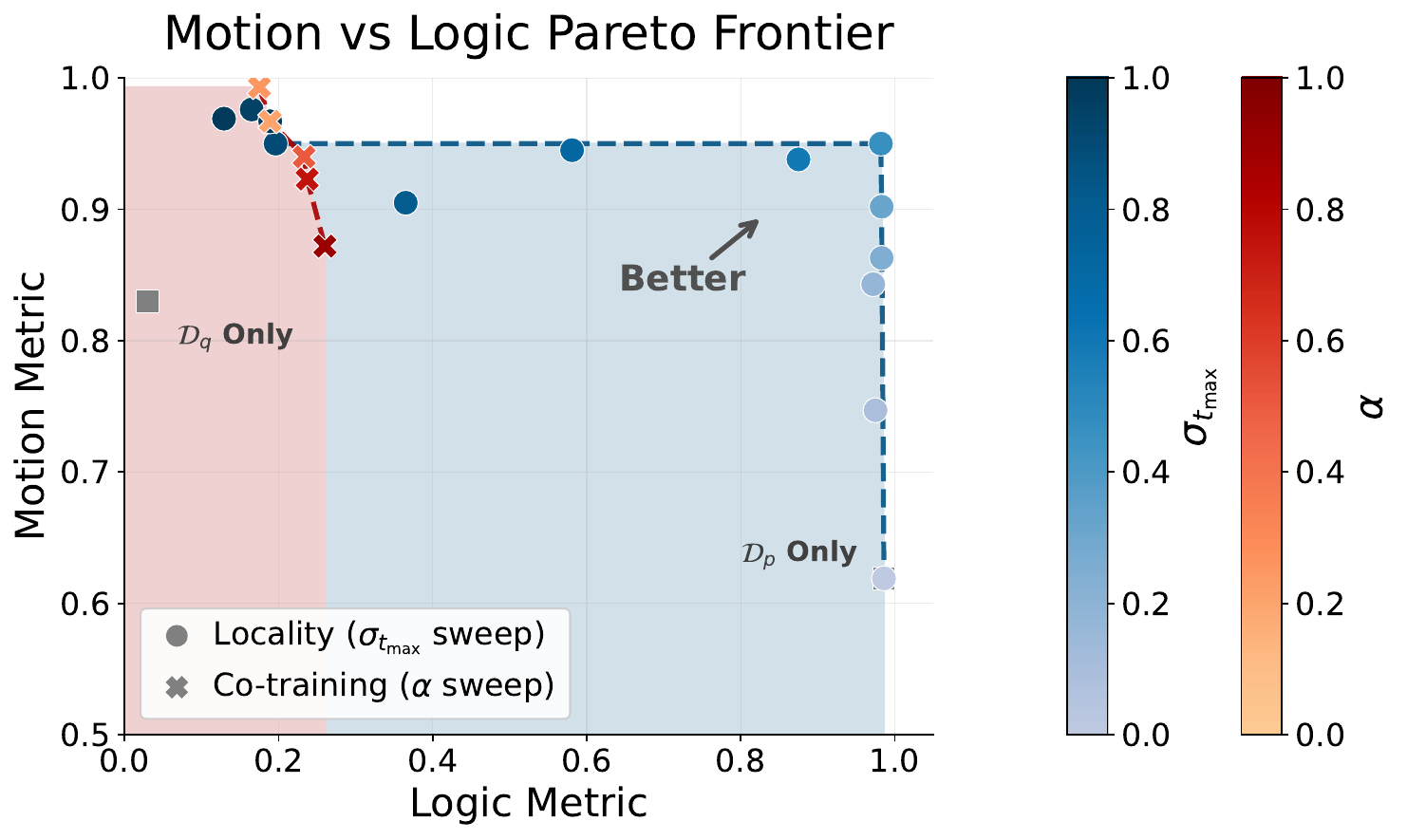}
    \caption{\textbf{Left:} \textbf{Performance metrics vs $\sigma_{t_{\max}}.$} The logic metric plateaus before $\sigma_{t_{\max}^*}=0.46$ and deteriorates rapidly afterwards; the opposite is true for the motion metric. Thus, the policy learns \textit{local} motion-level features for $t\in[0, t_{\max}^*)$ and \textit{global} logic-level features for $t\in(t_{\max}^*, T]$. \textbf{Right:} Unlike co-training, Ambient achieves a near-optimal trade-off between logic and motion metrics.}
    \label{fig:bin_sorting_combined}
\end{figure}

\textbf{Ablations:} We present two ablations in Appendix~\ref{sec:appendix:block_sorting}. \textbf{1)} we add a one-hot encoding indicating the sorting direction to the policies. Ambient still outperforms the co-training. \textbf{2)} we train a policy that exclusively uses $\mathcal{D}_q$ for $t \in [0, t_{\max})$ and $\mathcal{D}_p$ for $t \in[t_{\max}, T]$. This is the \textit{locality}-version of the ablation in the 2D maze experiment. The resulting policy achieves a logic and motion score of $97.9\%$ and $93.8\%$, reinforcing our claim that action diffusion is hierarchical.

\subsection{Finetuning Comparison}
\label{sec:controlled_experiments:finetuning}
\begin{figure}[!t]
    \centering
    \includegraphics[width=\linewidth]{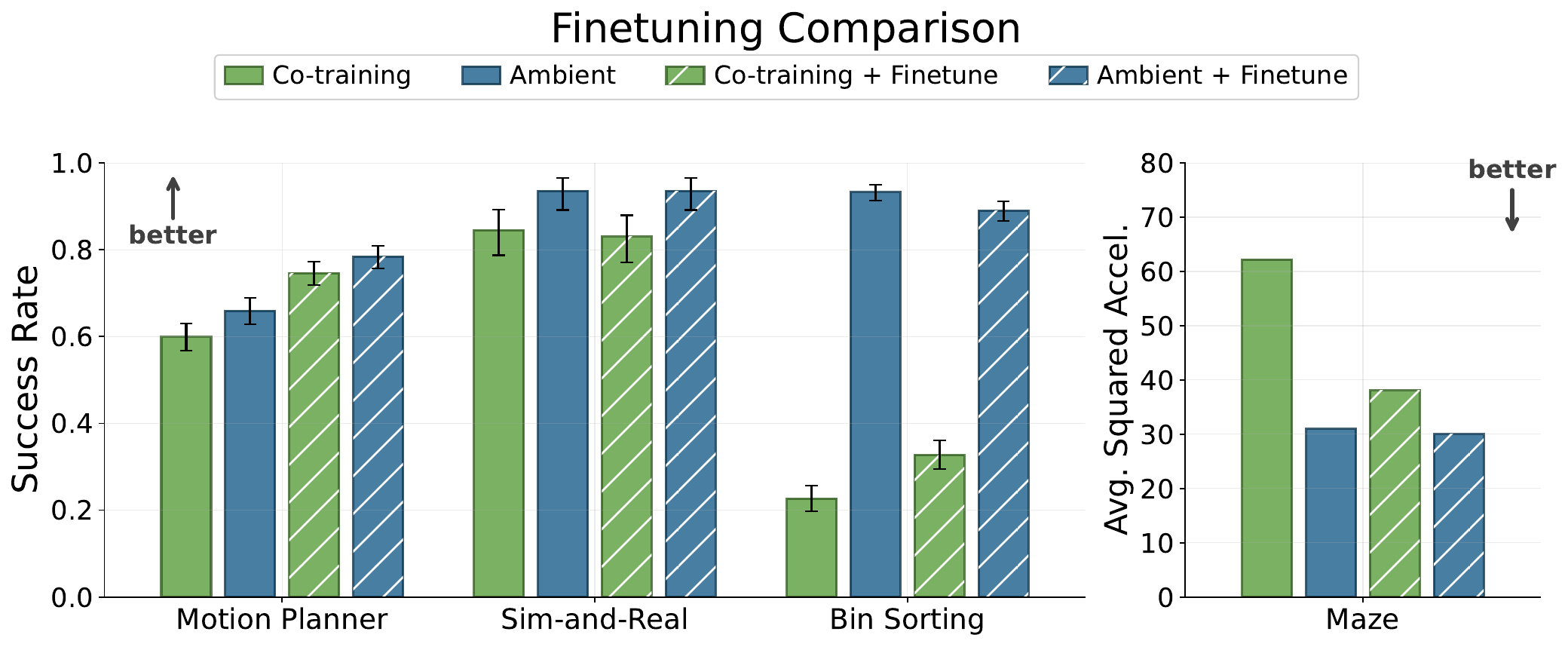}
    \caption{\textbf{Left:} success rate of policies finetuned from different base policies. \textbf{Right:} smoothness of finetuned policies in the maze experiment. We compare smoothness instead of success rate since all policies achieved 99.0\%+ success rate.}
    \label{fig:finetuning_comparison}
\end{figure}
Finally, we compare against the finetuning baseline. For each experiment, we finetune the best co-trained policy and the best Ambient policy on $\mathcal D_p$; implementation details are in Appendix~\ref{sec:appendix:finetuning_implementation}.

Figure~\ref{fig:finetuning_comparison} suggests three findings. First, finetuning an Ambient base policy consistently outperforms finetuning a co-trained base policy. Second, the Ambient \textit{base} policy can often outperform the \textit{finetuned} co-trained policy. Third, finetuning an Ambient base policy does not always help. We hypothesize that the Ambient policies already ignore undesirable features in the suboptimal data.

\section{Scaling Experiments: Open X-Embodiment}
\label{sec:oxe_experiments}
Having shown the generality of Ambient Diffusion Policy, we now scale to Open X-Embodiment (OXE)~\cite{open_x_embodiment_rt_x_2023}. OXE is a large collection of real-world data with heterogeneous quality and unstructured distribution shifts. Our goal is to extract useful learning signal from as much of this ``suboptimal'' data as possible. The experiments evaluate Ambient Diffusion Policy on two real-world tasks.

\textbf{1) Table Cleaning} (Figure~\ref{fig:teaser}e): the robot must open the drawer, place objects inside, and close the drawer. The diversity of the test-time objects (unseen during training) makes this challenging. Performance is measured with the task completion rubric in Table~\ref{tab:table_cleaning_rubric}.

\textbf{2) Tower Building} (Figure~\ref{fig:teaser}f): the robot must stack blocks in alternating directions and colors as high as possible. This task requires precise manipulation since placement errors compound as the tower grows. Performance is measured by the number of blocks placed before toppling the tower.

The training procedure for both tasks is identical. $\mathcal D_p$ consists of a small number of demonstrations for each task. $\mathcal D_q$ is one of two subsets of the Open X-Embodiment (OXE) dataset: \textbf{1) Magic Soup++ (MS++)}, 27 datasets from OXE used by OpenVLA \cite{kim2024openvla}, or \textbf{2) Custom OXE (COXE)}, 48 datasets from OXE that contain MS++ as a subset. $t_{\min}$ and $t_{\max}$ are annotated according to Phase 1 of the method. We note that COXE was curated to be as inclusive as possible; details are in Appendix~\ref{sec:appendix:oxe}.

Lastly, for both co-training and Ambient, we re-weight each dataset's sampling probability at every diffusion time. Let $w_i$ denote the weight for dataset $i$, and $S_t$ denote the indices of datasets that are permissible at diffusion time $t$. The sampling probability of dataset $i$ at diffusion time $t$ is
\begin{equation}
    \tilde{w}_t^{(i)} = \begin{cases}
        \dfrac{w_i}{\sum_{j\in S_t} w_j} & \text{if } i \in S_t \\
        \quad \quad 0 & \text{otherwise}
    \end{cases}.
\label{eq:reweighting}
\end{equation}
In co-training, $S_t$ contains all datasets for all $t\in[0,T]$; in Ambient, $S_t$ depends on each dataset's $t_{\min}$ and $t_{\max}$ (see Figures~\ref{fig:weighting_visualization} and~\ref{fig:tower_building_weighting_visualization}). Re-weighting marginally improves the Ambient policies but is not required; in contrast, it is \textit{necessary} for co-training (see Section~\ref{sec:oxe_experiments:ablations}).

\subsection{Main Results}
\label{sec:oxe_experiments:main_results}

\textbf{Table Cleaning ($\mathcal D_p$ with 50 demos):}
Data filtering achieves $68.2\%$ task completion; co-training adds just 2-3\% and plateaus when $\mathcal D_q$ scales from MS++ to COXE. On the other hand, the best Ambient Diffusion Policy outperforms co-training by $12\%$ and improves when $\mathcal D_q$ scales from MS++ to COXE. Adding locality ($t_{\max} > 0$) appears to provide further gains.

\textbf{Table Cleaning ($\mathcal{D}_p$ with 150 demos):}
When $\mathcal D_p$ is large, data filtering improves to $80.1\%$, and training on suboptimal data provides less relative value. However, for many real-world tasks, collecting enough data in $\mathcal D_p$ could be prohibitively expensive \cite{zhao2024alohaunleashedsimplerecipe}. Nonetheless, Ambient continues to outperform co-training by up to $10\%$. Adding locality did not improve performance in this setting. We hypothesize that when $\mathcal{D}_p$ is large, the policy can learn the relevant motion primitives from $\mathcal{D}_p$, making data from $\mathcal{D}_q$ unnecessary or even harmful with the wrong choice of $t_{\max}$. This is consistent with the classifier annotation method, which would have assigned $t_{\max}=0$ to nearly all datasets.

\textbf{Tower Building ($\mathcal{D}_p$ with 35 demos):} On average, the Ambient Diffusion Policies build up to $84\%$ and $33\%$ taller than the data filtering and co-training baselines, respectively (Figure~\ref{fig:oxe_results}). Including $\mathcal{D}_q$ at low diffusion times appears to improve performance.

\begin{figure*}
    \centering
    \includegraphics[width=\linewidth]{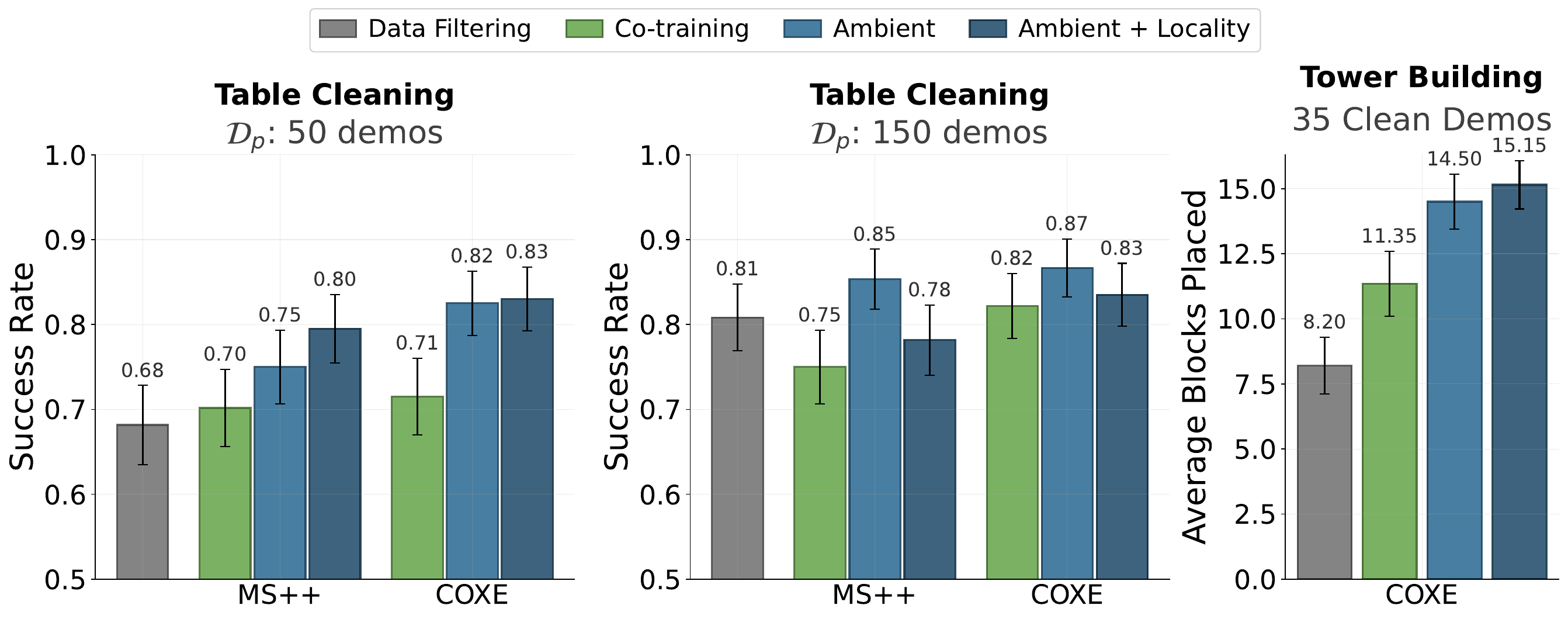}
    \caption{\textbf{When scaled to OXE, Ambient Diffusion Policy outperforms data filtering and co-training} by up to 15\% on table cleaning and 84\% on tower building (20 trials per policy).}
    \label{fig:oxe_results}
\end{figure*}

\subsection{Ablations}
\label{sec:oxe_experiments:ablations}

\begin{figure}[!t]
    \centering
    \includegraphics[width=\linewidth]{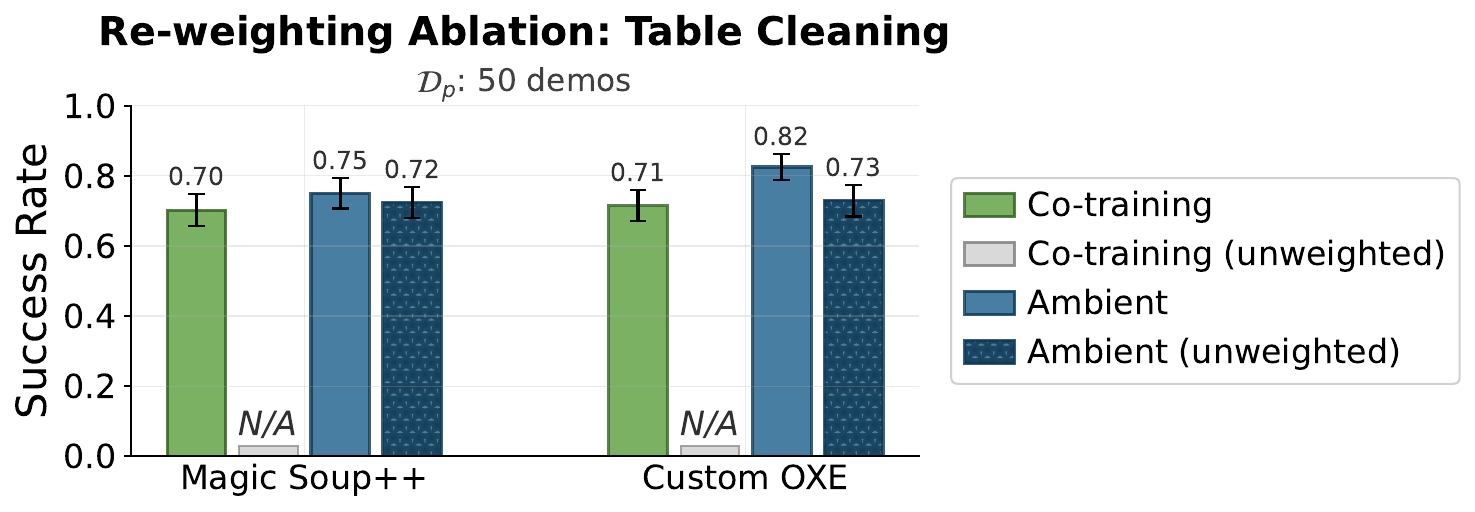}
    \caption{\textbf{Ambient is significantly less sensitive to dataset re-weighting than co-training.} The Ambient Diffusion Policy performed at most 9\% worse without re-weighting. The co-trained policies were too dangerous to evaluate without re-weighting.}
    \label{fig:table_cleaning_weighting_ablation}
\end{figure}

\textbf{Dataset Re-weighting:} Figure~\ref{fig:table_cleaning_weighting_ablation} ablates the effect of re-weighting in Eq.~\eqref{eq:reweighting} on co-training and Ambient. Re-weighting improves policy performance on the table cleaning task by $3$--$9\%$. On the other hand, re-weighting is \textit{necessary} for co-training: the unweighted co-trained policies were too unsafe to evaluate on hardware. This contrast highlights another practical advantage of Ambient: the training algorithm is inherently less sensitive to dataset weights because $\mathcal{D}_q$ is only used when $p_t \approx q_t$. Co-training uses $\mathcal{D}_q$ at all diffusion times, making the weights critical and fragile.

\textbf{Additional Ablations:} We present two further ablations in Appendix~\ref{sec:appendix:oxe_ablations}. \textbf{1)} Finetuning the OXE policies on $\mathcal{D}_p$ produced no statistically significant change in performance (Appendix~\ref{sec:appendix:finetuning_oxe}), consistent with the controlled experiments (Section~\ref{sec:controlled_experiments:finetuning}). \textbf{2)} On table cleaning, the Ambient Diffusion Policies require 25-40\% fewer grasps and less time per object than their co-trained counterparts.

\section{Limitations and Future Work}
\label{sec:limitations}
Ambient Diffusion Policy does not yet handle distribution shifts in the observation space. We explore two initial attempts in Appendix~\ref{sec:appendix:obs_shift}. The classifier-based annotation method for $t_{\min}$ and $t_{\max}$ is theoretically grounded, but does not always outperform a (costly) hyperparameter sweep. Better annotation methods could improve performance. In addition, the theory in Section~\ref{sec:theory} is limited to Gaussians; extending the results to a broader class of distributions could provide further insights. Lastly, our method is more computationally expensive than data filtering since it trains on more data, but in most cases, its benefits are worthwhile.

Ambient Diffusion Policy requires the user to partition their data into $\mathcal{D}_p$ and $\mathcal{D}_q$. The algorithm itself poses no restrictions on this partitioning, but it is an important design decision for downstream performance. This paper defines $\mathcal{D}_p$ as data from the target task and environment, but the right choice depends on the goal. For task-specific training or finetuning, our definition is appropriate; for pretraining a generalist policy, prior work suggests that $\mathcal{D}_p$ should contain diverse and high-quality demonstrations~\cite{hu2025datascalinglawsimitation, belkhale2023dataqualityimitationlearning}.
The former suggests that we could expand our finetuning dataset to include suboptimal data. The latter requires a more principled understanding of data quality in robotics \cite{belkhale2023dataqualityimitationlearning, mandlekar2021matterslearningofflinehuman}.
Both settings are exciting directions for future work.

\section{Conclusion}
\label{sec:conclusion}

High-quality robot data is scarce, making it essential to learn from suboptimal data sources. But even as data collection scales, suboptimal data will continue to grow alongside high-quality data. Thus, the problem of learning from suboptimal data sources is not a temporary artifact of today's data scarcity, but rather a fundamental and persistent challenge in robot imitation learning. The dominant approach for addressing this problem is to filter the worst data and co-train on the rest via re-weighting; but filtering is wasteful, and co-training teaches a policy the meaningful and the harmful parts of suboptimal distributions. In this paper, we propose Ambient Diffusion Policy, a principled method for training on suboptimal demonstrations.

Neighboring fields have recognized the power of noise-dependent data usage \cite{daras2025ambient, proteins}. Our key experimental and theoretical insight is that this training paradigm is also well-suited for robotics. Our method makes no assumptions about the nature of the action suboptimality, outperforms existing baselines, and greatly expands the set of useful data sources for robot imitation learning. With Ambient Diffusion Policy, the question shifts from \textit{which} data sources should be used, to \textit{when} each should be used in the diffusion process.

\clearpage
\acknowledgments{We thank Pablo Parrilo, Asuman Özdağlar, Abhinav Agarwal, Cole Becker, Ananth Kumar Srinivasan, Sarthak Kaingade, Max Wilcoxson, and the Robot Locomotion Group for fruitful discussions. We thank 
John Marangola for help with the hardware setup. 

This work was supported by the National Science Foundation Graduate Research Fellowship Program under Grant No. 2141064, and the Natural Sciences and Engineering Research Council of Canada CGS D-587703. Any opinions, findings, and conclusions or recommendations expressed in this material are those of the author(s) and do not necessarily reflect the views of the National Science Foundation. We thank MIT Supercloud \cite{reuther2018interactive} for providing computational resources.
}

\bibliography{references}

\clearpage
\appendix
\label{sec:appendix}

\addtocontents{toc}{\protect\setcounter{tocdepth}{2}}
\renewcommand{\contentsname}{Appendix: Table of Contents}
\tableofcontents
\clearpage

\section{Ambient Diffusion: Theoretical Foundations}
\label{sec:appendix:theory}

This appendix provides the proofs of the theoretical results stated in Section~\ref{sec:theory}: the spectral power law implies both \textit{fast contraction} through noise (Theorem~\ref{th:power_law_implies_contraction}) and \textit{locality} (Theorem~\ref{th:locality}) of the optimal denoiser. We then formally justify the classifier-based $t_{\min}$ annotation method from Phase 1 of our algorithm (Theorem~\ref{th:approximate-classifier-threshold}). The Spectral Power Law (Definition~\ref{def:spectral_power_law}) and the general TV-contraction bound of Eq.~\eqref{eq:tv_contraction_bound} are stated in Section~\ref{sec:theory}.

\paragraph{Preliminaries.}
Given distributions $P, Q$, we recall the following information-theoretic divergences, used throughout this appendix:
\begin{align*}
    \mathrm{d}_{\mathrm{TV}}(P, Q) &= \mathbb{E}_{Q}\!\left|1 - \frac{dP}{dQ}\right|, \\
    \chi^2(P \| Q) &= \mathbb{E}_{Q}\!\left(\frac{dP}{dQ}\right)^2 - 1, \\
    \mathrm{LC}(P, Q) &= 2 \chi^2\!\left(P \,\big\|\, \tfrac{1}{2}P + \tfrac{1}{2}Q\right).
\end{align*}
Pinsker's inequality gives $\mathrm{d}_{\mathrm{TV}}(P, Q)^2 \leq \tfrac{1}{2}\mathrm{KL}(P, Q)$, and the LeCam inequality gives $\mathrm{d}_{\mathrm{TV}}(P, Q)^2 \leq \tfrac{1}{2}\mathrm{LC}(P, Q)$. We refer the reader to~\cite{polyanskiy2025information} for further background.

\subsection{Spectral Power Law implies fast contraction through noise}

\begin{proof}[Proof of Theorem~\ref{th:power_law_implies_contraction}]
    We start by making the observation that since the Discrete Fourier Transform (DFT) is an invertible operation, it holds that:
    \begin{gather}
        \mathrm{d}_{\mathrm{KL}}(p_t \| q_t) = \mathrm{d}_{\mathrm{KL}}(\mathcal F \# p_t \| \mathcal F \# q_t).
    \label{eq:kl_does_not_change}
    \end{gather}
    We denote with $F$ the DFT-matrix. Multiplying with $F$ vectors that are distributed according to a zero-mean Gaussian gives another zero-mean Gaussian distribution with covariance:
    \begin{gather}
        \Sigma_{\mathcal F \# p_t} = \mathbb E_{x \sim p_t}\left[ (F x) (F x)^* \right] \\
        = F \mathbb E_{x \sim p_t}[xx^T] F^* \\
        = F \Sigma_{p_t} F^*.
    \end{gather}

    We have that $\Sigma_{p_t} = \Sigma_{p} + \sigma^2_t I$, where $\Sigma_p$ is a circulant matrix (stationarity assumption). Hence, $\Sigma_{p_t}$ is also a circulant matrix, and thus it can be diagonalized as follows:
    \begin{gather}
        \Sigma_{p_t} = F^* (\Lambda_p + \sigma^2_t I)F.
    \end{gather}
    Hence,
        \begin{gather}
        \Sigma_{\mathcal F \# p_t} = F F^* (\Lambda_p + \sigma^2_t I)F F^* \\
        = \Lambda_p + \sigma^2_t I.
    \end{gather}
    Similarly, $\Sigma_{\mathcal F \# q_t} = \Lambda_q + \sigma^2_t I$.

    Using the formula for the KL distance between two zero-mean, diagonal, multivariate Gaussians, we get:
    \begin{gather}
        \mathrm{d}_{\mathrm{KL}}(p_t \| q_t) \leq \frac{1}{2}\sum_{i=0}^{N-1}\frac{(\Lambda_p)_{i, i} + \sigma_t^2}{(\Lambda_q)_{i, i} + \sigma_t^2} - \mathrm{ln}\left(\frac{(\Lambda_p)_{i, i} + \sigma_t^2}{(\Lambda_q)_{i, i} + \sigma_t^2} \right) - 1 \\
        = \frac{1}{2}\sum_{f=0}^{N-1}\frac{S_p(f) + \sigma_t^2}{S_q(f) + \sigma_t^2} - \mathrm{ln}\left(\frac{S_p(f) + \sigma_t^2}{S_q(f) + \sigma_t^2} \right) - 1 \\
        \leq \sum_{f > f^\star}^{\lfloor N/2 \rfloor}\frac{S_p(f) + \sigma_t^2}{S_q(f) + \sigma_t^2} - \mathrm{ln}\left(\frac{S_p(f) + \sigma_t^2}{S_q(f) + \sigma_t^2} \right) - 1,
        \label{eq:polynomial_bound_incoming}
    \end{gather}
    where in the last inequality we used the fact that the spectrum is symmetric around the Nyquist frequency and that $p$ and $q$ agree on all frequencies up to $f^\star$.

    Let's consider the function: $g(r) = r - \ln r - 1$. We will start by showing that $g(r) \leq (r-1)^2$ for all $r\in [1/2, \infty)$.

    Let 
    \begin{gather}
       h(r) = g(r) - (r-1)^2 \\
       = r  - \ln r - 1 -r^2 +2r -1 \\
       = -r^2  + 3r - \ln r -2.
    \end{gather}
    We hence have:
    \begin{gather}
        h'(r) = -2r + 3 - \frac{1}{r} \\
        = \frac{-2r^2 + 3r - 1}{r}.
    \end{gather}
    The determinant of the nominator is $\Delta = 1$ and there are two points that make $h'(r) = 0$, which are:
    \begin{gather}
        r_{1} = \frac{1}{2}, \quad r_2 = 1.
    \end{gather}
    We have that $h'(r)$ is negative for $r > 1$ and for $r \in (0, 1/2)$. Hence, in $[1/2, \infty)$ the maximum of $h(r)$ is $h(1) = 0$. Hence, the function $h(r)$ is $\leq 0$ for all $r \in [1/2, \infty)$.

    We now want to get some bounds on the quantity $\frac{S_p(f) + \sigma_t^2}{S_q(f) + \sigma_t^2}$. First, note that:
    \begin{gather}
        \frac{S_p(f) + \sigma_t^2}{S_q(f) + \sigma_t^2} \geq \frac{\sigma_t^2}{\max(S_p(f), S_q(f)) + \sigma_t^2}.
    \end{gather}
    For all $\lfloor N/2 \rfloor \geq f > f^\star$, we further have (by Power-law tail assumption) that:
    \begin{gather}
        \max(S_p(f), S_q(f)) \le C f^{-\alpha},
    \end{gather}
    which gives that:
    \begin{gather}
         \frac{S_p(f) + \sigma_t^2}{S_q(f) + \sigma_t^2} \geq \frac{\sigma_t^2}{Cf^{-\alpha} + \sigma_t^2}.
         \label{eq:partial_bound}
    \end{gather}

    By assumption, we have that:
    \begin{gather}
        \sigma_t^2 \ge  C \frac{1}{(f^\star)^{\alpha}}.
    \end{gather}
    But:
    \begin{gather}
        \frac{C}{f^\alpha} < \frac{C}{(f^\star)^\alpha}, \quad \forall f \text{ such that } f^\star < f \leq \lfloor N/2 \rfloor.
    \end{gather}
    Hence,
    \begin{gather}
          Cf^{-\alpha} < \sigma_t^2.
    \end{gather}
    Plugging this into Eq.~\eqref{eq:partial_bound}, we get:
    \begin{gather}
         \frac{S_p(f) + \sigma_t^2}{S_q(f) + \sigma_t^2} \geq 1/2.
         \label{eq:bound_on_r}
    \end{gather}

    By combining Eq.~\ref{eq:bound_on_r} and the fact that $g(r) \leq (r-1)^2$ for all $r\in [1/2, \infty)$, we can write Eq.~\eqref{eq:polynomial_bound_incoming} as follows:
    \begin{gather}
        \mathrm{d}_{\mathrm{KL}}(p_t \| q_t) \leq \sum_{f > f^\star}^{\lfloor N/2 \rfloor}\left(\frac{S_p(f) - S_q(f)}{S_q(f) + \sigma_t^2}\right)^2 \\
         \leq \sum_{f > f^\star}^{\lfloor N/2 \rfloor}\left(\frac{2 \max(S_p(f), S_q(f))}{\sigma_t^2}\right)^2 \\
         \leq  \sum_{f > f^\star}^{\lfloor N/2 \rfloor} 4 \frac{C^2f^{-2\alpha}}{\sigma_t^4} \\
         = \frac{4C^2}{\sigma_t^4} \sum_{f > f^\star}^{\lfloor N/2 \rfloor} \left(\frac{1}{f}\right)^{2\alpha} \\
         \leq \frac{4C^2}{\sigma_t^4} \int_{f^\star}^{\infty} \left(\frac{1}{f}\right)^{2\alpha} \mathrm{d}f \\
         = \frac{4C^2}{\sigma_t^4 (2\alpha - 1) (f^\star)^{2\alpha - 1}}. 
    \end{gather}

    The TV distance bound follows from a direct application of Pinsker's inequality.
\end{proof}

\subsection{Spectral Power Law implies locality}
\begin{proof}[Proof of Theorem~\ref{th:locality}]
    Since $X_0$ is distributed according to a Gaussian distribution, the MMSE optimal denoiser is a linear function, given as:
    \begin{gather}
        \mathbb E[X_0 | X_t=x] = \Sigma ( \Sigma + \sigma_t^2I)^{-1}x.
    \end{gather}

    From the stationarity assumption, $\Sigma$ is a circulant matrix and so is the matrix $\Sigma + \sigma_t^2I$. The inverse of a circulant matrix is another circulant matrix, and multiplying two circulant matrices gives another circulant matrix. Hence, the whole matrix $ \Sigma ( \Sigma + \sigma_t^2I)^{-1}$ is circulant.

    Multiplication with a circulant matrix can be written as a circulant convolution; hence, we can express the optimal denoiser in the following equivalent way:
    \begin{gather}
        (\mathbb E[X_0 | X_t=x])_i = \sum_{j=0}^{N-1} k_t((i - j) \mod N) x_j,
    \end{gather}
    where $k_t$ is the first column of the matrix $\Sigma ( \Sigma + \sigma_t^2I)^{-1}$.
    Let us now analyze how the prediction of the optimal denoiser changes when, instead of feeding the full input, we input a masked version $M_ix$ of the input where $M_i$ is a circulant mask.

    We have that:
    \begin{gather}
        \mathrm{error}_i = \bigg|\sum_{j=0}^{N-1} k_t((i - j) \mod N) x_j \nonumber \\ - \sum_{j=0}^{N-1} k_t((i - j) \mod N) M_ix_j\bigg| \\
        = \bigg|\sum_{j: \mathrm{dist}(i, j) > L} k_t((i - j) \mod N) x_j\bigg|, \\
        \leq ||x||_\infty  \sum_{j: \mathrm{dist}(i, j) > L}\bigg| k_t((i - j) \mod N) \bigg|,
    \end{gather}
    where $ \mathrm{dist}(i, j) = \mathrm{min}(|i-j|, N - |i-j|)$. Let $m = (i-j) \mod N$. First note that $ L + 1 \leq m \leq  \lfloor N/2 \rfloor$. Second, note that for each value $m$, there are two indices $j$ that can give this value. Hence, we can rewrite the sum as:
    \begin{gather}
         \mathrm{error}_i \leq ||x||_\infty \sum_{m=L+1}^{\lfloor N/2 \rfloor} |k_t(m)| + |k_t(N-m)|.
    \end{gather}

    We now need to analyze what the kernel $k_t$ is. We have defined this vector as the first column of the matrix $\Sigma ( \Sigma + \sigma_t^2I)^{-1}$. Because the matrix is circulant, we have that $k_t(m) = k_t(N-m)$. Hence,
    \begin{gather}
         \mathrm{error}_i \leq 2||x||_\infty \sum_{m=L+1}^{\lfloor N/2 \rfloor} |k_t(m)|.
    \end{gather}    
    By definition:
    \begin{gather}
        k_t(m) = (\Sigma ( \Sigma + \sigma_t^2I)^{-1} e_1)(m).
    \end{gather}
    
    The matrix $\Sigma ( \Sigma + \sigma_t^2I)^{-1}$ is circulant, and hence it is diagonalized by the DFT. Hence, we have:
    \begin{gather}
        k_t(m) = (F^* W F e_1)(m),
    \end{gather}
    where $W$ is the diagonal matrix with diagonal entries $W_{f,f} = \frac{S(f)}{S(f) + \sigma_t^2}$. 

    We observe now that $k_t(m)$ is exactly the Inverse Discrete Fourier Transform of the Wiener filter. Hence, $k_t(m)$ can be equivalently written as:
    \begin{gather}
        k_t(m) = \frac{1}{N} \sum_{f=0}^{N-1} W_{f, f}e^{i \frac{2\pi f m}{N}}.
    \end{gather}
    Because the original signal $X_0$ is real-valued, its power spectrum and the resulting Wiener filter $W_f = W_{f,f}$ are symmetric around the Nyquist frequency, meaning $W_{N-f} = W_f$. This conjugate symmetry perfectly cancels the imaginary sine components, allowing us to write the kernel entirely in terms of real cosines:
    \begin{gather}
        k_t(m) = \frac{1}{N} \sum_{f=0}^{N-1} W_f \cos\left(\frac{2\pi f m}{N}\right).
    \end{gather}

    To extract the spatial decay of this kernel as $m$ increases, let $\omega_m = \frac{2\pi m}{N}$. We multiply both sides of the equation by the phase-shifting factor $2(1 - \cos \omega_m)$:
    \begin{gather}
        2(1 - \cos \omega_m) k_t(m) \nonumber \\
        = \frac{1}{N} \sum_{f=0}^{N-1} W_f \left[ 2 \cos(f \omega_m) - 2 \cos(f \omega_m) \cos(\omega_m) \right].
    \end{gather}
    Applying the standard trigonometric identity $2 \cos(A)\cos(B) = \cos(A+B) + \cos(A-B)$ to the second term gives:
    \begin{gather}
        2(1 - \cos \omega_m) k_t(m) \nonumber \\ = \frac{1}{N} \sum_{f=0}^{N-1} W_f \Big[ 2 \cos(f \omega_m) - \cos((f+1)\omega_m) \nonumber \\ - \cos((f-1)\omega_m) \Big].
    \end{gather}
    Because the discrete spectrum is periodic ($W_N = W_0$) and the cosine function is periodic over $N$, we can shift the indices of summation for the $(f+1)$ and $(f-1)$ terms to group them by the common $\cos(f \omega_m)$ factor. This perfectly regroups the sum into the central second difference $\Delta^2 W_f = W_{f+1} - 2W_f + W_{f-1}$:
    \begin{gather}
        2(1 - \cos \omega_m) k_t(m) = - \frac{1}{N} \sum_{f=0}^{N-1} (\Delta^2 W_f) \cos(f \omega_m).
    \end{gather}

    We now take the absolute value of both sides. For the left side, we apply the half-angle trigonometric identity $2(1 - \cos \omega_m) = 4 \sin^2\left(\frac{\omega_m}{2}\right) = 4 \sin^2\left(\frac{\pi m}{N}\right)$. Since $|\cos(f \omega_m)| \le 1$, it drops out of the right side:
    \begin{gather}
        4 \sin^2\left(\frac{\pi m}{N}\right) |k_t(m)| \le \frac{1}{N} \sum_{f=0}^{N-1} |\Delta^2 W_f|.
    \end{gather}
    
    Crucially, the assumption that $S(f)$ follows a spectral power law dictates the structural behavior of this Wiener filter. For high frequencies $f \ge 1$, the power law $S(f) = C f^{-\alpha}$ ensures that $W_f$ monotonically and smoothly decays to zero. We can bound the discrete slope of this filter by analyzing its continuous extension $W(f) = \frac{C}{C + \sigma_t^2 f^\alpha}$. Its derivative is:
    \begin{gather}
        W'(f) = \frac{-\alpha C \sigma_t^2 f^{\alpha-1}}{(C + \sigma_t^2 f^\alpha)^2}.
    \end{gather}
    By the Arithmetic Mean - Geometric Mean inequality, $(C + \sigma_t^2 f^\alpha)^2 \ge 4 C \sigma_t^2 f^\alpha$. Substituting this into the derivative gives a strict bound on the slope:
    \begin{gather}
        |W'(f)| \le \frac{\alpha C \sigma_t^2 f^{\alpha-1}}{4 C \sigma_t^2 f^\alpha} = \frac{\alpha}{4f} \le \frac{\alpha}{4}.
    \end{gather}
    The sequence of discrete first differences $\Delta W_f$ samples this continuous slope. Over the positive half of the spectrum, the slope starts near zero, reaches a maximum magnitude bounded by $\alpha/4$, and gradually returns to zero. Because the function is convex in the tail, the second differences maintain a constant sign, allowing the absolute sum to telescope perfectly. The total absolute second variation for the positive half of the spectrum is therefore exactly bounded by twice the maximum slope, $2(\alpha/4) = \alpha/2$. Due to the conjugate symmetry of the spectrum, the negative frequencies contribute an identical amount. 
    
    Therefore, the total absolute second variation over the entire spectrum is strictly bounded by the power-law exponent:
    \begin{gather}
        \sum_{f=0}^{N-1} |\Delta^2 W_f| \le \alpha.
    \end{gather}
    Substituting this variation bound back into our inequality yields:
    \begin{gather}
        4 \sin^2\left(\frac{\pi m}{N}\right) |k_t(m)| \le \frac{\alpha}{N}.
    \end{gather}

    For the spatial domain $m \le \lfloor N/2 \rfloor$, the argument $\frac{\pi m}{N}$ falls within the interval $[0, \pi/2]$. In this domain, we can apply the linear lower bound for the sine function, $\sin(x) \ge \frac{2x}{\pi}$, which gives $\sin^2\left(\frac{\pi m}{N}\right) \ge \frac{4 m^2}{N^2}$. Substituting this yields:
    \begin{gather}
        4 \left(\frac{4 m^2}{N^2}\right) |k_t(m)| \le \frac{\alpha}{N} \implies |k_t(m)| \le \frac{\alpha N}{16 m^2}.
    \end{gather}

    We have thus proven that the optimal denoiser kernel decays spatially at a rate of $\mathcal{O}(1/m^2)$. Returning to the error of the masked denoiser, we substitute this polynomial decay bound into the error sum:
    \begin{gather}
        \mathrm{error}_i \leq 2\|x\|_\infty \sum_{m=L+1}^{\lfloor N/2 \rfloor} |k_t(m)| \nonumber \\ \leq \frac{\alpha N \|x\|_\infty}{8} \sum_{m=L+1}^{\lfloor N/2 \rfloor} \frac{1}{m^2}.
    \end{gather}
    We can upper-bound this remaining discrete sum by the corresponding continuous integral over the spatial tail:
    \begin{gather}
        \sum_{m=L+1}^{\lfloor N/2 \rfloor} \frac{1}{m^2} \le \int_{L}^{\infty} \frac{1}{y^2} \mathrm{d}y = \frac{1}{L}.
    \end{gather}
    Hence, the approximation error of the masked denoiser is strictly bounded by:
    \begin{gather}
        \mathrm{error}_i \le \frac{\alpha N \|x\|_\infty}{8 L}.
    \end{gather}
    This demonstrates that as the mask size $L$ increases, the error drops off linearly, mathematically proving that the spectral power law forces the MMSE optimal denoiser to be spatially local.
\end{proof}

\subsection{Theoretical Justification of the Classifier-Based Annotation}
\label{sec:appendix:classifier_theory}

We prove that the classifier-based $t_{\min}$ annotation method introduced in Phase 1 of Section~\ref{sec:method} comes with a formal guarantee: suppose the classifier, $f_\theta$,\footnote{Section~\ref{sec:method} uses $c_\phi$ to denote the classifier. We switched notation to $f_\theta$ for the proofs since they use $c$ to denote a Bernoulli random variable.} is $\varepsilon$-optimal and the threshold is $\tau$. Then annotating $t_{\min}$ using the method described in Section~\ref{sec:method} bounds the LeCam distance between the noisy distributions $p_t$ and $q_t$ (Theorem~\ref{th:approximate-classifier-threshold}). Simply put, it guarantees that $p_t$ and $q_t$ are similar (i.e. contracted) for $t > t_{\min}$.

\paragraph{Classifier threshold implies distributional closeness.}
Consider the population cross-entropy loss
\begin{equation*}
    \mathcal{L}_{\mathrm{CE}}(f) = -\mathbb{E}_{c, x}\!\left[(1-c) \log (1-f(x)) + c \log f(x)\right],
\end{equation*}
where we first sample $c \sim \mathrm{Bern}(1/2)$ and then sample $x \sim P$ if $c = 1$ or $x \sim Q$ otherwise. Let $M = \tfrac{1}{2} P + \tfrac{1}{2} Q$ denote the mixture. The optimal classifier is
\begin{equation}
    f^* = \arg\min_f \mathcal{L}_{\mathrm{CE}}(f), \qquad f^*(x) = \frac{dP}{dP+dQ}(x) = \frac{1}{2}\frac{dP}{dM}(x).
    \label{eq:optimal-classifier}
\end{equation}

\begin{theorem}[Optimal classifier threshold implies distributional closeness]
Let $f^*$ be the optimal classifier in Eq.~\eqref{eq:optimal-classifier}. Then
\begin{equation*}
    \mathbb{E}_Q f^*(x) = \frac{1}{2} - \frac{1}{2}\mathrm{LC}(P, Q).
\end{equation*}
In particular, $\mathbb{E}_Q f^*(x) \geq \tfrac{1}{2}-\tau$ implies $\mathrm{LC}(P, Q) \leq 2\tau$.
\label{th:optimal-classifier-threshold}
\end{theorem}

\begin{proof}
We compute
\begin{equation*}
    \mathbb{E}_{Q} f^* = \int \frac{dP \, dQ}{dP+dQ} = \frac{1}{2}\int \frac{(dP+dQ)^2 - dP^2 - dQ^2}{dP+dQ}.
\end{equation*}
Then
\begin{equation*}
    \mathbb{E}_{Q} f^* = 1 - \frac{1}{2}\int \frac{dP^2}{dP+dQ} - \frac{1}{2}\int \frac{dQ^2}{dP+dQ}.
\end{equation*}
Noting $\int \frac{dP^2}{dP+dQ} = \tfrac{1}{2}(\chi^2(P \| M) + 1)$ gives
\begin{equation*}
    \mathbb{E}_{Q} f^* = \frac{1}{2} - \frac{1}{2}\chi^2(P \| M) - \frac{1}{2}\chi^2(Q \| M).
\end{equation*}
Using $\mathrm{LC}(P, Q) = 2\chi^2(P \| M) = 2\chi^2(Q \| M)$, we conclude
\begin{equation*}
    \mathbb{E}_{Q} f^* = \frac{1}{2} - \frac{1}{2}\mathrm{LC}(P, Q).
\end{equation*}
The threshold condition $\mathbb{E}_Q f^* \geq \tfrac{1}{2} - \tau$ then immediately gives $\mathrm{LC}(P, Q) \leq 2\tau$.
\end{proof}

We now extend this result to approximate classifiers trained on the cross-entropy loss.

\begin{theorem}[Approximate classifier threshold implies distributional closeness]
Suppose $f_\theta : \mathbb{R}^d \to [0, 1]$ is an $\varepsilon$-optimal classifier, i.e.\ $\mathcal{L}_{\mathrm{CE}}(f_\theta) \leq \mathcal{L}_{\mathrm{CE}}(f^*) + \varepsilon$. Then $\mathbb{E}_{Q} f_\theta(x) \geq \tfrac{1}{2} - \tau$ implies $\mathrm{LC}(P, Q) \leq 2\tau + 2\sqrt{\varepsilon}$.
\label{th:approximate-classifier-threshold}
\end{theorem}

\begin{proof}
The cross-entropy loss admits the decomposition
\begin{equation*}
    \mathcal{L}_{\mathrm{CE}}(f) = \mathbb{E}_{x \sim M}\, d(f^*(x) \| f(x)) + \text{const},
\end{equation*}
where $d(p \| q) = \mathrm{KL}(\mathrm{Bern}(p) \| \mathrm{Bern}(q))$ is the binary KL divergence. The $\varepsilon$-accuracy condition implies $\mathbb{E}_{x \sim M} d(f^*(x) \| f(x)) \leq \varepsilon$, and by non-negativity $\mathbb{E}_{x \sim Q} d(f^*(x) \| f(x)) \leq 2\varepsilon$. Applying Pinsker's inequality and $\mathrm{d}_{\mathrm{TV}}(p, q) = |p - q|$ for Bernoulli variables yields
\begin{equation*}
    \mathbb{E}_{x \sim Q}(f_\theta(x) - f^*(x))^2 \leq \varepsilon.
\end{equation*}
By Jensen, $\mathbb{E}_{x \sim Q} f^*(x) \geq \mathbb{E}_{x \sim Q} f_\theta(x) - \sqrt{\varepsilon}$, so
\begin{equation*}
    \mathbb{E}_{x \sim Q} f^*(x) \geq \frac{1}{2} - \tau - \sqrt{\varepsilon}.
\end{equation*}
Applying Theorem~\ref{th:optimal-classifier-threshold} recovers the claim.
\end{proof}

\section{Variance-Preserving vs Variance Exploding Diffusion Processes}
\label{sec:appendix:vp_vs_ve}

The paper presents Ambient Diffusion Policy from the variance-exploding perspective of diffusion~\cite{song2021scorebasedgenerativemodelingstochastic} to lighten notation. Our implementation builds off the original Diffusion Policy~\cite{chi2024diffusionpolicyvisuomotorpolicy} work, which uses a variance-preserving implementation. We first show that the two perspectives are equivalent up to a change of variables.

The \textit{variance exploding} (VE) and \textit{variance preserving} (VP) forward processes are defined in Eq.~\eqref{eq:ve} and~\eqref{eq:vp} respectively \cite{song2021scorebasedgenerativemodelingstochastic}:
\begin{align}
    X_t &= X_0 + \sigma_{\mathrm{VE}}(t)\, Z, \label{eq:ve} \\
    X_t &= \sqrt{1 - \sigma_{\mathrm{VP}}^2(t)}\, X_0 + \sigma_{\mathrm{VP}}(t)\, Z, \label{eq:vp}
\end{align}
where $Z \sim \mathcal{N}(0, I)$, $\sigma_{\mathrm{VE}}(t)$ is an unbounded increasing function with $\sigma_{\mathrm{VE}}(0) = 0$, and $\sigma_{\mathrm{VP}}(t)$ is an increasing function with $\sigma_{\mathrm{VP}}(0) = 0$ and $\sigma_{\mathrm{VP}}(T) = 1$. Without loss of generality, we assume the data is normalized so that $\mathrm{Var}[X_0] = 1$.

\begin{proposition}
The VP and VE processes are equivalent up to a time-dependent rescaling of space.
\end{proposition}

\begin{proof}
Given a VE schedule $\sigma_{\mathrm{VE}}(t)$, define
\begin{equation}
    \sigma_{\mathrm{VP}}(t) := \frac{\sigma_{\mathrm{VE}}(t)}{\sqrt{1 + \sigma_{\mathrm{VE}}^2(t)}}, \qquad c(t) := \sqrt{1 + \sigma_{\mathrm{VE}}^2(t)}.
    \label{eq:rescaling}
\end{equation}
Under this mapping,
\begin{equation}
    \mathrm{SNR}_{\mathrm{VE}}(t) = \frac{1}{\sigma_{\mathrm{VE}}^2(t)} = \frac{1 - \sigma_{\mathrm{VP}}^2(t)}{\sigma_{\mathrm{VP}}^2(t)} = \mathrm{SNR}_{\mathrm{VP}}(t),
\end{equation}
so both processes have the same SNR at every $t$. The rescaled VE variable, $\tilde X_t$, satisfies
\begin{align}
    \widetilde{X}_t := \frac{X_t^{\mathrm{VE}}}{c(t)} = X_t^{\mathrm{VP}}.
\end{align}
Thus, the VP process is the VE process in normalized coordinates $x^\mathrm{VE}_t \mapsto x^\mathrm{VE}_t/c(t) =x^\mathrm{VP}_t$.
\end{proof}

\begin{proposition}
The score functions of the VP and VE processes are related by:
\begin{equation}
    s^{\mathrm{VP}}(\widetilde{x}_t, t) = c(t) \cdot s^{\mathrm{VE}}\!\left(c(t)\,\widetilde{x}_t,\, t\right).
\end{equation}
\end{proposition}

\begin{proof}
Let $p^{\mathrm{VP}}_t$ and $p^{\mathrm{VE}}_t$ denote the marginal densities of $X_t^{\mathrm{VP}}$ and $X_t^{\mathrm{VE}}$ respectively. Given the change of variables $x_t^{\mathrm{VE}} = c(t)\,x_t^{\mathrm{VP}}$, we have that $p^{\mathrm{VP}}_t(x_t^{\mathrm{VP}}) =c(t)^d\,p^{\mathrm{VE}}_t(c(t)\,x_t^{\mathrm{VP}})$, where $d$ is the dimension.
Taking the gradient with respect to $x_t^{\mathrm{VP}}$:
\begin{align}
    s^{\mathrm{VP}}(x_t^{\mathrm{VP}}, t) 
    &= \nabla_{x_t^{\mathrm{VP}}} \log p^{\mathrm{VP}}_t(x_t^{\mathrm{VP}}) \nonumber \\
    &= \nabla_{x_t^{\mathrm{VP}}} \log p^{\mathrm{VE}}_t(c(t)\,x_t^{\mathrm{VP}}) + \cancelto{0}{\nabla_{x_t^{\mathrm{VP}}} \log c(t)^d}\nonumber \\
    &= c(t) \cdot s^{\mathrm{VE}}\!\left(c(t)\,x_t^{\mathrm{VP}},\, t\right),
\end{align}
where the last line follows from the chain rule.
\end{proof}

\begin{proposition}
Let $h^*_{\mathrm{VE}}$ and $h^*_{\mathrm{VP}}$ denote the optimal denoisers for the VE and VP processes respectively. Then:
\begin{equation}
    h^*_{\mathrm{VP}}(x^\mathrm{VP}_t, t) = h^*_{\mathrm{VE}}\!\left(c(t)\,x^\mathrm{VP}_t,\, t\right).
\end{equation}
\end{proposition}

\begin{proof}
By Tweedie's formula, the optimal denoisers are related to the scores via:
\begin{align}
    h^*_{\mathrm{VE}}(x_t^{\mathrm{VE}}, t) &= x_t^{\mathrm{VE}} + \sigma_{\mathrm{VE}}^2(t)\, s^{\mathrm{VE}}(x_t^{\mathrm{VE}}, t), \label{eq:tweedie_ve}\\
    h^*_{\mathrm{VP}}(x_t^{\mathrm{VP}}, t) &= \frac{x_t^{\mathrm{VP}} + \sigma_{\mathrm{VP}}^2(t)\, s^{\mathrm{VP}}(x_t^{\mathrm{VP}}, t)}{\sqrt{1 - \sigma_{\mathrm{VP}}^2(t)}}. \label{eq:tweedie_vp}
\end{align}
Substituting the score correspondence from Proposition 2 and the definition of $c(t)$ in \eqref{eq:rescaling} yields the desired result.
\end{proof}

Therefore, any result derived in one framework translates to the other via $c(t) = \sqrt{1+\sigma_{\mathrm{VE}}^2(t)}$.

\section{VP Ambient Loss Derivation}
\label{sec:appendix:ambient_loss}
In the original Ambient Diffusion work, \citet{daras2024consistent} propose an alternative loss function, known as the Ambient loss. Unlike the denoising loss from DDPM~\cite{ho2020denoisingdiffusionprobabilisticmodels}, the Ambient loss assumes no access to the clean samples, $x_0$, and instead uses $x_{t_{\min}}$ as the prediction target. Nonetheless, it is minimized by the same denoiser as the DDPM loss. In the image domain, the Ambient loss has been shown to improve generalization. \citet{daras2024consistent} derive the Ambient loss for VE diffusion. In this section, we derive it for VP diffusion.

\begin{proposition}
The $\epsilon$-prediction VP Ambient loss for $t > t_{\min}$ is given by
\begin{equation}
    \mathcal{L}(\theta) = \mathbb{E}\!\left[\left\lVert \epsilon_\theta(x_t, t) - \frac{\sigma_t(1-\sigma_{t_{\min}}^2)}{\sigma_t^2 - \sigma_{t_{\min}}^2}\,x_t + \frac{\sigma_t\sqrt{(1-\sigma_t^2)(1-\sigma_{t_{\min}}^2)}}{\sigma_t^2 - \sigma_{t_{\min}}^2}\,x_{t_{\min}} \right\rVert_2^2\right],
    \label{eq:vp_ambient_denoiser_loss}
\end{equation}
and minimized by $\epsilon_\theta^* = \mathbb{E}[Z \mid X_t=x_t]$.
\end{proposition}

\begin{proof}
Let $X_t$ and $Z$ be the random variables from the forward process in VP diffusion, Eq.~\eqref{eq:vp}. Then:
\begin{equation}
    \mathbb{E}[Z \mid X_t = x_t] = \frac{x_t - \sqrt{1-\sigma_t^2}\, \mathbb{E}[X_0 \mid X_t=x_t]}{\sigma_t}.
    \label{eq:vp_denoiser}
\end{equation}
By the double Tweedie identity from \citet{daras2024consistent}:
\begin{equation}
    \mathbb{E}[X_0 \mid X_t=x_t] = \frac{\sigma_t^2\sqrt{1-\sigma_{t_{\min}}^2}}{\sigma_t^2-\sigma_{t_{\min}}^2} \mathbb{E}[X_{t_{\min}} \mid X_t = x_t] - \frac{\sigma_{t_{\min}}^2 \sqrt{1-\sigma_t^2}}{\sigma_t^2 - \sigma_{t_{\min}}^2}x_t.
    \label{eq:double_tweedies}
\end{equation}
Substituting Eq.~\eqref{eq:double_tweedies} into Eq.~\eqref{eq:vp_denoiser} and rearranging gives:
\begin{align}
    \mathbb{E}[X_{t_{\min}} \mid X_t=x_t] &= \alpha_t\, \mathbb{E}[Z \mid X_t=x_t] + \beta_t\, x_t, \label{eq:alpha_beta_denoiser}\\
    \alpha_t &:= -\frac{\sigma_t^2-\sigma_{t_{\min}}^2}{\sigma_t \sqrt{(1-\sigma_t^2)(1-\sigma_{t_{\min}}^2)}}, \label{eq:alpha_t}\\
    \beta_t &:= \frac{1-\sigma_{t_{\min}}^2}{\sqrt{(1-\sigma_t^2)(1-\sigma_{t_{\min}}^2)}}. \label{eq:beta_t}
\end{align}
Define $g_\theta(x_t, t) := \alpha_t\,\epsilon_\theta(x_t, t) + \beta_t\, x_t$. Since $t > t_{\min}$, we have $\alpha_t \neq 0$ and the map $\epsilon_\theta \mapsto g_\theta$ is a bijection. Thus, minimizing
\begin{equation}
    \mathbb{E}\!\left[\lVert g_\theta(x_t,t) - x_{t_{\min}}\rVert^2\right]
    \label{eq:g_theta_loss_eps}
\end{equation}
over $g_\theta$ is equivalent to minimizing over $\epsilon_\theta$. The minimizer of Eq.~\eqref{eq:g_theta_loss_eps} is $g_\theta^* = \mathbb{E}[X_{t_{\min}} \mid X_t=x_t]$, which by Eq.~\eqref{eq:alpha_beta_denoiser} implies $\epsilon_\theta^* = \mathbb{E}[Z \mid X_t=x_t]$. Rescaling the loss in Eq.~\eqref{eq:g_theta_loss_eps} by $\alpha_t^2$ yields:
\begin{equation}
    \mathbb{E}\!\left[\left\lVert \epsilon_\theta(x_t,t) + \frac{\beta_t}{\alpha_t}x_t - \frac{1}{\alpha_t}x_{t_{\min}}\right\rVert_2^2\right],
\end{equation}
which expands to the loss in the proposition.
\end{proof}
\begin{proposition}
The $x_0$-prediction VP Ambient loss for $t > t_{\min}$ is given by
\begin{equation}
    \mathbb{E}\!\left[\left\lVert h_\theta(x_t, t) + \frac{\sigma_{t_{\min}}^2\sqrt{1-\sigma_{t}^2}}{\sigma_t^2-\sigma_{t_{\min}}^2}x_{t} - \frac{\sigma_{t}^2\sqrt{1-\sigma_{t_{\min}}^2}}{\sigma_t^2-\sigma_{t_{\min}}^2} x_{t_{\min}} \right\rVert_2^2\right],
    \label{eq:vp_ambient_sampler_loss}
\end{equation}
and minimized by $h_\theta^* = \mathbb{E}[X_0 \mid X_t=x_t]$.
\end{proposition}

\begin{proof}
Isolating $\mathbb{E}[X_{t_{\min}} \mid X_t = x_t]$ in the double Tweedie identity Eq.~\eqref{eq:double_tweedies} gives:
\begin{align}
    \mathbb{E}[X_{t_{\min}}\mid X_t=x_t] &= \tilde{\alpha}_t\, \mathbb{E}[X_0 \mid X_t=x_t] + \tilde{\beta}_t\, x_t, \label{eq:alpha_beta_denoiser_x0}\\
    \tilde{\alpha}_t &:= \frac{\sigma_t^2 - \sigma_{t_{\min}}^2}{\sigma_t^2\sqrt{1-\sigma_{t_{\min}}^2}}, \quad
    \tilde{\beta}_t := \frac{\sigma_{t_{\min}}^2\sqrt{1-\sigma_t^2}}{\sigma_t^2\sqrt{1-\sigma_{t_{\min}}^2}}.
\end{align}
Define $\tilde g_\theta(x_t, t) := \tilde{\alpha}_t\, h_\theta(x_t, t) + \tilde{\beta}_t\, x_t$. Since $t > t_{\min}$, then $\tilde{\alpha}_t \neq 0$ and the map $h_\theta \mapsto \tilde g_\theta$ is a bijection. Thus, minimizing
\begin{equation}
    \mathbb{E}\!\left[\lVert \tilde g_\theta(x_t,t) - x_{t_{\min}}\rVert^2\right]
    \label{eq:g_theta_loss_x0}
\end{equation}
over $\tilde g_\theta$ is equivalent to minimizing over $h_\theta$. The minimizer of Eq.~\eqref{eq:g_theta_loss_x0} is $\tilde g_\theta^* = \mathbb{E}[X_{t_{\min}} \mid X_t=x_t]$, which by Eq.~\eqref{eq:alpha_beta_denoiser_x0} implies $h_\theta^* = \mathbb{E}[X_0 \mid X_t=x_t]$. Rescaling the loss in Eq.~\eqref{eq:g_theta_loss_x0} by $\tilde{\alpha}_t^2$ yields:
\begin{equation}
    \mathbb{E}\!\left[\left\lVert h_\theta(x_t,t) + \frac{\tilde{\beta}_t}{\tilde{\alpha}_t}x_t - \frac{1}{\tilde{\alpha}_t}x_{t_{\min}}\right\rVert_2^2\right],
\end{equation}
which expands to the loss in the proposition.
\end{proof}

\begin{remark}
Both losses reduce to the canonical MSE diffusion loss \cite{ho2020denoisingdiffusionprobabilisticmodels} when $t_{\min}=0$.
\end{remark}

\begin{remark}
For both the $\epsilon$- and $x_0$-prediction losses, the rescaling factor, $1/\alpha_t^2$ (resp. $1/\tilde{\alpha}_t^2$), diverges as $t \rightarrow t_{\min}$. This can cause training instability. Following \cite{daras2025ambient}, we introduce a buffer $B$: a sample from $\mathcal{D}_q$ is only used at diffusion times satisfying $t > t_{\min}$ and $1/\alpha_t^2 < B$ (resp. $1/\tilde{\alpha}_t^2 < B$). The effect of $B$ is studied in Figure~\ref{fig:planar_pushing_ablations}\subref{fig:perf_vs_buffer}.
\end{remark}

\section{Characterization of Robot Action Data and Action Diffusion}
\label{sec:appendix:action_diffusion}

\textbf{Spectral Power Law:} Figures~\ref{fig:spectral_power_law} and~\ref{fig:psd_all} visualize the spectral power law for several robot datasets. These datasets include human teleop vs algorithmically generated actions, absolute vs delta actions, end effector vs joint space actions, different parameterizations of SE(3) end-effector poses, and data from different tasks. Remarkably, a spectral power law existed in all the datasets we tested. Table~\ref{tab:psd_datasets} outlines the characteristics of each dataset in more detail.

The average PSD for each dataset was computed as follows. Given a dataset, we normalize the range of each action dimension to [-1,1]. For OXE, we instead normalized the 2\% and 98\% quartiles to [-1,1]. We do not normalize the rotation dimensions if they are represented in $\mathbb R^9$ format since these values already lie within [0,1] by construction. All the datasets were collected at 10 Hz; however, action frequencies in OXE vary. To ensure consistency, we set the frequency of OXE to 10 Hz by downsampling higher-frequency data and artificially speeding up lower-frequency data. We chose an action horizon of 100. For each action dimension, we compute the periodogram along the time axis. Finally, we average these periodograms across all action chunks in the dataset. Our plots exclude the DC component, which is often close to 0 due to the range normalization. These plots use a log-log scale; a power law, $f^{-\alpha}$, manifests as a straight line of slope $-\alpha$.

We fix the action frequency and horizon for consistency across datasets, but remark that the spectral power law emerges regardless of these choices. Concurrent work performed the analysis using the discrete cosine transform and obtained similar results \cite{zhang2026hydradp3frequencyawarerightsizing3d}.

\begin{figure*}[!t]
    \centering
    \begin{subfigure}[b]{0.32\textwidth}
        \includegraphics[width=\linewidth]{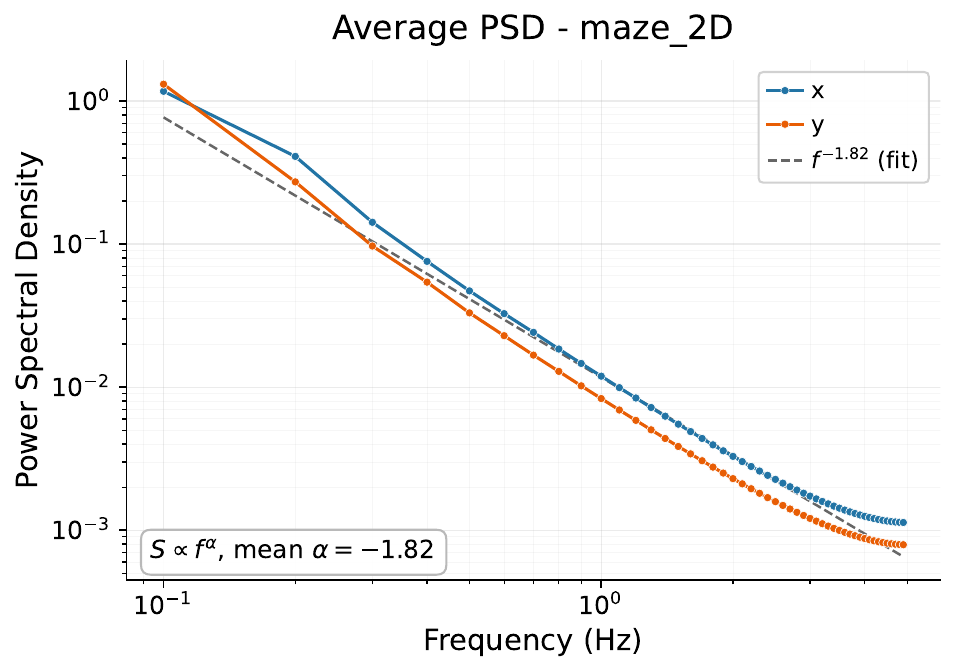}
        \caption{$\mathcal D_p$ from Section~\ref{sec:controlled_experiments:motion_planning}}
        \label{fig:maze_2d_psd}
    \end{subfigure}
    \hfill
    \begin{subfigure}[b]{0.32\textwidth}
        \includegraphics[width=\linewidth]{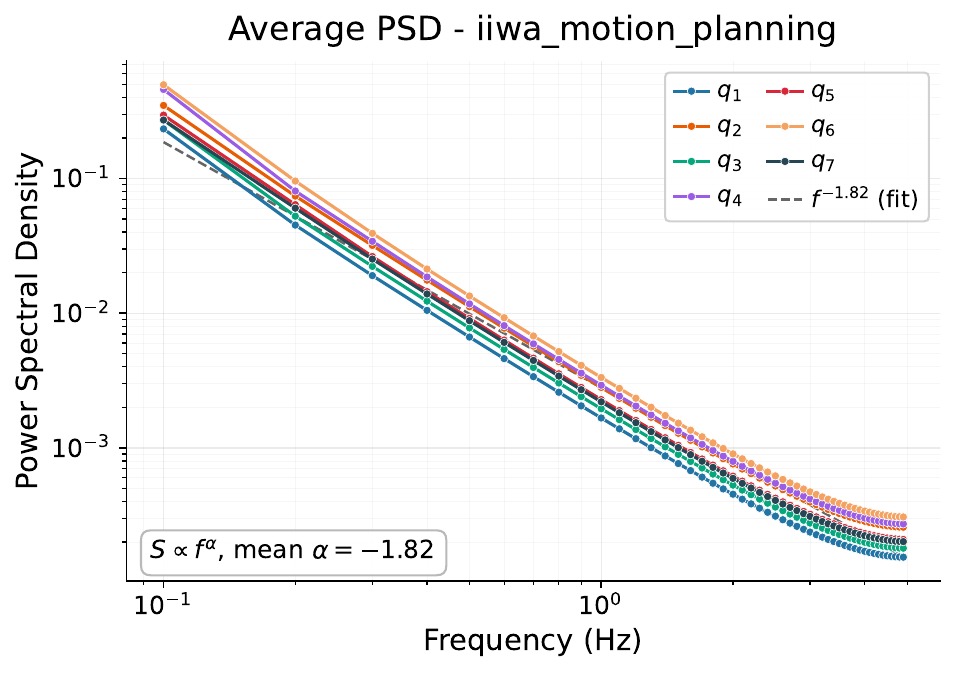}
        \caption{$\mathcal D_p$ from Appendix~\ref{sec:appendix:nmp}}
        \label{fig:nmp_psd}
    \end{subfigure}
    \hfill
    \begin{subfigure}[b]{0.32\textwidth}
        \includegraphics[width=\linewidth]{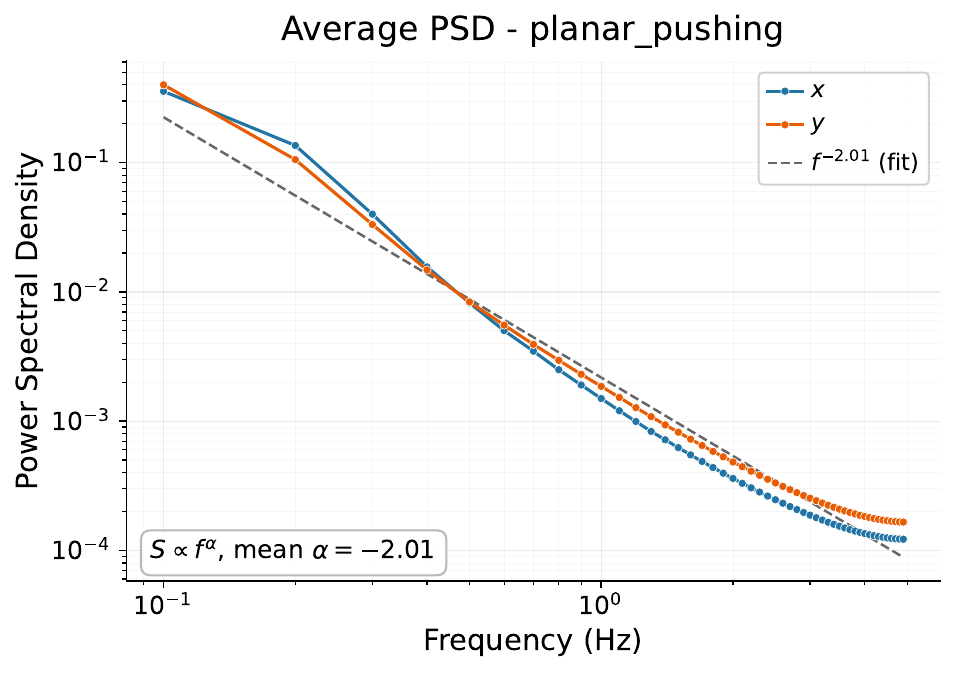}
        \caption{$\mathcal D_p$ from Section~\ref{sec:controlled_experiments:sim2real}}
        \label{fig:planar_pushing_psd}
    \end{subfigure}

    \vspace{0.5em}

    \begin{subfigure}[b]{0.32\textwidth}
        \includegraphics[width=\linewidth]{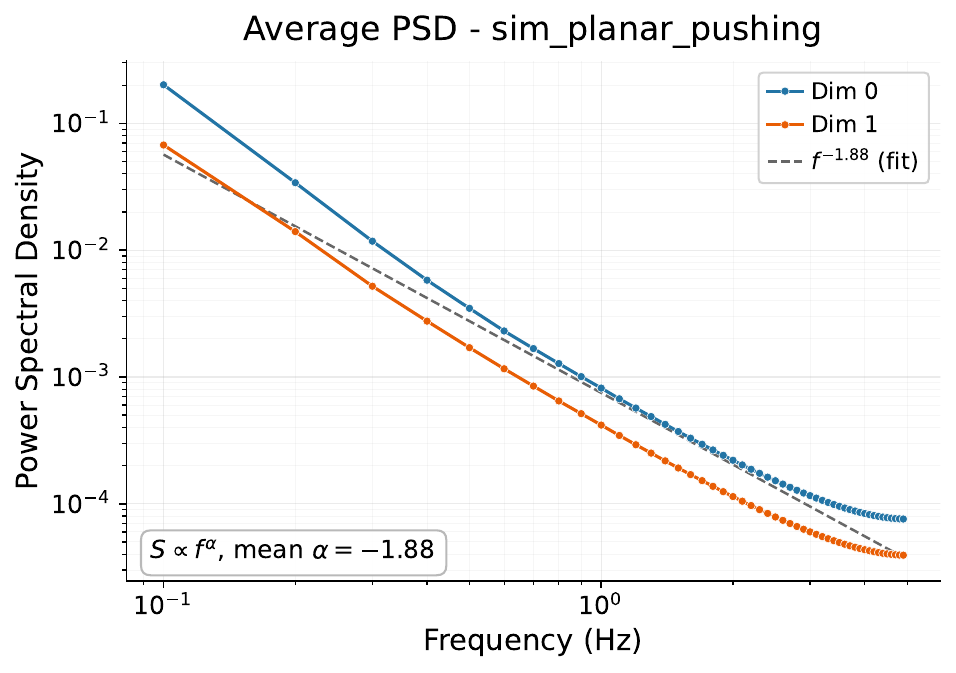}
        \caption{$\mathcal D_q$ from Section~\ref{sec:controlled_experiments:sim2real}}
        \label{fig:sim_planar_pushing_psd}
    \end{subfigure}
    \hfill
    \begin{subfigure}[b]{0.32\textwidth}
        \includegraphics[width=\linewidth]{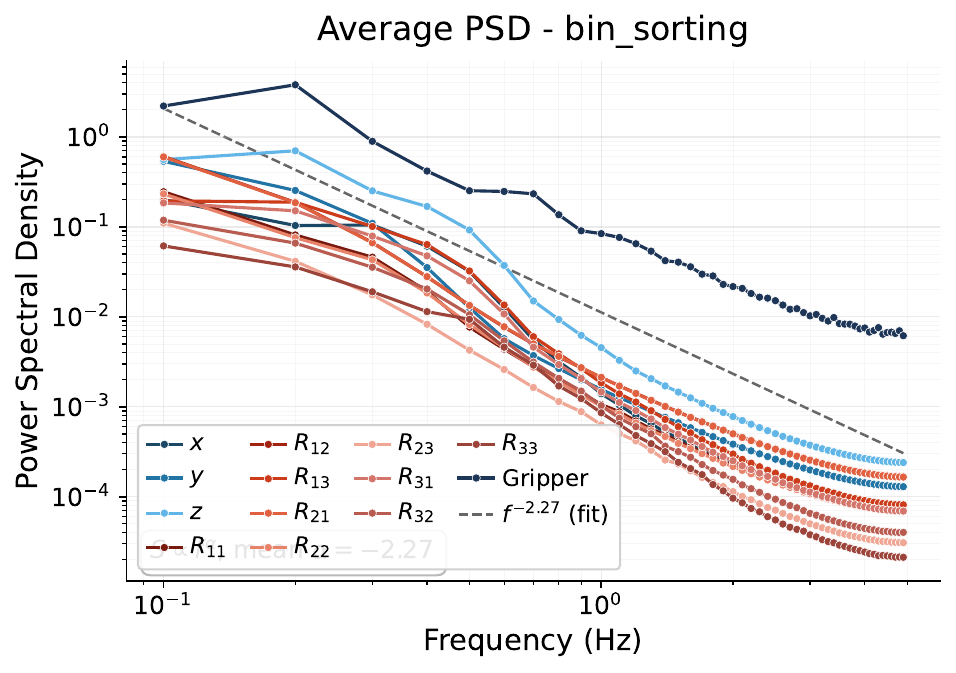}
        \caption{$\mathcal D_p$ from Section~\ref{sec:controlled_experiments:bin_sorting}}
        \label{fig:bin_sorting_psd}
    \end{subfigure}
    \hfill
    \begin{subfigure}[b]{0.32\textwidth}
        \includegraphics[width=\linewidth]{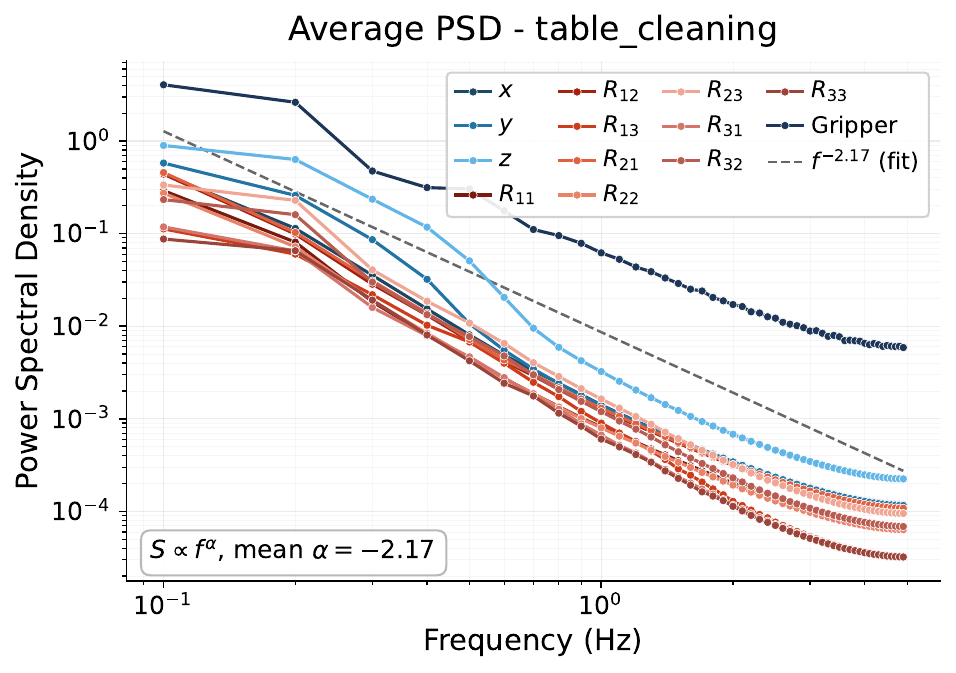}
        \caption{$\mathcal D_p$ from Section~\ref{sec:oxe_experiments}}
        \label{fig:table_cleaning_psd}
    \end{subfigure}
    \caption{\textbf{A visualization of the average PSD for a diverse range of robot datasets; a spectral power law exists in all of them.} Table~\ref{tab:psd_datasets} characterizes the datasets shown in (a)-(f) as well as the OXE dataset shown in Figure~\ref{fig:spectral_power_law}.}
    \label{fig:psd_all}
\end{figure*}

\begin{table}[!t]
\caption{Characteristics of the datasets analyzed via power spectral density (PSD) in Figures~\ref{fig:spectral_power_law} and~\ref{fig:psd_all}. Despite their differences, all datasets exhibit a spectral power law.}
\label{tab:psd_datasets}
\centering
\resizebox{\columnwidth}{!}{%
\begin{tabular}{clcccccc}
\toprule
\textbf{Fig.} & \textbf{Dataset} & \textbf{Source} & \textbf{Generation} & \textbf{Action Space} & \textbf{Action Type} & \textbf{Rotation Rep.} & $\bar{\alpha}^\dagger$ \\
\midrule
Fig.~\ref{fig:spectral_power_law} & COXE~\cite{open_x_embodiment_rt_x_2023} & $\mathcal{D}_q$ (Sec.~\ref{sec:oxe_experiments}) & Mixed & Task space & Delta & Axis-angle & -1.47 \\
\midrule
Fig.~\ref{fig:maze_2d_psd} & GCS~\cite{marcucci2022motionplanningobstaclesconvex} & $\mathcal{D}_p$ (Sec.~\ref{sec:controlled_experiments:motion_planning}) & Algorithmic & $x$-$y$ & Absolute & N/A & -1.82 \\
Fig.~\ref{fig:nmp_psd} & TrajOpt. & $\mathcal{D}_p$ (App.~\ref{sec:appendix:nmp}) & Algorithmic & Joint space & Absolute & N/A & -1.82 \\
Fig.~\ref{fig:planar_pushing_psd} & Planar pushing \cite{wei2025empirical} & $\mathcal{D}_p$ (Sec.~\ref{sec:controlled_experiments:sim2real}) & Human teleop & $x$-$y$ & Absolute & N/A & -2.01 \\
Fig.~\ref{fig:sim_planar_pushing_psd} & \citet{graesdal2024tightconvexrelaxationscontactrich} & $\mathcal{D}_q$ (Sec.~\ref{sec:controlled_experiments:sim2real}) & Algorithmic & $x$-$y$ & Absolute & N/A & -1.88 \\
Fig.~\ref{fig:bin_sorting_psd} & Block sorting & $\mathcal{D}_p$ (Sec.~\ref{sec:controlled_experiments:bin_sorting}) & Human teleop & Task space & Absolute & $\mathbb{R}^9$ & -2.27 \\
Fig.~\ref{fig:table_cleaning_psd} & Table cleaning & $\mathcal{D}_p$ (Sec.~\ref{sec:oxe_experiments}) & Human teleop & Task space & Absolute & $\mathbb{R}^9$ & -2.17 \\
\bottomrule
\end{tabular}%
}\\[0.5em]
\parbox{\columnwidth}{\footnotesize $\dagger$ $\bar{\alpha}$: mean fitted exponent across action dimensions, where $S(f) \propto f^{\alpha}$}
\end{table}

\textbf{Region Classification Experiment (from Section~\ref{sec:robot_data:spectral_power_law}):} We test how well the global plan in $A_0$ can be recovered from $A_t$ using a learned classifier. We train the classifier on 5000 trajectories generated by GCS trajectory optimization~\cite{marcucci2022motionplanningobstaclesconvex} in the 2D maze.
GCS decomposes the planning problem into discrete decisions (which regions of space to move through) and continuous decisions (the precise trajectory to follow within each region). Thus, the set of regions a trajectory passes through acts as a proxy for the global planning decisions, as it captures the topological structure of the maze.

Concretely, we train a classifier $c_\theta: \mathcal{A} \times \mathbb{R} \rightarrow \mathbb{R}^{20}$ that takes as input $(A_t, t)$ and predicts the probability that $A_0$ passes through each of the 20 convex regions\footnote{We intentionally exclude the start and goal conditioning from the classifier since those are sufficient for predicting the global plan.}. Figure~\ref{fig:maze_2d_plan_visualization} shows an example trajectory and its corresponding region labels. Figure~\ref{fig:maze_2d_region_accuracy} plots the average prediction accuracy at each noise level on a held-out test dataset: it saturates quickly as the noise decreases. This implies that nearly all the global planning information is diffused early in the backward process. Altogether, this experiment supports our claim that action diffusion exhibits a global-to-local hierarchy.

\begin{figure*}
    \centering
    \begin{subfigure}[b]{0.24\textwidth}
        \centering
        \includegraphics[width=\linewidth]{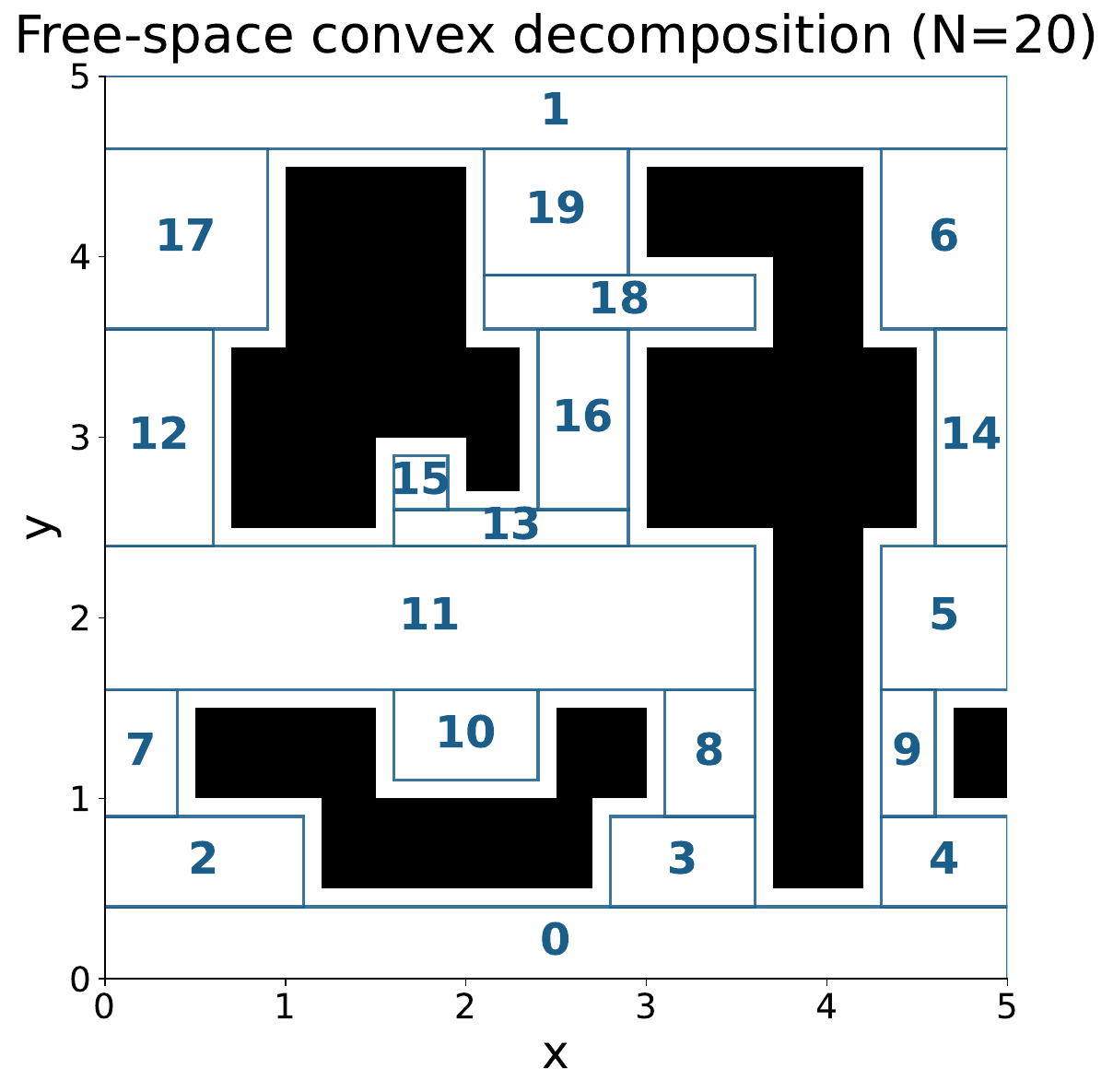}
        \caption{Convex decomposition of the collision-free space.}
        \label{fig:maze_2d_regions}
    \end{subfigure}
    \hfill
    \begin{subfigure}[b]{0.28\textwidth}
        \centering
        \includegraphics[width=\linewidth]{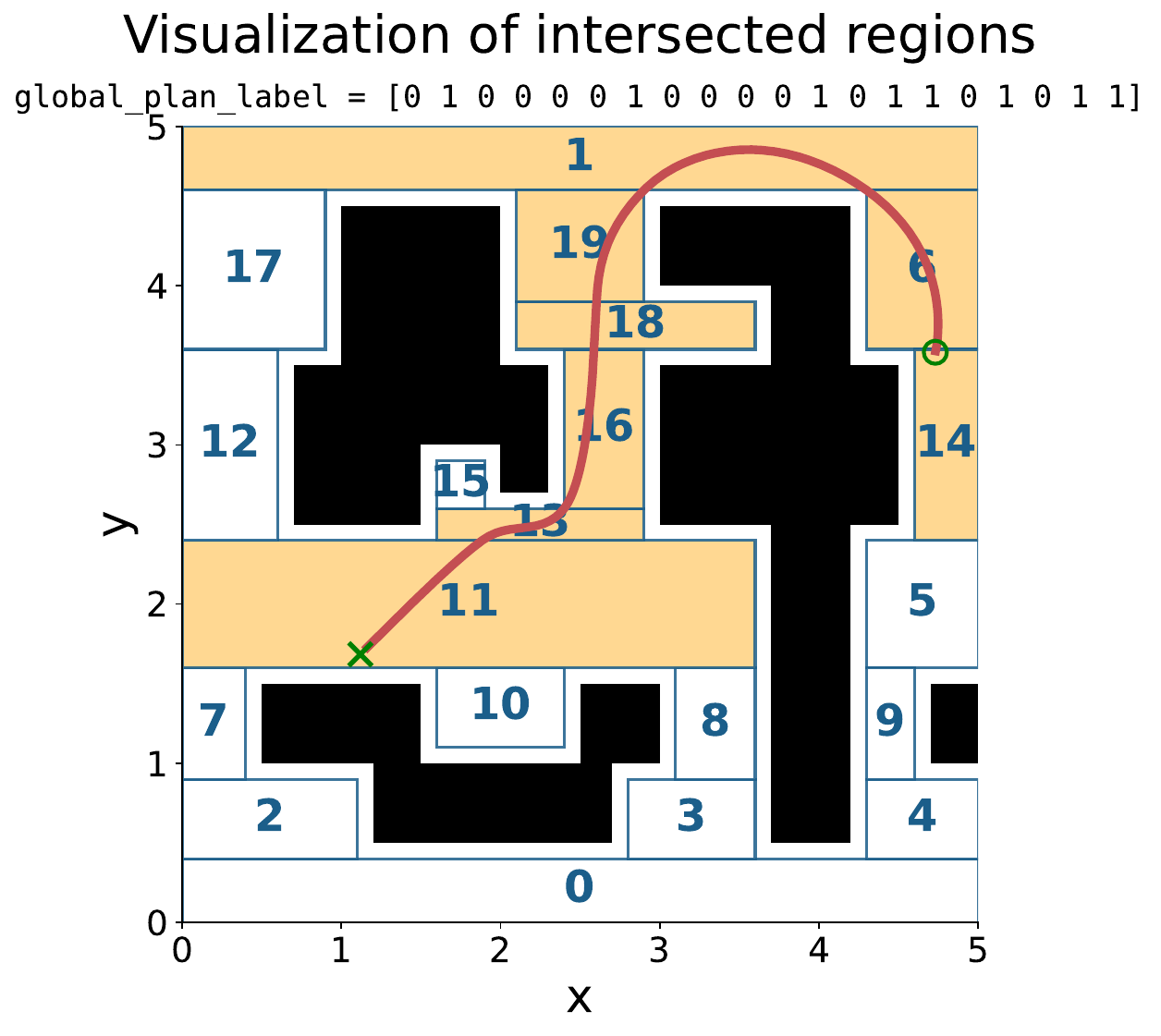}
        \caption{Visualization of the discrete decisions in GCS.}
        \label{fig:maze_2d_plan_visualization}
    \end{subfigure}
    \hfill
    \begin{subfigure}[b]{0.45\textwidth}
        \centering
        \includegraphics[width=\linewidth]{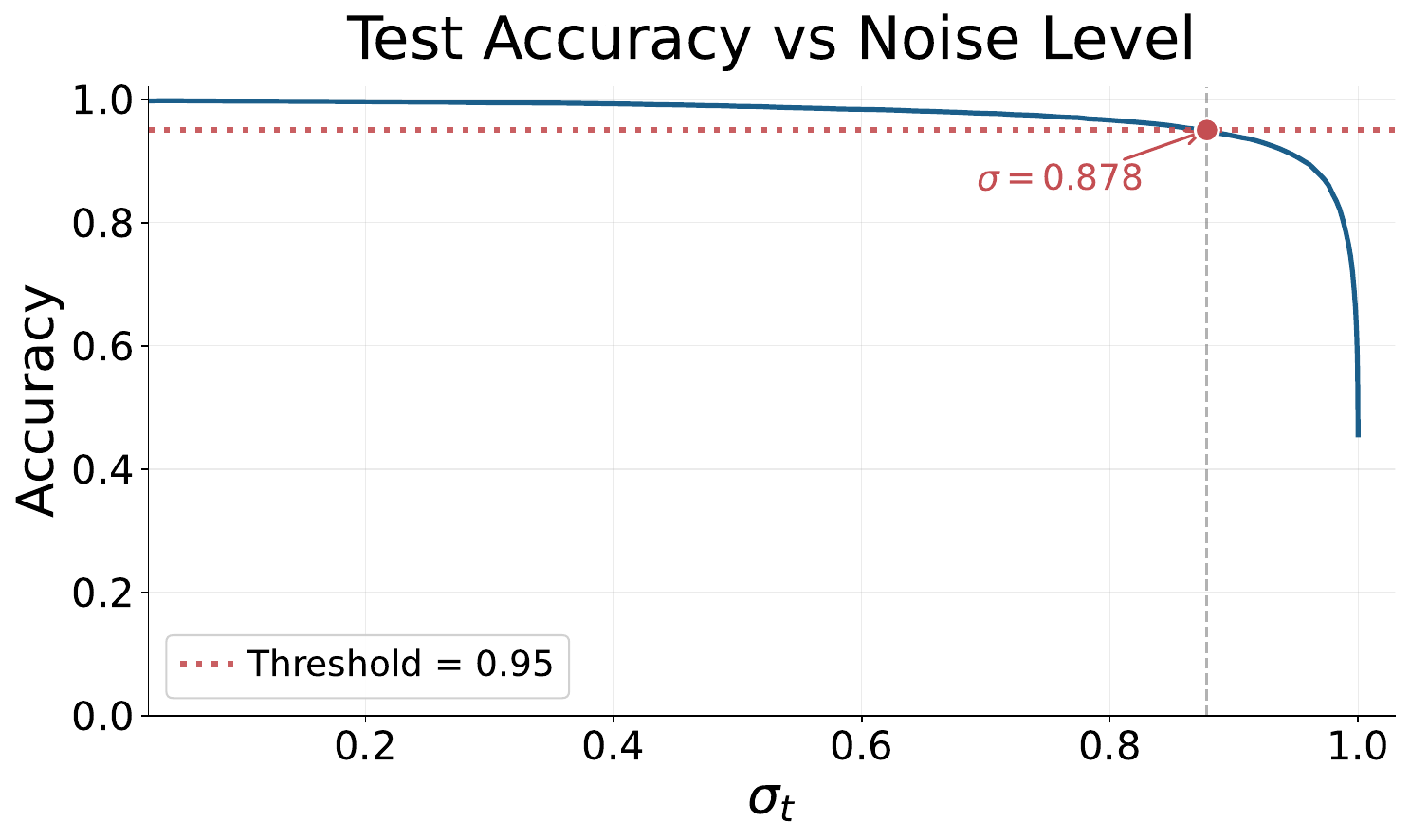}
        \caption{Prediction accuracy for the plan at different noise levels.}
        \label{fig:maze_2d_region_accuracy}
    \end{subfigure}
    \caption{\textbf{The backward process commits to a global plan at high noise.} By $\sigma_t = 0.878$, the global plan in $A_0$ can already be recovered from $A_t$ with 95\% accuracy. Accuracy plateaus afterward. This indicates that the start of the backward process performs global planning. Figure~\ref{fig:visualize_backward} shows that the remainder of the backward process is devoted to local trajectory refinement.}
    \label{fig:maze_2d_region_classifier}
\end{figure*}

\textbf{Related Work:} Decomposing action generation hierarchically is a classic approach for solving problems in robotics \cite{garrett2020integratedtaskmotionplanning}. Recent works have begun incorporating this idea into Diffusion Policy. Generally, these approaches factorize action generation into a high and low level policy~\cite{pi05, nvidia2025gr00tn1openfoundation}; however, this factorization is often handcrafted and task-specific~\cite{wang2025articubotlearninguniversalarticulated, zhao2025generalizablehierarchicalskilllearning, gireesh2026hdflowhierarchicaldiffusionflowplanning}. Our results show that a global-to-local action hierarchy emerges naturally from the spectral structure of the data. Thus, we \textit{hypothesize} that with sufficient data, an end-to-end policy could potentially learn the ``correct'' hierarchy in the diffusion process without manual design.

\textbf{Locality (Sensitivity Field from Figure~\ref{fig:locality}):} Let $h_\theta^{(8)}$ denote an $x_0$-predictor's estimate of the 8th clean action and let $A_t^{(i)}$ denote the $i$-th action in the noisy action chunk $A_t$. For each diffusion time $t$, the sensitivity field is defined as:
\begin{equation}
    S_t(i) := \mathbb{E}_{(O,A_0)\sim\mathcal{D}}\!\left[\left\lVert \frac{\partial h_\theta^{(8)}(A_t, O, t)}{\partial A_t^{(i)}}\right\rVert_F\right], \quad i \in [0, 15].
    \label{eq:sensitivity_field}
\end{equation}
We normalize each sensitivity field so that $\max_i S_t(i) = 1$ at each diffusion time. Intuitively, $S_t(i)$ measures the sensitivity of the 8th action estimate to the $i$-th noisy action at diffusion time $t$.

\section{Pseudocode and Possible Extensions}
\label{sec:appendix:implementation_details}

\subsection{Ambient Diffusion Policy Pseudocode}
\label{sec:appendix:pseudocode}
Given $\mathcal D_p$, $\mathcal D_q$, and annotations from Phase 1 in the Method section, we train an Ambient Diffusion Policy following Algorithm~\ref{alg:adp}. At each step, we sample a batch of diffusion times $t^{(i)} \stackrel{\text{i.i.d.}}{\sim} \mathcal{U}[0, T]$ and then draw admissible samples from $\mathcal{D}_p \cup \mathcal{D}_q$. We update the policy parameters, $\theta$, to minimize the loss in \textsc{ComputeLoss}.

\textsc{ComputeLoss} executes the forward process and can use two choices for $\mathcal{L}$. The first is the standard denoising loss, $\mathcal L_\text{denoising}$~\cite{ho2020denoisingdiffusionprobabilisticmodels}. The second is the \emph{Ambient loss}, $\mathcal L_\text{ambient}$, discussed in Appendix~\ref{sec:appendix:ambient_loss}.\footnote{The Ambient loss was only used in the maze experiments (see Table~\ref{tab:experiments_overview}). We tested the Ambient loss on the experiments in the OXE experiments and found it made no statistically significant difference; thus, we continued using the regular denoising loss for simplicity.} We present two versions of \textsc{ComputeLoss} for variance-exploding (Algorithm~\ref{alg:adp_loss}) and variance-preserving diffusion processes (Algorithm~\ref{alg:adp_loss_vp}).

\begin{algorithm}
\caption{Ambient Diffusion Policy: Training}
\label{alg:adp}
\begin{algorithmic}[1]
\Require $\mathcal{D}_p$, $\mathcal{D}_q$ with $(t_{\min}^{(i)}, t_{\max}^{(i)})$, $\mathcal{L} \in \{\mathcal{L}_\text{ambient}, \mathcal{L}_\text{denoising}\}$, noise schedule $\sigma(t)$, max time $T$, batch size $B$, policy parameters $\theta$
\Repeat
    \State Sample $t^{(i)} \stackrel{\text{i.i.d.}}{\sim} \mathcal{U}[0, T]$ for $i = 1, \ldots, B$
    \State Sample batch $\{(O^{(i)}, A_0^{(i)})\}_{i=1}^{B}$ s.t. $t^{(i)} \in [0, t_{\max}^{(i)}) \cup (t_{\min}^{(i)}, T]$
    \State $\ell \gets \tfrac{1}{B}\sum_{i=1}^{B} \textsc{ComputeLoss}\left(\mathcal{L}, A_0^{(i)}, O^{(i)}, t^{(i)}, t_{\max}^{(i)}, t_{\min}^{(i)}\right)$
    \State Update $\theta$ on $\ell$
\Until{converged}
\State \Return $\pi_\theta$
\end{algorithmic}
\end{algorithm}

\begin{algorithm}
\caption{\textsc{ComputeLoss} (variance-exploding)}
\label{alg:adp_loss}
\begin{algorithmic}[1]
\Require $\mathcal{L}$, $A_0$, $O$, $t$, $t_{\max}$, $t_{\min}$
\If{$t \in [0, t_{\max})$ \textbf{or} $\mathcal{L} = \mathcal{L}_{\text{denoising}}$}
    \State $A_t \gets A_0 + \sigma_t\, Z$, \quad $Z \sim \mathcal{N}(0, I)$
    \State \Return $\mathcal{L}_{\text{denoising}}(A_t, A_0, O, t)$ \Comment{standard DDPM loss~\cite{ho2020denoisingdiffusionprobabilisticmodels}}
\Else \Comment{$t\in(t_{\min}, T]$ and $\mathcal{L} = \mathcal{L}_{\text{ambient}}^\text{VE}$}
    \State Sample $Z_1, Z_2 \stackrel{\text{i.i.d.}}{\sim} \mathcal{N}(0, I)$
    \State $A_{t_{\min}} \gets A_0 + \sigma_{t_{\min}}\, Z_1$
    \State $A_t \gets A_{t_{\min}} + \sqrt{\sigma_t^2 - \sigma_{t_{\min}}^2}\ Z_2$
    \State \Return $\mathcal{L}_{\text{ambient}}^\text{VE}(A_t, A_{t_{\min}}, O, t)$ \Comment{VE Ambient loss~\cite{daras2024consistent}}
\EndIf
\end{algorithmic}
\end{algorithm}

\begin{algorithm}
\caption{\textsc{ComputeLoss} (variance-preserving)}
\label{alg:adp_loss_vp}
\begin{algorithmic}[1]
\Require $\mathcal{L}$, $A_0$, $O$, $t$, $t_{\max}$, $t_{\min}$
\If{$t \in [0, t_{\max})$ \textbf{or} $\mathcal{L} = \mathcal{L}_{\text{denoising}}$}
    \State $A_t \gets \sqrt{1 - \sigma_t^2} A_0 + \sigma_t Z$, \quad $Z \sim \mathcal{N}(0, I)$
    \State \Return $\mathcal{L}_{\text{denoising}}(A_t, A_0, O, t)$ \Comment{standard DDPM loss~\cite{ho2020denoisingdiffusionprobabilisticmodels}}
\Else \Comment{$t\in(t_{\min}, T]$ and $\mathcal{L} = \mathcal{L}_{\text{ambient}}^\text{VP}$}
    \State Sample $Z_1, Z_2 \stackrel{\text{i.i.d.}}{\sim} \mathcal{N}(0, I)$
    \State $A_{t_{\min}} \gets \sqrt{1 - \sigma_{t_{\min}}^2} A_0 + \sigma_{t_{\min}} Z_1$
    \State $A_t \gets \sqrt{\frac{1 - \sigma_t^2}{1 - \sigma_{t_{\min}}^2}} A_{t_{\min}} + \sqrt{\frac{\sigma_t^2 - \sigma_{t_{\min}}^2}{1 - \sigma_{t_{\min}}^2}} Z_2$
    \State \Return $\mathcal{L}_{\text{ambient}}^\text{VP}(A_t, A_{t_{\min}}, O, t)$ \Comment{Eq.~\eqref{eq:vp_ambient_denoiser_loss} ($\epsilon$-pred) or Eq.~\eqref{eq:vp_ambient_sampler_loss} ($x_0$-pred)}
\EndIf
\end{algorithmic}
\end{algorithm}

\begin{figure}[!t]
    \centering
    \begin{minipage}{0.63\linewidth}
        \centering
        \includegraphics[width=\linewidth]{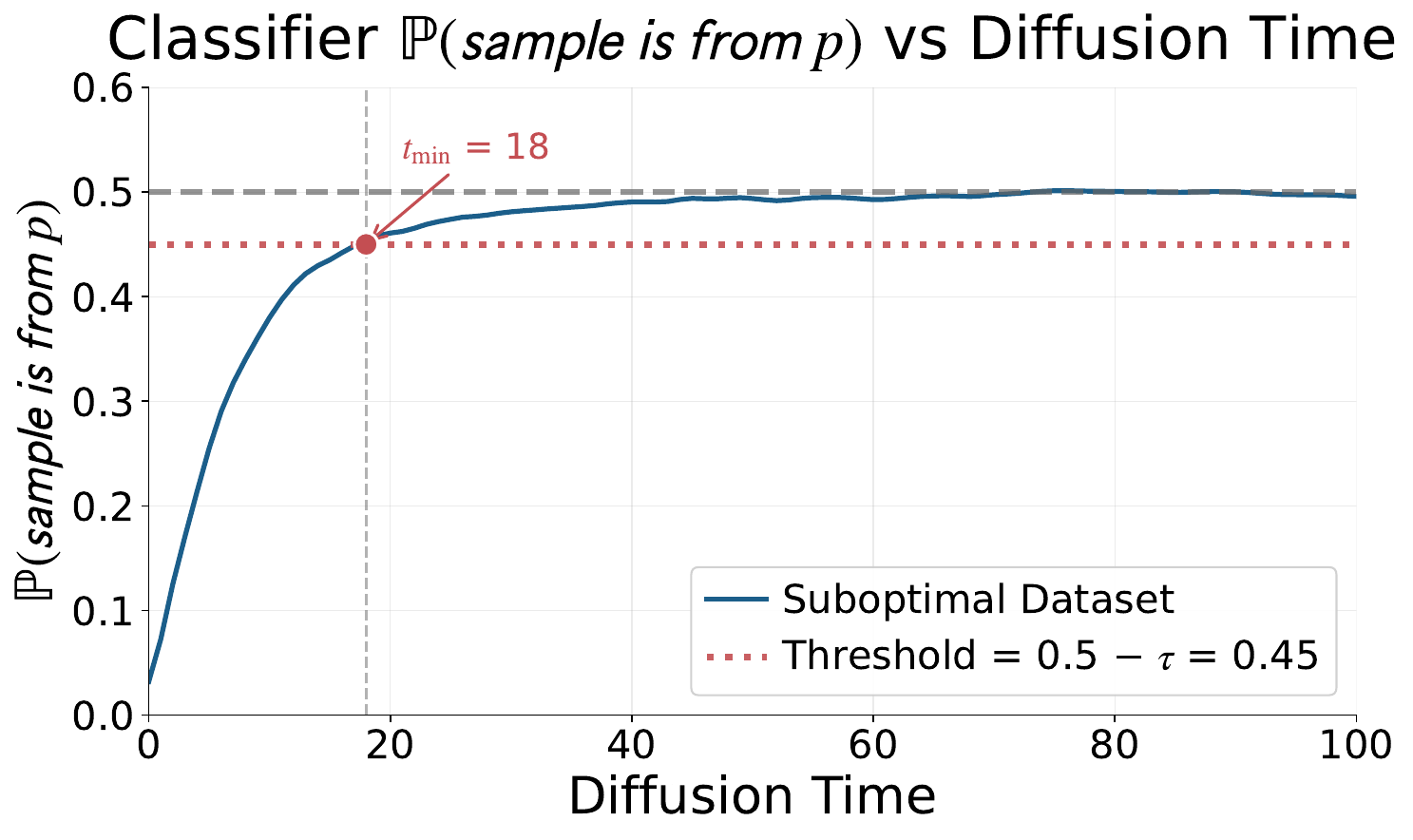}
    \end{minipage}\hfill
    \begin{minipage}{0.35\linewidth}
        \caption{\textbf{Average classifier output for suboptimal samples in the 2D maze dataset (Section~\ref{sec:controlled_experiments:motion_planning}) at different diffusion times.} As $t$ increases, $p_t$ and $q_t$ become harder to distinguish. The dotted line is the threshold $0.5-\tau$, which is crossed for $t_{\min}=18$.}
        \label{fig:classifier_accuracy}
    \end{minipage}
\end{figure}

\subsection{Diffusion Policy Implementation Details}
\label{sec:appendix:diffusion_policy}
Unless otherwise noted, policies are trained to predict the next 16 actions at 10 Hz with a fixed horizon for observations and actions. We execute the first 8 actions and predict new actions in an MPC-like fashion. Policies are conditioned on the 2 most recent observations. We consider a fixed horizon for observations and actions and assume that the learned policy is \textit{time-invariant}, i.e. the conditional action distribution $\pi(\cdot \mid O)$ depends only on the most recent observations and not on the time $t$. The per-experiment observation and action spaces are summarized in Tables~\ref{tab:experiments_overview} and~\ref{tab:experiments_overview_oxe}.

We use the U-Net architecture from \cite{chi2024diffusionpolicyvisuomotorpolicy}, but add the option for conditioning on a goal. When a policy is image-conditioned, we use the ResNet18 backbone for the vision encoder. We train the policy end-to-end with a cosine learning rate schedule. The noise schedule $\sigma(t)$ is a cosine schedule with $T=100$ diffusion steps. We apply EMA to the training loop. Sampling is performed online using 10 steps of DDIM \cite{song2022denoisingdiffusionimplicitmodels}. Predicted actions are linearly interpolated into piecewise linear trajectories and passed through a first-order low-pass filter before being sent to the robot.

\subsection{Future Direction: Classifier-Based Rejection Sampling}
\label{sec:appendix:rejection_sampling}
This section proposes a similar algorithm to Ambient Diffusion Policy based on rejection sampling. We provide theoretical foundations for this method, but do not experimentally verify it. This is an exciting direction for future work.

Recall that we trained a classifier that outputs the probability that a noisy action chunk came from $p_t$ (as opposed to $q_t$). Algorithm~\ref{alg:rejection-sampling-classifier} samples (approximately) from $P$ given only access to samples from the mixture $M = \tfrac{1}{2}P + \tfrac{1}{2}Q$ and the classifier. Proposition~\ref{prop:rejection-sampling-guarantee} bounds the total variation between the sampled distribution and $P$. By using the classifier for rejection sampling, we can train denoisers for $P$ despite leveraging data from an auxiliary source $Q$. In the proofs, we use $f_\theta$ to denote the classifier instead of $c_\phi$ from Section~\ref{sec:method}.

\begin{algorithm}[H]
\caption{Rejection sampling via classifier}
\label{alg:rejection-sampling-classifier}
\begin{algorithmic}[1]
\Require Noisy classifier $f_\theta(x, \sigma)$, mixture distribution $M = \tfrac{1}{2}P + \tfrac{1}{2}Q$, noise distribution $\rho_\sigma$
\State Sample $\sigma \sim \rho_\sigma$
\While{True}
    \State Sample $X \sim M$
    \State $X_\sigma \gets \textsc{AddNoise}(X, \sigma)$
    \State Sample $c \sim \mathrm{Bern}(f_\theta(X_\sigma, \sigma))$
    \If{$c = 1$}
        \State \Return $(\sigma, X_\sigma)$
    \EndIf
\EndWhile
\end{algorithmic}
\end{algorithm}

\begin{proposition}[Rejection sampling guarantee]
Suppose $f_\theta$ is uniformly $\varepsilon$-optimal across $\sigma$:
\begin{equation*}
    \mathcal{L}_{\mathrm{CE}}(f \mid \sigma) \leq \mathcal{L}_{\mathrm{CE}}(f^* \mid \sigma) + \varepsilon, \qquad \forall \sigma.
\end{equation*}
For any $\sigma \geq 0$, let $P_\sigma$ be the noisy distribution of $X \sim P$, and let $\hat P_\sigma$ be the distribution of $X_\sigma$ output by Algorithm~\ref{alg:rejection-sampling-classifier}. Then
\begin{equation*}
    \mathrm{d}_{\mathrm{TV}}(\hat P_\sigma, P_\sigma) \leq O(\sqrt{\varepsilon}).
\end{equation*}
\label{prop:rejection-sampling-guarantee}
\end{proposition}

\begin{proof}
We drop the subscript $\sigma$ and write $P \equiv P_\sigma$, $Q \equiv Q_\sigma$. Let $m(x)$ denote the mixture density. The true distribution $P$ satisfies $p(x) = 2 m(x) f^*(x)$, while the rejection-sampled density is
\begin{equation*}
    \hat p(x) = \frac{m(x) f_\theta(x)}{Z}, \quad Z = \mathbb{E}_{X \sim M}[f_\theta(X)],
\end{equation*}
where $Z$ is the acceptance rate. By Pinsker applied to the cross-entropy gap, $\mathbb{E}_{X \sim M}[(f^*(X) - f_\theta(X))^2] \leq \varepsilon/2$, so by Jensen
\begin{equation*}
    \delta \triangleq \mathbb{E}_{X \sim M} |f_\theta(X) - f^*(X)| \leq \sqrt{\varepsilon/2}.
\end{equation*}
This gives $|Z - \tfrac{1}{2}| \leq \delta$, hence $Z \geq \tfrac{1}{2} - \delta$. The total variation between $\hat P$ and $P$ is
\begin{align*}
    \mathrm{d}_{\mathrm{TV}}(\hat P, P) &= \frac{1}{2} \int \left| \frac{m(x) f_\theta(x)}{Z} - 2 m(x) f^*(x) \right| dx \\
    &= \mathbb{E}_{X \sim M} \left| \frac{f_\theta(X) - 2 Z f^*(X)}{2Z} \right|.
\end{align*}
By the triangle inequality,
\begin{equation*}
    \left| \frac{f_\theta(X) - 2 Z f^*(X)}{2Z} \right| \leq \frac{|f_\theta(X) - f^*(X)|}{2Z} + \frac{f^*(X)}{2Z} |1 - 2Z|.
\end{equation*}
Taking the expectation over $M$ and using $\mathbb{E}_{X \sim M}[f^*(X)] = \tfrac{1}{2}$ and $|1 - 2Z| \leq 2\delta$:
\begin{equation*}
    \mathrm{d}_{\mathrm{TV}}(\hat P, P) \leq \frac{\delta}{2Z} + \frac{2\delta}{4Z} = \frac{\delta}{Z} \leq \frac{2\delta}{1 - 2\delta}.
\end{equation*}
For $\varepsilon < 2/9$ we have $1 - 2\delta > 1/3$, hence
\begin{equation*}
    \mathrm{d}_{\mathrm{TV}}(\hat P, P) \leq 6\sqrt{\varepsilon/2} = O(\sqrt{\varepsilon}).
\end{equation*}
\end{proof}

\section{Experiments Overview}
\label{sec:appendix:experiments_overview}
Tables~\ref{tab:experiments_overview} and~\ref{tab:experiments_overview_oxe} summarize the datasets, distribution shifts, observation and action spaces, hyperparameters, loss functions, and evaluation details for all main experiments in Sections~\ref{sec:controlled_experiments} and~\ref{sec:oxe_experiments} respectively; ablation parameters may differ. Importantly, we swept hyperparameters for each baseline to ensure a fair comparison.

Binary success metrics are reported with 95\% Clopper-Pearson confidence intervals. Otherwise, error bars are standard error intervals. For each training run, we save 5-10 checkpoints and report the performance of the best-performing one. While the initial exposition of our algorithm was presented with variance-exploding diffusion, our implementation uses variance-preserving diffusion (i.e.\ $\sigma(T)=1$).

\begin{table*}[!hb]
\centering
\caption{Experimental configurations for the four \textit{generality} experiments. Parameters vary for the ablations.}
\renewcommand{\arraystretch}{1.2}
\resizebox{\textwidth}{!}{%
\begin{tabular}{>{\bfseries}l | >{\bfseries}l | cccc}
    \toprule
    & & \textbf{Maze}
    & \makecell{\textbf{Neural}\\\textbf{Motion Planning}}
    & \makecell{\textbf{Sim-and-Real}\\\textbf{Co-training}}
    & \makecell{\textbf{Block}\\\textbf{Sorting}} \\
    \midrule
    \multirow{5}{*}{\rotatebox[origin=c]{90}{Dataset Details}}
    & $\mathcal{D}_p$ & GCS \cite{marcucci2022motionplanningobstaclesconvex} & \makecell{Trajectory\\Optimization \cite{manipulation}} & \makecell{Joystick\\Teleop \cite{wei2025empirical}} & VR Teleop \\
    & $\mathcal{D}_q$ & RRT \cite{LaValle1998RapidlyexploringRT} & BiRRT \cite{kuffner2000rrt} & \makecell{Contact-Rich\\Planner \cite{graesdal2024tightconvexrelaxationscontactrich}} & VR Teleop \\
    & $|\mathcal{D}_p|$ (Num. Demos.) & 50 & 100,000 & 50 & 50 \\
    & $|\mathcal{D}_q|$ (Num. Demos.) & 5000 & 1,000,000 & 2000 & 200 \\
    & Nature of shift & \makecell{Low-quality/\\noisy trajectories} & \makecell{Low-quality/\\noisy trajectories} & \makecell{Sim-to-real\\gap} & \makecell{Task\\mismatch} \\
    \midrule
    \multirow{2}{*}{\rotatebox[origin=c]{90}{Policy I/O}}
    & Observation space
    & Proprio
    & Proprio
    & \makecell{EE position\\Scene camera\\Wrist camera}
    & \makecell{EE pose (xyz, R9)\\Gripper state\\Scene camera\\Blocks camera\\Wrist camera} \\
    & Action space
    & $x$-$y$ position
    & \makecell{Joint angles\\$\mathbb{R}^7$}
    & \makecell{$x$-$y$ pusher\\position}
    & \makecell{EE command\\(xyz, R9, gripper)} \\
    \midrule
    \multirow{3}{*}{\rotatebox[origin=c]{90}{Parameters}}
    & $\sigma_{t_{\min}^*}\in[0,1]$ & \makecell{0.074\\(sweep)} & \makecell{0.025\\(sweep)} & \makecell{varies\\see Table~\ref{tab:sim_and_real_results}} & \makecell{N/A\\(sweep)} \\
    & $\sigma_{t_{\max}^*}\in[0,1]$ & N/A & N/A & \makecell{0\\(sweep)} & \makecell{0.46\\(sweep)} \\
    & $\alpha^*\in[0,1]$ for co-training & \makecell{0.019\\(sweep)} & \makecell{0.091\\(sweep)} & \makecell{0.25\\(sweep)} & \makecell{0.9\\(sweep)} \\
    \midrule
    \multirow{2}{*}{\rotatebox[origin=c]{90}{Loss}}
    & Loss prediction target & $x_0$ & $\epsilon$ & $x_0$ & $\epsilon$ \\
    & Loss function & $\mathcal L_\text{ambient} $ & $\mathcal L_\text{denoising}$ & $\mathcal L_\text{denoising}$ & $\mathcal L_\text{denoising}$ \\
    \midrule
    \rotatebox[origin=c]{90}{Eval.}
    & Num.\ trials & 1,000 & 1,000 & 200 & \makecell{200\\(~800 blocks)} \\
    \bottomrule
\end{tabular}%
}
\label{tab:experiments_overview}
\end{table*}

\begin{table}[t]
\centering
\footnotesize
\setlength{\tabcolsep}{4pt}
\caption{Experimental configurations for the two \textit{OXE scaling} experiments. Parameters vary for the ablations.}
\renewcommand{\arraystretch}{1.2}
\begin{tabular}{>{\bfseries}l | >{\bfseries}l | cc}
    \toprule
    & & \makecell{\textbf{Table}\\\textbf{Cleaning}}
    & \makecell{\textbf{Tower}\\\textbf{Building}} \\
    \midrule
    \multirow{5}{*}{\rotatebox[origin=c]{90}{Dataset Details}}
    & $\mathcal{D}_p$ & VR Teleop & VR Teleop \\
    & $\mathcal{D}_q$ & \makecell{OXE\\\cite{open_x_embodiment_rt_x_2023}} & \makecell{OXE\\\cite{open_x_embodiment_rt_x_2023}} \\
    & $|\mathcal{D}_p|$ (Num. Demos.) & 50 \& 150 & 35 \\
    & $|\mathcal{D}_q|$ (Num. Demos.) & \makecell{1{,}279{,}774 (MS++)\\1{,}402{,}606 (COXE)} & 1{,}402{,}606 (COXE) \\
    & Nature of shift & Unstructured & Unstructured \\
    \midrule
    \multirow{2}{*}{\rotatebox[origin=c]{90}{Policy I/O}}
    & Observation space
    & \makecell{EE pose (xyz, R9)\\Gripper state\\Table camera\\Drawer camera\\2x Wrist camera}
    & \makecell{EE pose (xyz, R9)\\Gripper state\\Table camera\\Tower camera\\2x Wrist camera} \\
    & Action space
    & \makecell{Delta pose\\(xyz, axis-angle,\\gripper)}
    & \makecell{Delta pose\\(xyz, axis-angle,\\gripper)} \\
    \midrule
    \multirow{3}{*}{\rotatebox[origin=c]{90}{Parameters}}
    & $\sigma_{t_{\min}^*}\in[0,1]$ & \makecell{See Figure~\ref{fig:weighting_visualization}\\(classifier)} & \makecell{See Figure~\ref{fig:tower_building_weighting_visualization}\\(classifier)} \\
    & $\sigma_{t_{\max}^*}\in[0,1]$ & \makecell{0.167\\(no sweep)} & \makecell{0.09\\(no sweep)} \\
    & $\alpha^*\in[0,1]$ for co-training & \makecell{0.5\\(no sweep)} & \makecell{0.5\\(no sweep)} \\
    \midrule
    \multirow{2}{*}{\rotatebox[origin=c]{90}{Loss}}
    & Loss prediction target & $\epsilon$ & $\epsilon$ \\
    & Loss function & $\mathcal L_\text{denoising}$ & $\mathcal L_\text{denoising}$ \\
    \midrule
    \rotatebox[origin=c]{90}{Eval.}
    & Num.\ trials & \makecell{20\\(60 objects)} & 20 \\
    \bottomrule
\end{tabular}
\label{tab:experiments_overview_oxe}
\end{table}

\section{Controlled Experiments}
\label{sec:appendix:generality}

\subsection{Noisy Trajectories: 2D Maze}
\textbf{Data Generation:} The GCS \cite{marcucci2022motionplanningobstaclesconvex} trajectories were generated using third-order Bezier curves, first-order continuity constraints, and a maximum of 10 rounded paths. Both the GCS and the RRT datasets were generated with a velocity constraint of 1m/s in the $x$ and $y$ directions and 0.1m of obstacle padding. The maze environment is 5m by 5m.

\textbf{Evaluation:} During each trial, we sample the start and goal uniformly in the padded collision-free configuration space; the robot is initialized at the start configuration. A trial is successful if the robot can navigate to a 0.15m ball around the end configuration without collision within 30s of simulation time. We performed 1000 trials per policy.

We compute the smoothness metric as the average squared acceleration of the policy's rollout. Concretely, we fit a cubic spline to the rollout and evaluate
\begin{equation}
    \mathrm{Smoothness} = \tfrac{1}{T}\!\!\int_0^T\!\!\lVert a(t)\rVert_2^2\, dt,
    \label{eq:smoothness}
\end{equation}
where $a(t)$ is the acceleration of the spline at time $t$. We only report smoothness for the successful rollouts, for three reasons: 1) failed rollouts have orders of magnitude higher accelerations that disproportionally bias the results, 2) across 1000 trials, most policies had less than 10 failed rollouts, 3) the overall trend when averaging across all trials is the same.

\label{sec:appendix:maze}

Figure~\ref{fig:maze_sigma_min_sweep} shows the maze success rate as a function of $\sigma_{t_{\min}}$, and Figure~\ref{fig:more_maze_rollouts} visualizes representative policy rollouts for each training algorithm.

\begin{figure}[!t]
    \centering
    \includegraphics[width=\linewidth]{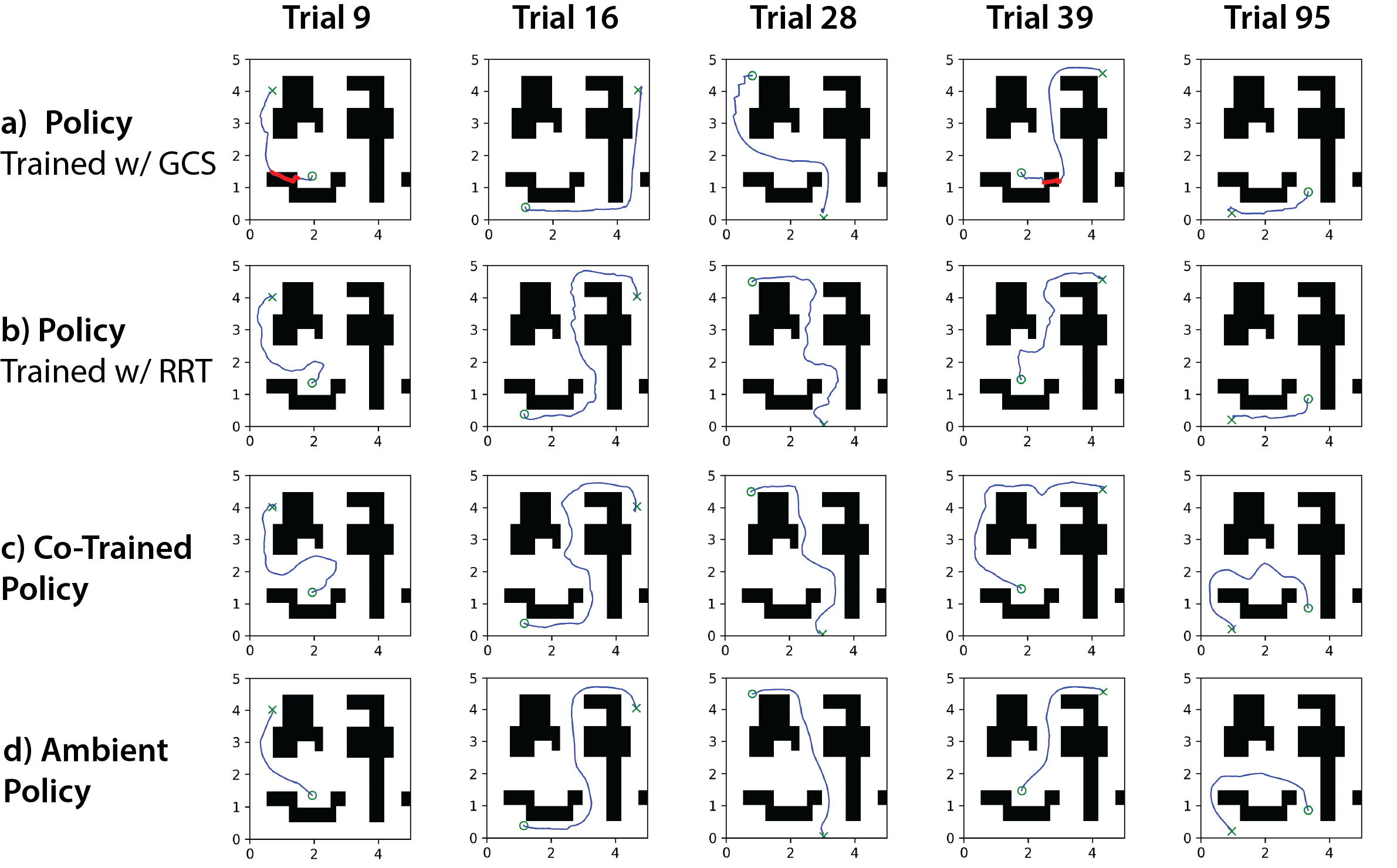}
    \caption{\textbf{Policy rollouts from the same start and goal pairs for different training algorithms.}
    \textbf{a)} The policy trained with data filtering (i.e. GCS-only) has not seen enough data to learn the maze. Its rollouts frequently collide with obstacles.
    \textbf{b)} The policy trained with RRT-only has memorized the maze, but exhibits non-smoothness.
    \textbf{c)} The co-trained policy's behavior is similar to the RRT-only policy.
    \textbf{d)} The Ambient Diffusion Policy is smooth and understands the structure of the maze.}
    \label{fig:more_maze_rollouts}
\end{figure}

\begin{figure}[!t]
    \centering
    \includegraphics[width=0.6\linewidth]{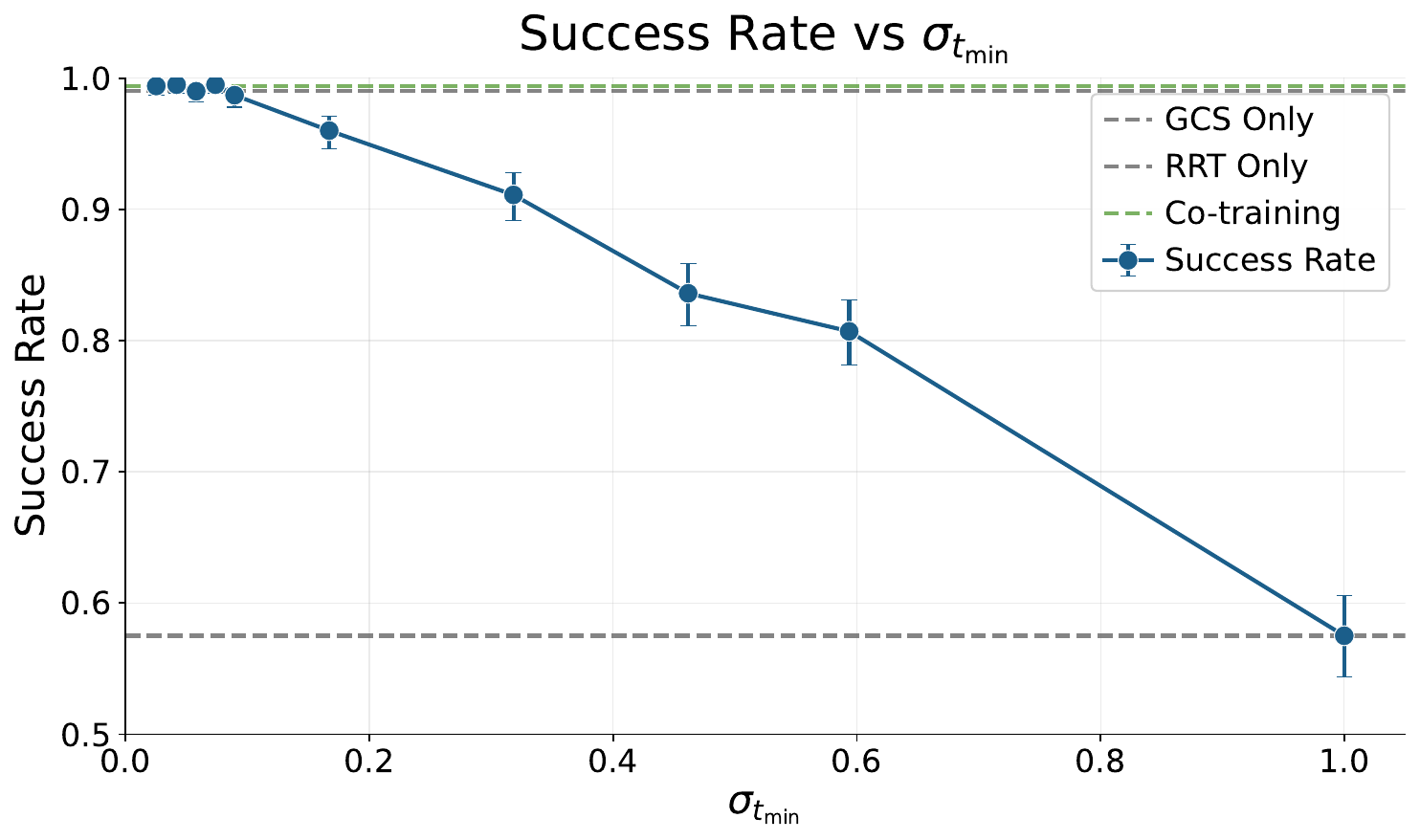}
    \caption{\textbf{Maze success rate vs $\sigma_{t_{\min}}.$} For small $\sigma_{t_{\min}}$, the Ambient Policies are smooth and achieve high success rate (see Figure~\ref{fig:maze_pareto}). As $\sigma_{t_{\min}}$ increases, the RRT data is used less during training, which reduces success rate.}
    \label{fig:maze_sigma_min_sweep}
\end{figure}

\subsection{Noisy Trajectories: Neural Motion Planning}
\label{sec:appendix:nmp}

This appendix provides the full experimental details for the neural motion planning experiment in Section~\ref{sec:controlled_experiments:nmp}, along with best-of-100 results and a $\sigma_{t_{\min}}$ sweep. The main results (Table~\ref{tab:nmp_results} and Figure~\ref{fig:nmp_pareto}) are presented in Section~\ref{sec:controlled_experiments:nmp}. Table~\ref{tab:maze_vs_nmp} compares the 2D maze and 7-DoF NMP setups side by side.

\begin{table}[!t]
    \centering
    \caption{Comparison between the 2D maze and 7-DoF neural motion planning experiments.}
    \label{tab:maze_vs_nmp}
    \setlength{\tabcolsep}{4pt}
    \begin{tabular}{l cc}
        \toprule
        & \textbf{2D Maze} & \textbf{7-DoF NMP} \\
        \midrule
        \rowcolor{rowgray}
        \textbf{$\mathcal{D}_p$ Source}     & GCS \cite{marcucci2022motionplanningobstaclesconvex} & Trajectory Optimizer \cite{manipulation} \\
        \textbf{$\mathcal{D}_q$ Source}     & RRT \cite{LaValle1998RapidlyexploringRT} & BiRRT \cite{kuffner2000rrt} \\
        \rowcolor{rowgray}
        \textbf{$|\mathcal{D}_p|$}          & 50 & 100,000 \\
        \textbf{$|\mathcal{D}_q|$}          & 5,000 & 1,000,000 \\
        \rowcolor{rowgray}
        \textbf{Output Type}                & Action Chunk & Full Plan \\
        \textbf{Rollouts}                   & Closed-Loop & Open-Loop \\
        \rowcolor{rowgray}
        \textbf{Penetration Threshold}      & 0cm & 2cm \\
        \textbf{Output Space}               & 2D $(x, y)$ & 7D (Config.\ Space) \\
        \rowcolor{rowgray}
        \textbf{Start \& End Goals} & Uniformly Random & \makecell[c]{Around points of interest \\ (See Appendix~\ref{sec:appendix:nmp})} \\
        \bottomrule
    \end{tabular}
\end{table}

The environment in Figure~\ref{fig:teaser}b was generated using a procedural method \cite{izatt2022capturing} and the assets are from~\citet{pfaff2025steerablescenegenerationpost}. Start and goal configurations are determined by sampling a random end-effector target point $p\in\SE(3)$ and then solving the optimization inverse kinematics problem
\begin{equation}
    \begin{array}{rl}
        \operatorname{find} & q\in\R^7 \\
        \mathrm{s.t.} & \operatorname{FK}_{\mathrm{pos}}(q)=p, \\
        & \operatorname{SDF}(q) \ge 0.01,
    \end{array}
    \label{eq:ik}
\end{equation}
where $\operatorname{FK}_{\mathrm{pos}}$ is the position forward kinematics, returning the position of the gripper as a function of the joint angles, and $\operatorname{SDF}$ returns the clearance (in meters) between the robot and any obstacles.
The constraint $\operatorname{SDF}(q)\ge 0.01$ is implemented via Drake's \texttt{MinimumDistanceLowerBoundConstraint}.
Targets are sampled 1) within shelves, 2) in bins, 3) near objects, and 4) uniformly in the free space. We sample from all 4 classes of target configurations with equal probability.
\eqref{eq:ik} is solved using SNOPT~\cite{gill2005snopt}.

Given a start and goal configuration $q_0,q_1\in\R^7$, we use BiRRT~\cite{kuffner2000rrtconnect} to produce the suboptimal data.
To produce the high-quality data, we first run randomized shortcutting~\cite{geraerts2007creating} to slightly simplify the path.
We then solve the trajectory optimization problem
\begin{equation}
    \begin{array}{rl}
        \min_{q(s)} & \int_0^1 ||\dot q(s)||^2\,\mathrm{d}s \\
        \mathrm{s.t.} & q(0)=q_0,q(1)=q_1, \\
        & \operatorname{SDF}(q(s))\ge 0,\;\forall s\in[0,1]
    \end{array}
\end{equation}
with Drake's \texttt{KinematicTrajectoryOptimization}, using its standard B-spline discretization \cite{tedrake2019drake}.
We use the shortcut trajectory as an initial guess and solve with IPOPT~\cite{wachter2006ipopt}, with solver options tuned aggressively to maintain feasibility.

The collision-free constraint $\operatorname{SDF}(q(s))\ge 0,\;\forall s\in[0,1]$ is applied at 200 points evenly spaced along the trajectory. Collision checking with finer granularity is performed after trajectory optimization to filter trajectories with collisions. In principle, the discarded trajectories could be added to $\mathcal D_q$ and used at sufficiently high diffusion times during training.

Each trajectory in $\mathcal D_p$ and $\mathcal D_q$ is subsampled to 100 waypoints, equally spaced in time. This subsampling procedure acts as a low-pass filter which reduces some of the non-smoothness in the BiRRT data.

\textbf{Evaluation:} Instead of rolling out the model closed-loop (as in the maze experiments), we train the planner to predict the entire plan at once. Start and end configurations are sampled from the 4 classes described above. Each trial is successful if the planner predicts a collision-free trajectory with up to 2cm of penetration. We allow 2cm of padding because the training trajectories were generated with 0cm of padding and frequently pass very close to obstacles. Both the first and last waypoints of the predicted plan must be within a 0.2rad ball of the start and end configurations respectively.

\begin{figure}[!t]
    \centering
    \begin{minipage}[c]{0.46\textwidth}
        \centering
        \resizebox{\linewidth}{!}{%
        \begin{tabular}{l|cc}
            \toprule
            Training Algo. & Success Rate $\uparrow$ & \makecell{Smoothness $\downarrow$\\(Eq.~\eqref{eq:smoothness})} \\
            \midrule
            Data Filtering & $73.9^{+2.7}_{-2.8}$ \% & $\mathbf{3.7}$ \\
            \makecell[l]{Co-train\\($\alpha^*=0.091$)} & $91.3^{+1.7}_{-1.9}$ \% & $32.6$ \\
            \makecell[l]{Ambient\\($\sigma_{t_{\min}^*}=0.025$)} & $\mathbf{91.9^{+1.6}_{-1.9}}$ \% & $20.2$ \\
            \bottomrule
        \end{tabular}}
        \captionof{table}{\textbf{Neural Motion Planning Best-of-100 Results (1000 trials).} Ambient retains the highest success rate while achieving lower squared acceleration than the co-trained baseline.}
        \label{tab:nmp_best_of_100}
    \end{minipage}
    \hfill
    \begin{minipage}[c]{0.52\textwidth}
        \centering
        \includegraphics[width=\linewidth]{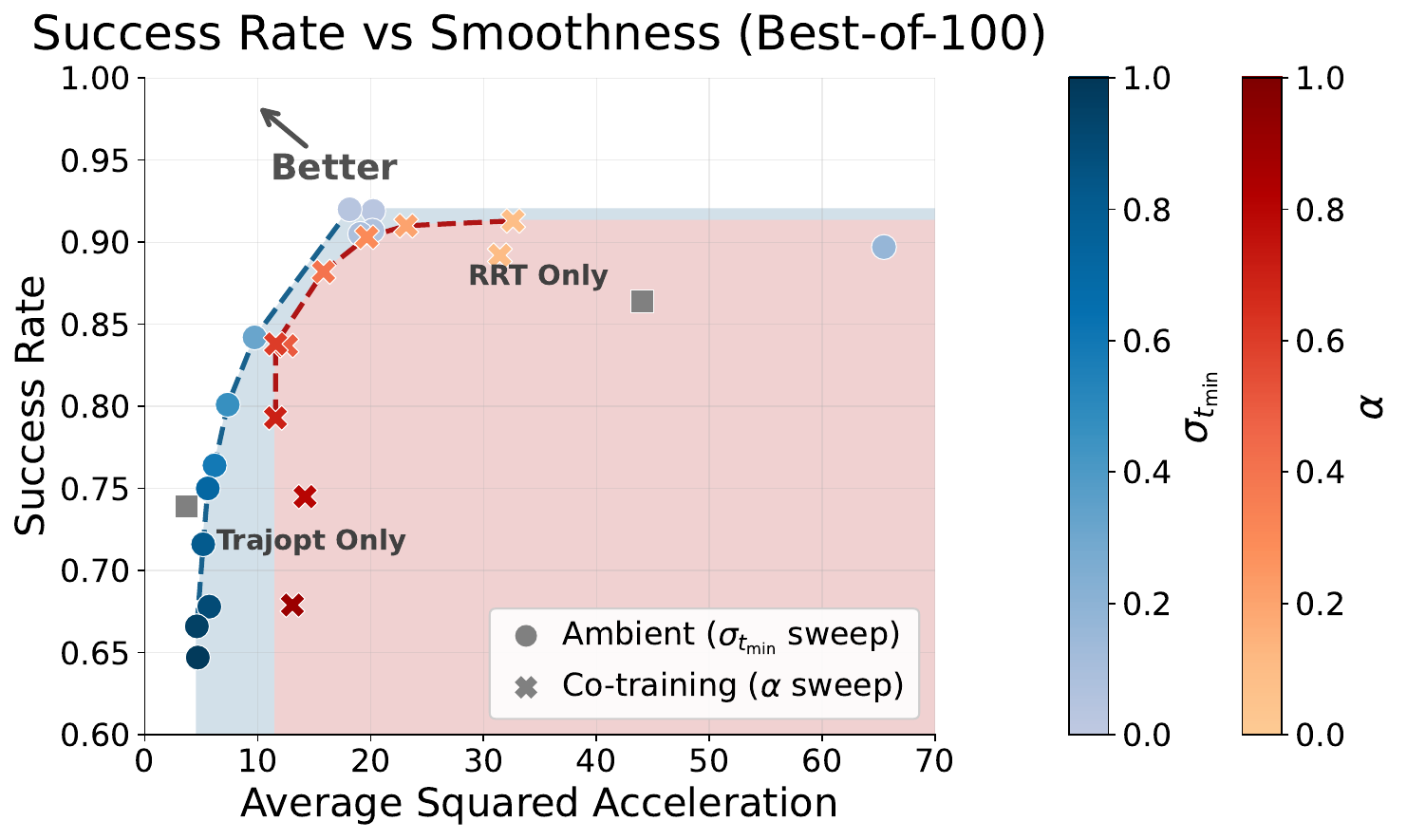}
        \captionof{figure}{\textbf{Figure~\ref{fig:nmp_pareto} repeated with best-of-100 evaluation.} The overall trends are the same, but the success rates are higher and the gap between the two Pareto frontiers is smaller.}
        \label{fig:nmp_pareto_best_of_100}
    \end{minipage}
\end{figure}

\begin{figure}[!t]
    \centering
    \begin{subfigure}[t]{0.49\linewidth}
        \centering
        \includegraphics[width=\linewidth]{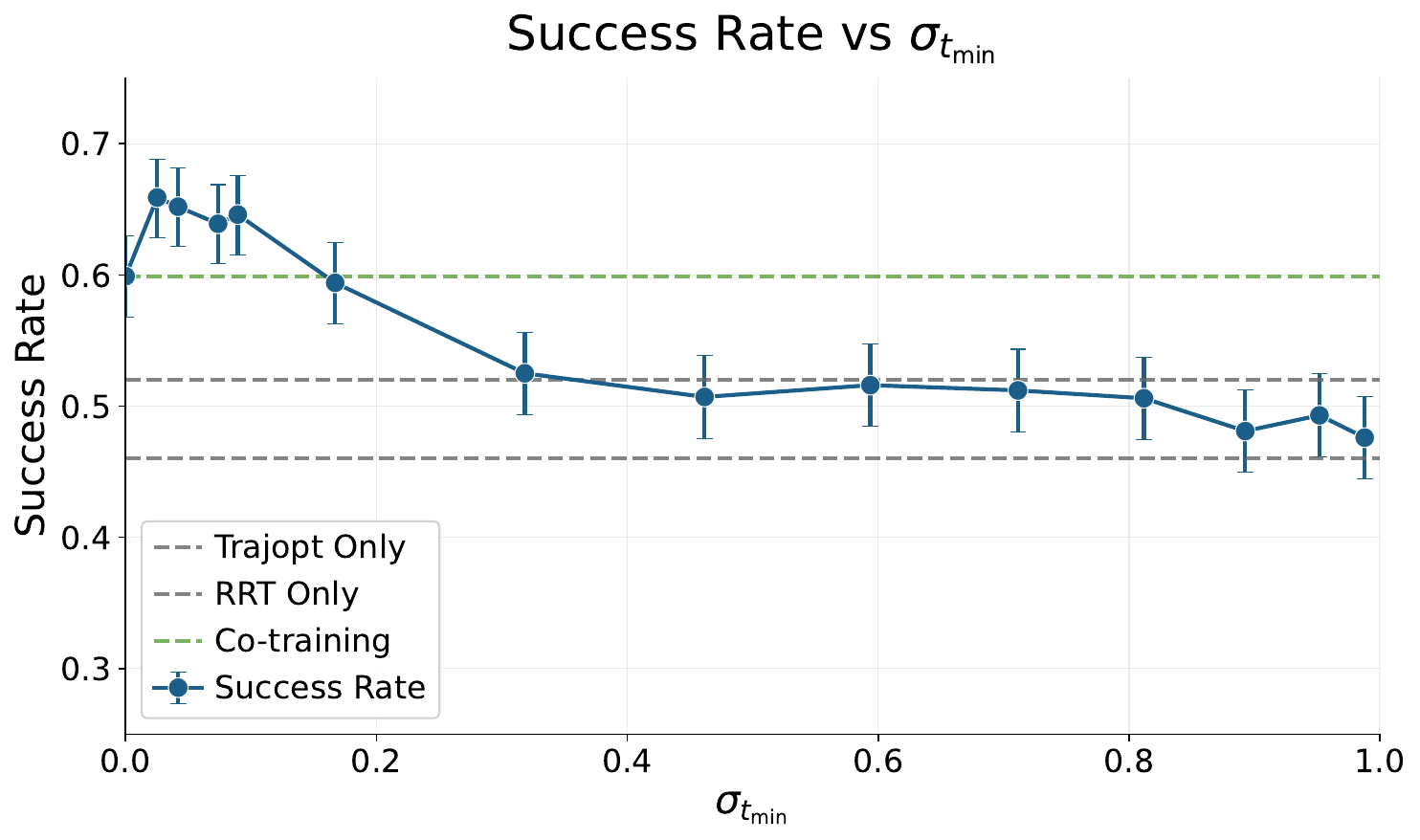}
        \caption{Success rate vs $\sigma_{t_{\min}}$.}
        \label{fig:nmp_success_vs_sigma_min}
    \end{subfigure}
    \hfill
    \begin{subfigure}[t]{0.49\linewidth}
        \centering
        \includegraphics[width=\linewidth]{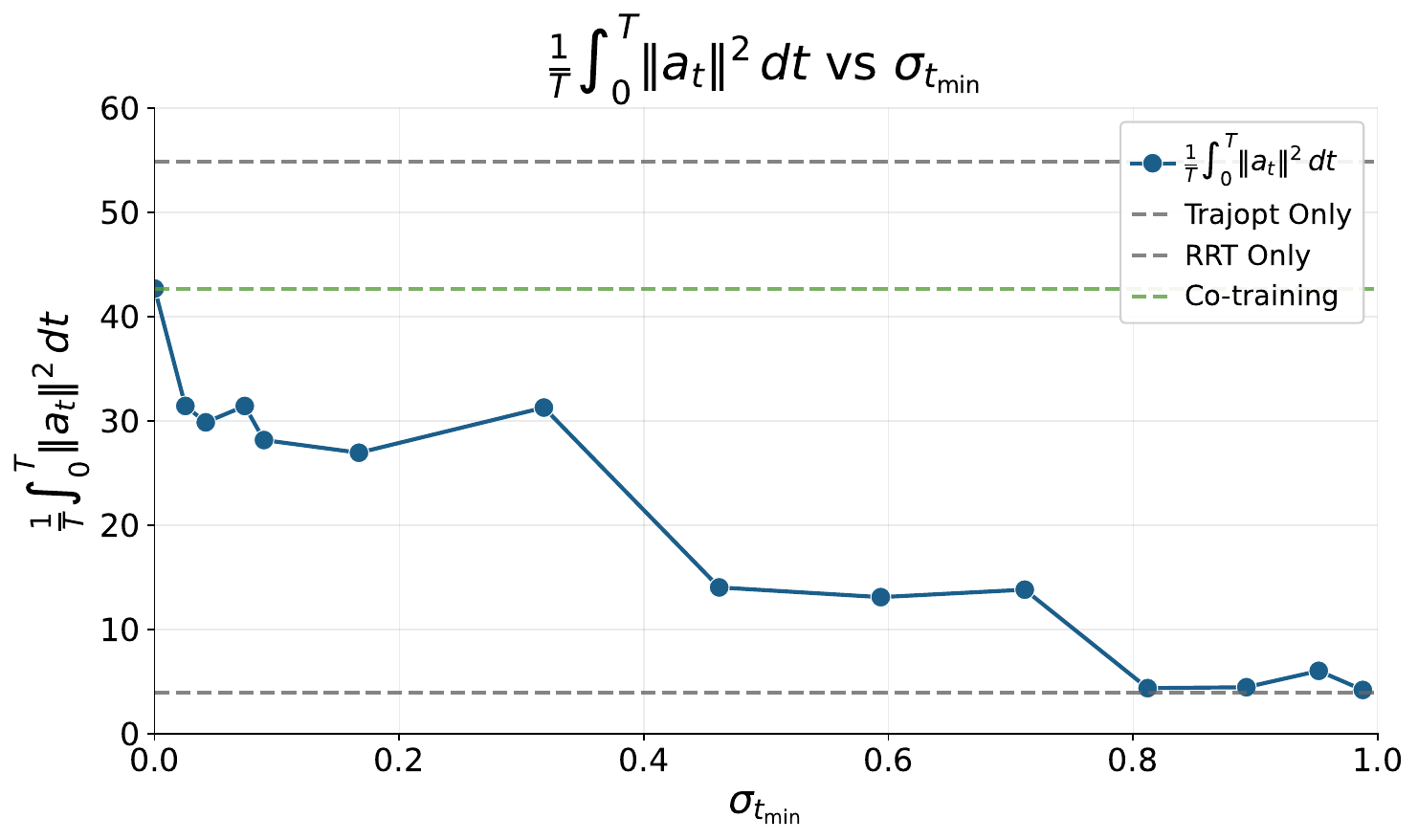}
        \caption{Smoothness metric vs $\sigma_{t_{\min}}$.}
        \label{fig:nmp_isa_vs_sigma_min}
    \end{subfigure}
    \caption{\textbf{Neural motion planning $\sigma_{t_{\min}}$ sweep.} \textbf{(\subref{fig:nmp_success_vs_sigma_min})} For small $\sigma_{t_{\min}}$, the policies are smoother, which results in fewer accidental collisions and increases success rate. For large $\sigma_{t_{\min}}$, the RRT data becomes under-utilized, which decreases the success rate. \textbf{(\subref{fig:nmp_isa_vs_sigma_min})} Increasing $\sigma_{t_{\min}}$ improves smoothness.}
    \label{fig:nmp_sigma_min_combined}
\end{figure}

Previous works in neural motion planning adopt a best-of-N approach when evaluating their planners \cite{dalal2024neuralmpgeneralistneural}. We present the results using $N=1$ (Table~\ref{tab:nmp_results} and Figure~\ref{fig:nmp_pareto}) and $N=100$ as in \cite{dalal2023imitatingtaskmotionplanning} (Table~\ref{tab:nmp_best_of_100} and Figure~\ref{fig:nmp_pareto_best_of_100}). Concretely, for each trial, we generate 100 plans and pick the one with the least collisions. The general trends for both choices of $N$ are the same, but the success rates are significantly higher. The Ambient policy continues to be smoother than the co-trained policy in both cases.

\subsection{Sim-and-Real Co-training: Planar Pushing}
\label{sec:appendix:sim_and_real}
\textbf{Data Generation}: We use the datasets from \citet{wei2025empirical}. $\mathcal D_p$ contains human teleoperated demonstrations in their ``target'' simulation environment and $\mathcal D_q$ contains automatically generated demonstrations~\cite{graesdal2024tightconvexrelaxationscontactrich} from their ``source'' simulation environment.

Table~\ref{tab:sim2target_gaps} shows the distribution shift between the source and target simulation datasets. For more details and visualizations of the environments, see~\citet{wei2025empirical}.

\begin{table}[!htbp]
\centering
\caption{\textbf{Sim-to-target gap along three axes}. The sim-to-target gap in the environments from \cite{wei2025empirical} was designed to mimic the sim-to-real gap. See \cite{wei2025empirical} for visualizations.}
\label{tab:sim2target_gaps}
\begin{tabular}{l cc}
    \hline
    \textbf{Gap} & \textbf{Simulation} & \textbf{Target Sim} \\
    \hline
    Visual  & \makecell{White spectrum light\\Different object colors\\Different camera pose\\No shadows} & \makecell{Warm spectrum light\\Different object colors\\Different camera pose\\Shadows} \\
    \hline
    Physics & Quasistatic & Drake physics \cite{tedrake2019drake} \\
    \hline
    Action  & Motion planning & Human teleop \\
    \hline
\end{tabular}
\end{table}

\textbf{Evaluation:} The slider is randomly initialized within the robot's workspace. A trial is successful if the robot returns to its default position and the slider’s pose is within 1.5cm and $3.5^o$ of
the goal. The time limit is 75s. The evaluation criteria exactly match \citet{wei2025empirical} to ensure a fair comparison.

\textbf{Ablations:} We use these experiments as a test-bed for evaluating key design decisions in Ambient Diffusion Policy. Figure~\ref{fig:planar_pushing_ablations} consolidates eight ablations; the results suggest the following:

\begin{itemize}[leftmargin=2em, topsep=1pt, itemsep=1pt, parsep=0pt]
    \item \textbf{Per-datapoint $t_{\min}$ (Figure~\ref{fig:planar_pushing_ablations}\subref{fig:planar_pushing_sigma_min_sweep}):} Labeling $\sigma_{t_{\min}}$ per dataset improves over co-training, and labeling $\sigma_{t_{\min}}$ per datapoint performs the best. The distribution of $t_{\min}$ assigned by the best classifier is shown in Figure~\ref{fig:sim_sigma_min_dist}.

    \item \textbf{$\sigma_{t_{\max}}$ sweep (Figure~\ref{fig:planar_pushing_ablations}\subref{fig:perf_vs_sigma_max}):} Using locality ($\sigma_{t_{\max}}>0$) does not appear to improve or hurt performance in sim-and-real co-training until $\sigma_{t_{\max}}$ is larger. Sim and real data do not share the same local action structure (e.g. due to differing contact dynamics); thus, policies should not train on sim data at low diffusion times.

    \item \textbf{Classifier robustness (Figures~\ref{fig:planar_pushing_ablations}\subref{fig:perf_vs_classifier} and~\ref{fig:planar_pushing_ablations}\subref{fig:perf_vs_tau}):} Policy performance is insensitive to the classifier checkpoint and threshold $\tau$; the classifier merely needs to be sufficiently trained.

    \item \textbf{Ambient loss buffer (Figure~\ref{fig:planar_pushing_ablations}\subref{fig:perf_vs_buffer}):} The Ambient loss is highly sensitive to the buffer size from Appendix~\ref{sec:appendix:ambient_loss}. In most experiments (Table~\ref{tab:experiments_overview}), we avoid this sensitivity by training with the regular denoising objective~\cite{ho2020denoisingdiffusionprobabilisticmodels}.

    \item \textbf{$t_{\min}$ granularity (Figure~\ref{fig:planar_pushing_ablations}\subref{fig:planar_pushing_partitions}):} Assigning $t_{\min}$ at intermediate granularities (one $t_{\min}$ per partition of $\mathcal D_q$) appears to be a promising direction. We split $\mathcal{D}_q$ into $N$ partitions and annotate one $\sigma_{t_{\min}}$ per partition. $N=1$ corresponds to assigning a single $t_{\min}$ to all of $\mathcal{D}_q$, and $N\rightarrow\infty$ corresponds to annotating a $t_{\min}$ per sample. Given a partition $P_i$, we annotate
    \[ t_{\min} = \inf \{t : \mathbb E_{(O, A_0)\sim P_i, A_t|A_0}[c_{\phi^*}(A_t, t)] > 0.5 - \tau \}. \]
    We compare three partitioning strategies: (1) randomly partition the action chunks (blue); (2) randomly partition the trajectories -- each partition contains the action chunks of its trajectories (red); (3) find $t_{\min}$ for each action chunk using the classifier and sort the chunks according to these values. Then form the partitions in increasing order (orange). For this task, Strategy (1) with $N=10{,}000$ performed best.

    \item \textbf{$|\mathcal D_q|$ scaling (Figures~\ref{fig:planar_pushing_ablations}\subref{fig:perf_vs_scale_10} and~\ref{fig:planar_pushing_ablations}\subref{fig:perf_vs_scale_50}):} As $\mathcal D_q$ scales, Ambient consistently outperforms co-training via re-weighting at both $|\mathcal D_p|=10$ and $|\mathcal D_p|=50$ (except at $|\mathcal D_q|=4000$ in the latter). The co-training baseline results (green) are from~\cite{wei2025empirical}.
\end{itemize}

\begin{figure}[!htbp]
    \centering
    \includegraphics[width=0.7\linewidth]{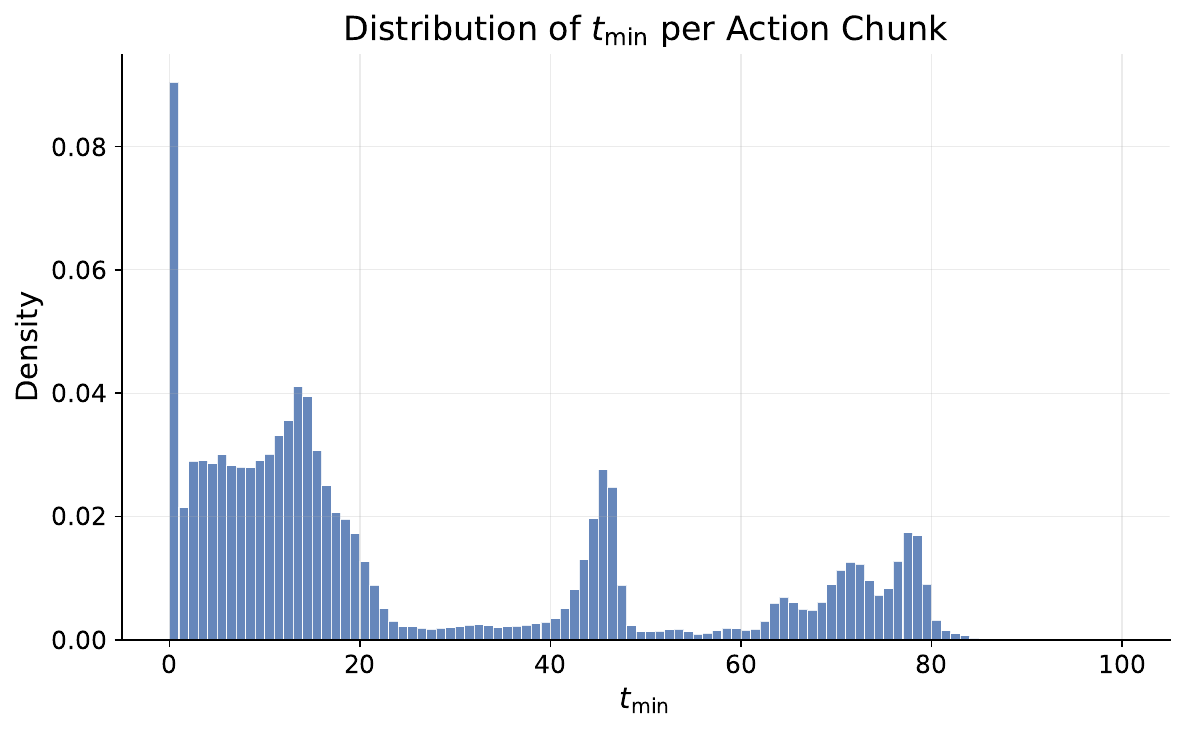}
    \caption{\textbf{Distribution of $t_{\min}$ labels.} This histogram shows the distribution of $t_{\min} \in [0, 100]$ assigned to each action chunk by the classifier for the sim-and-real experiments.}
    \label{fig:sim_sigma_min_dist}
\end{figure}

\begin{figure}[!htbp]
    \centering
    \begin{subfigure}[t]{0.49\linewidth}
        \centering
        \includegraphics[width=\linewidth]{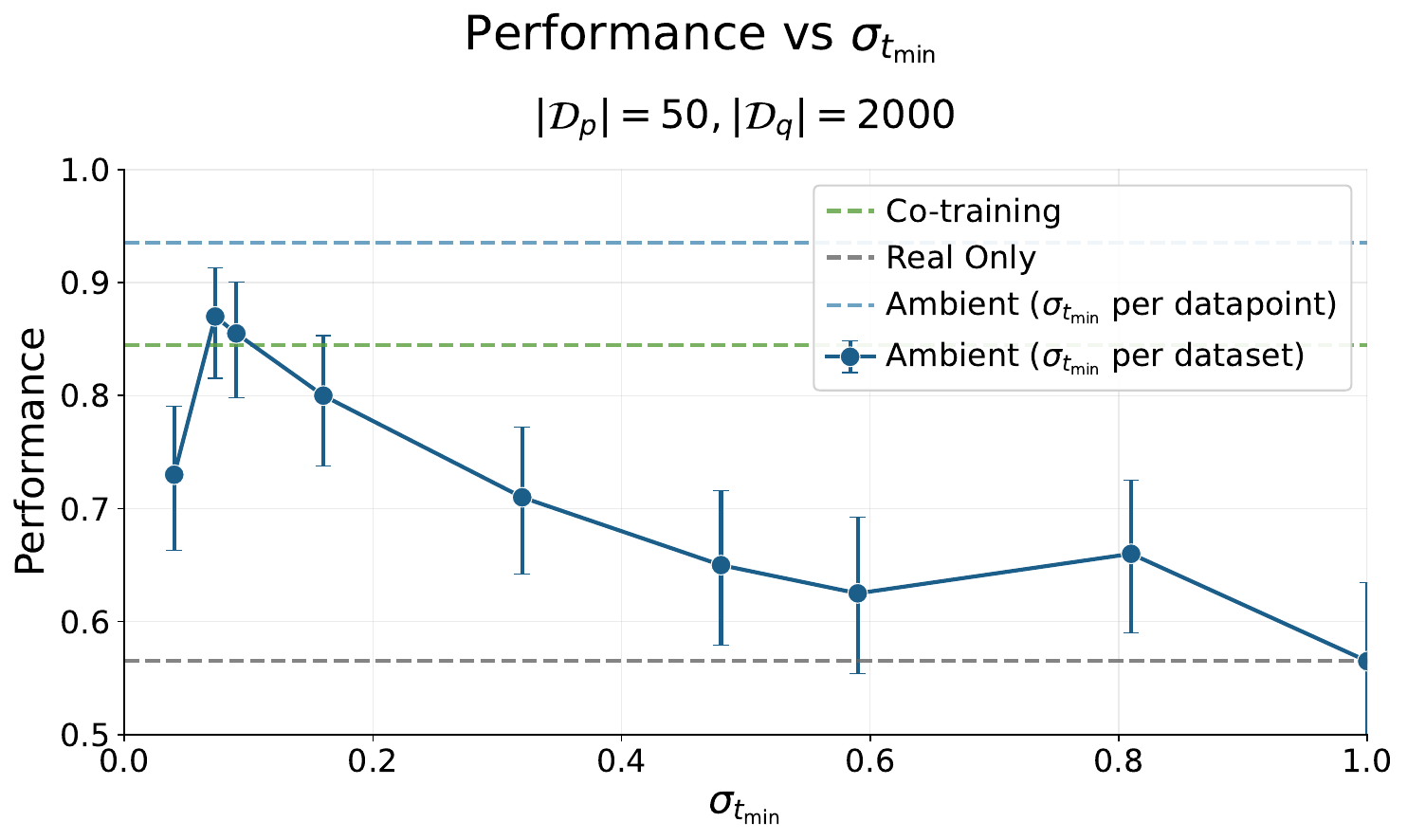}
        \caption{Performance vs $\sigma_{t_{\min}}$ (per dataset).}
        \label{fig:planar_pushing_sigma_min_sweep}
    \end{subfigure}
    \hfill
    \begin{subfigure}[t]{0.49\linewidth}
        \centering
        \includegraphics[width=\linewidth]{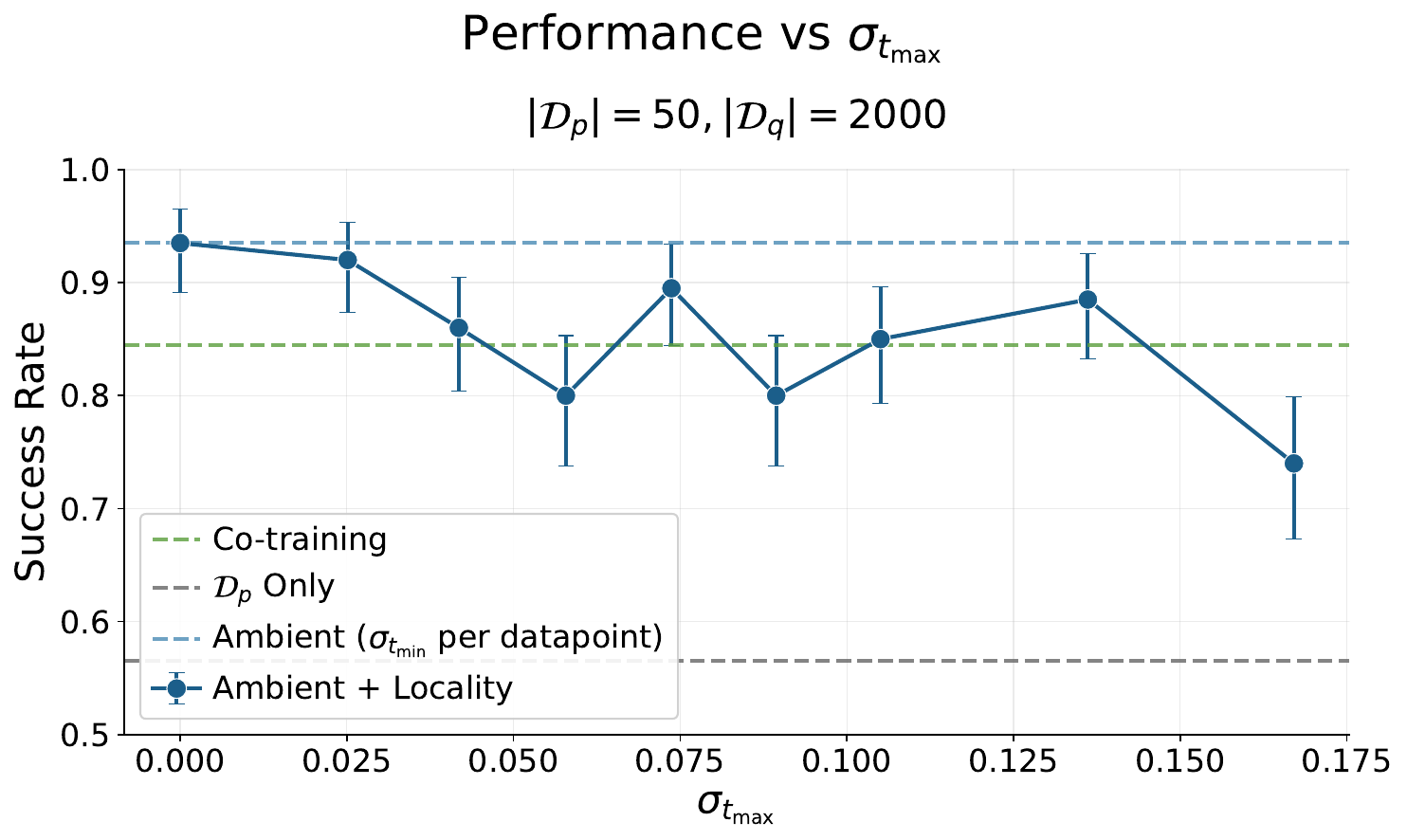}
        \caption{Performance vs $\sigma_{t_{\max}}$.}
        \label{fig:perf_vs_sigma_max}
    \end{subfigure}

    \vspace{0.5em}

    \begin{subfigure}[t]{0.49\linewidth}
        \centering
        \includegraphics[width=\linewidth]{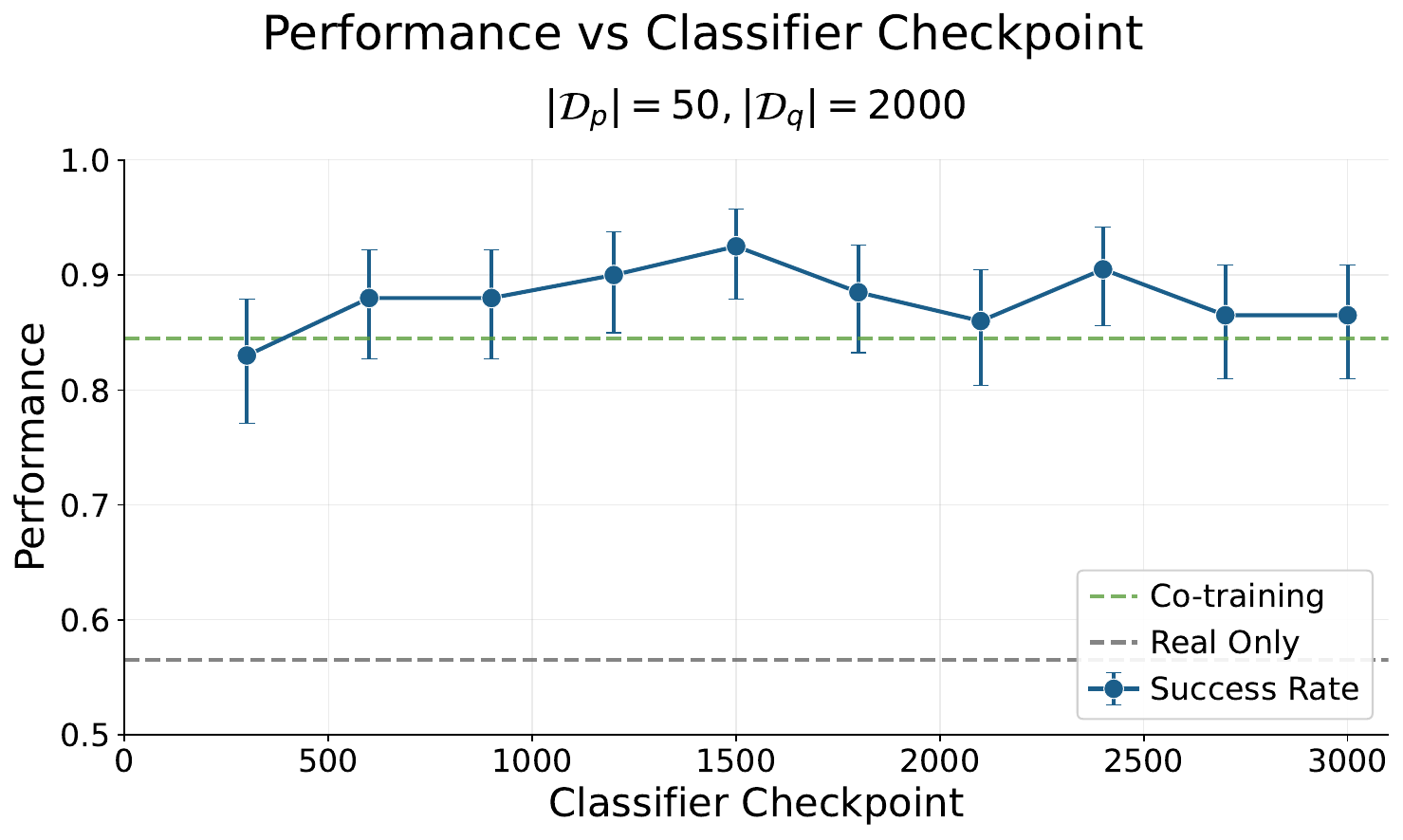}
        \caption{Sensitivity to classifier training duration.}
        \label{fig:perf_vs_classifier}
    \end{subfigure}
    \hfill
    \begin{subfigure}[t]{0.49\linewidth}
        \centering
        \includegraphics[width=\linewidth]{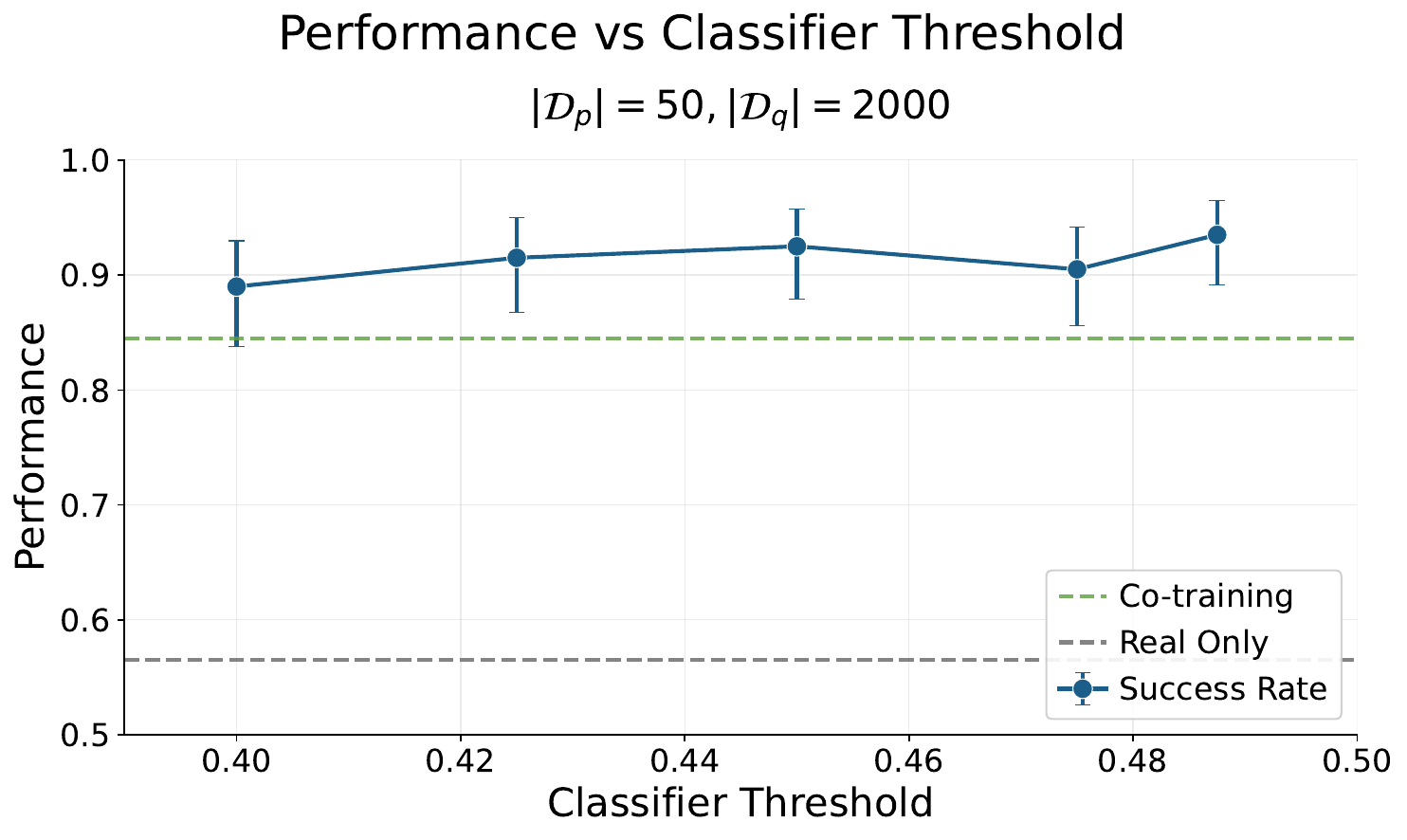}
        \caption{Sensitivity to classifier threshold $\tau$.}
        \label{fig:perf_vs_tau}
    \end{subfigure}

    \vspace{0.5em}

    \begin{subfigure}[t]{0.49\linewidth}
        \centering
        \includegraphics[width=\linewidth]{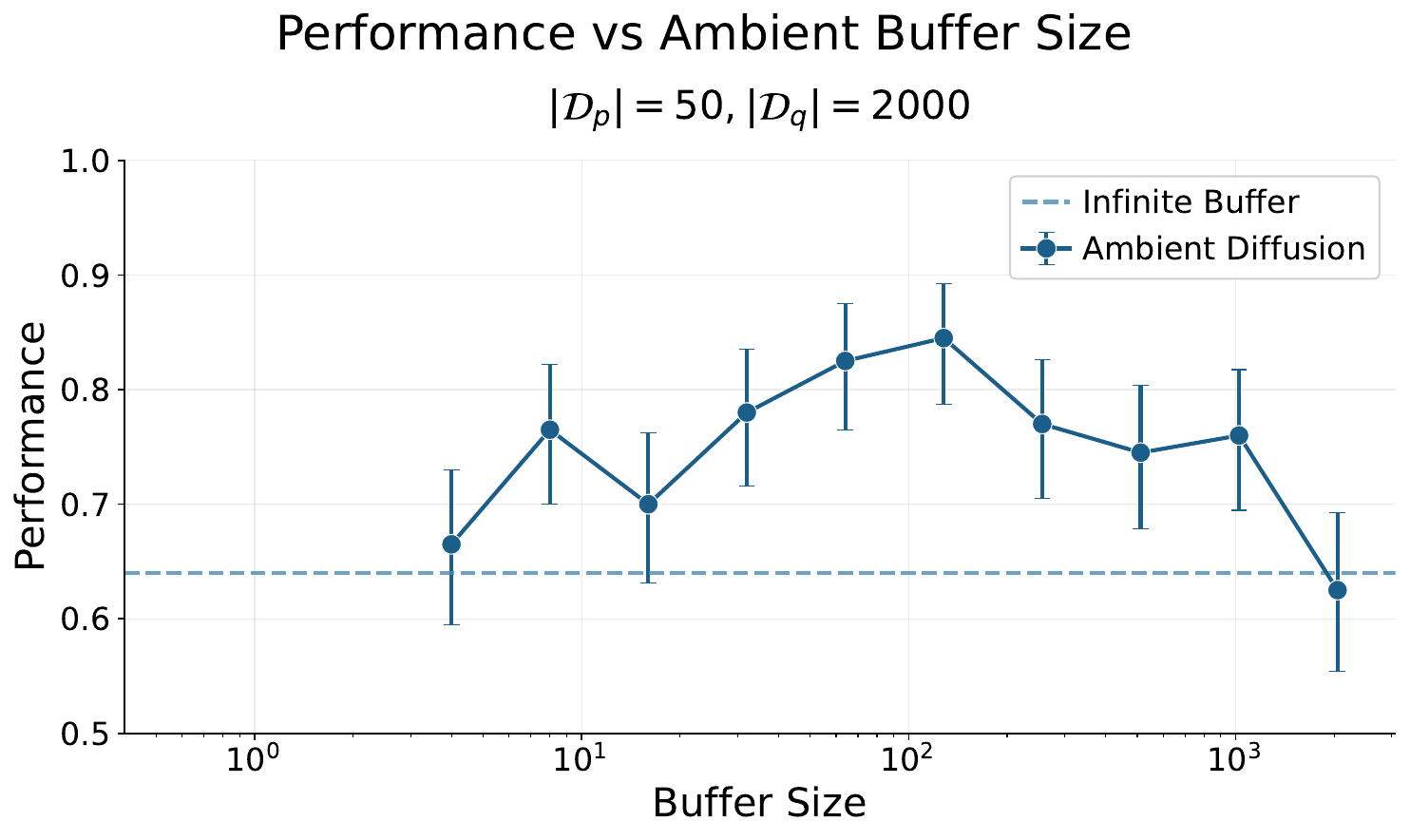}
        \caption{Sensitivity to the Ambient loss buffer (Appendix~\ref{sec:appendix:ambient_loss}).}
        \label{fig:perf_vs_buffer}
    \end{subfigure}
    \hfill
    \begin{subfigure}[t]{0.49\linewidth}
        \centering
        \includegraphics[width=\linewidth]{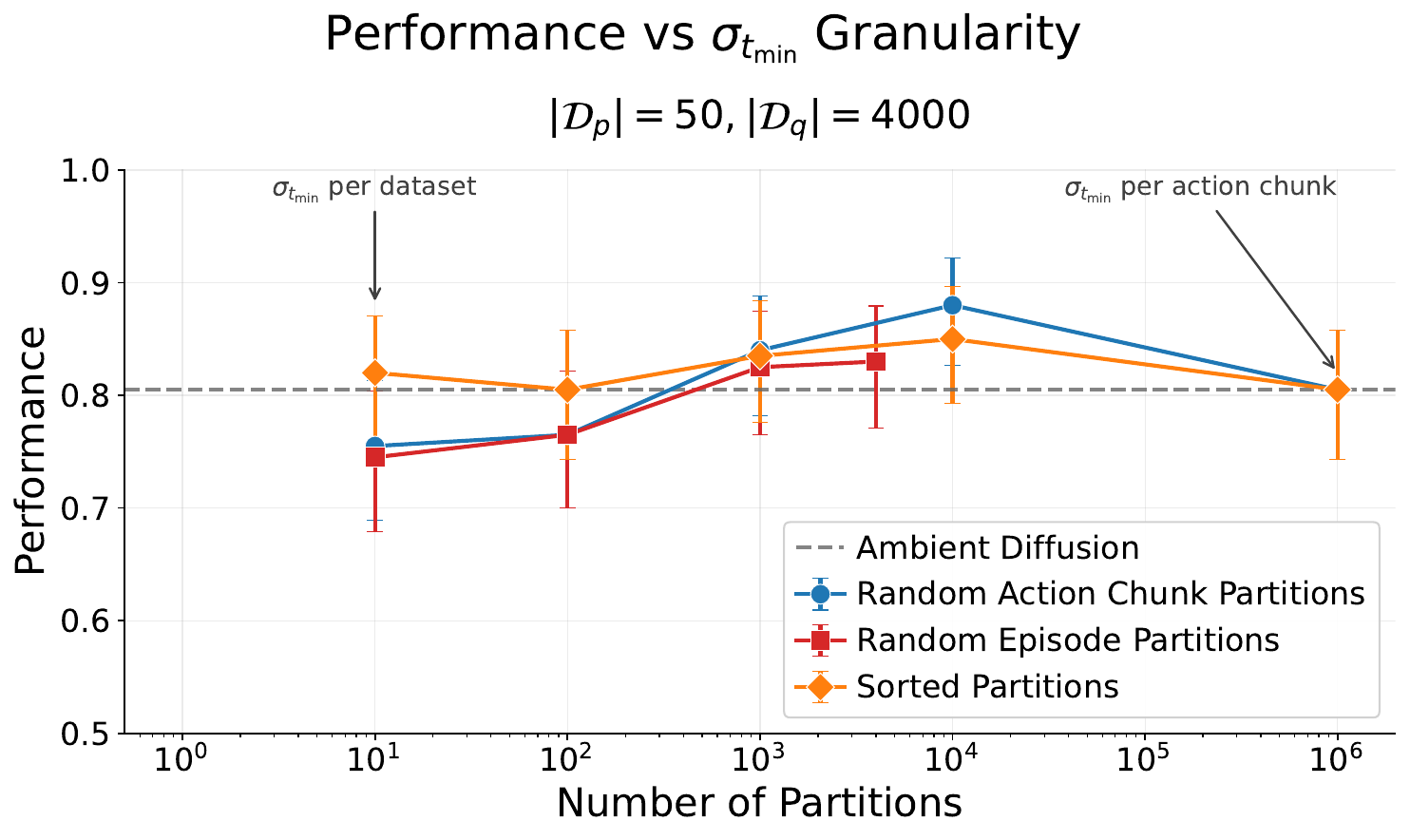}
        \caption{Performance vs granularity of $t_{\min}$ labels.}
        \label{fig:planar_pushing_partitions}
    \end{subfigure}

    \vspace{0.5em}

    \begin{subfigure}[t]{0.49\linewidth}
        \centering
        \includegraphics[width=\linewidth]{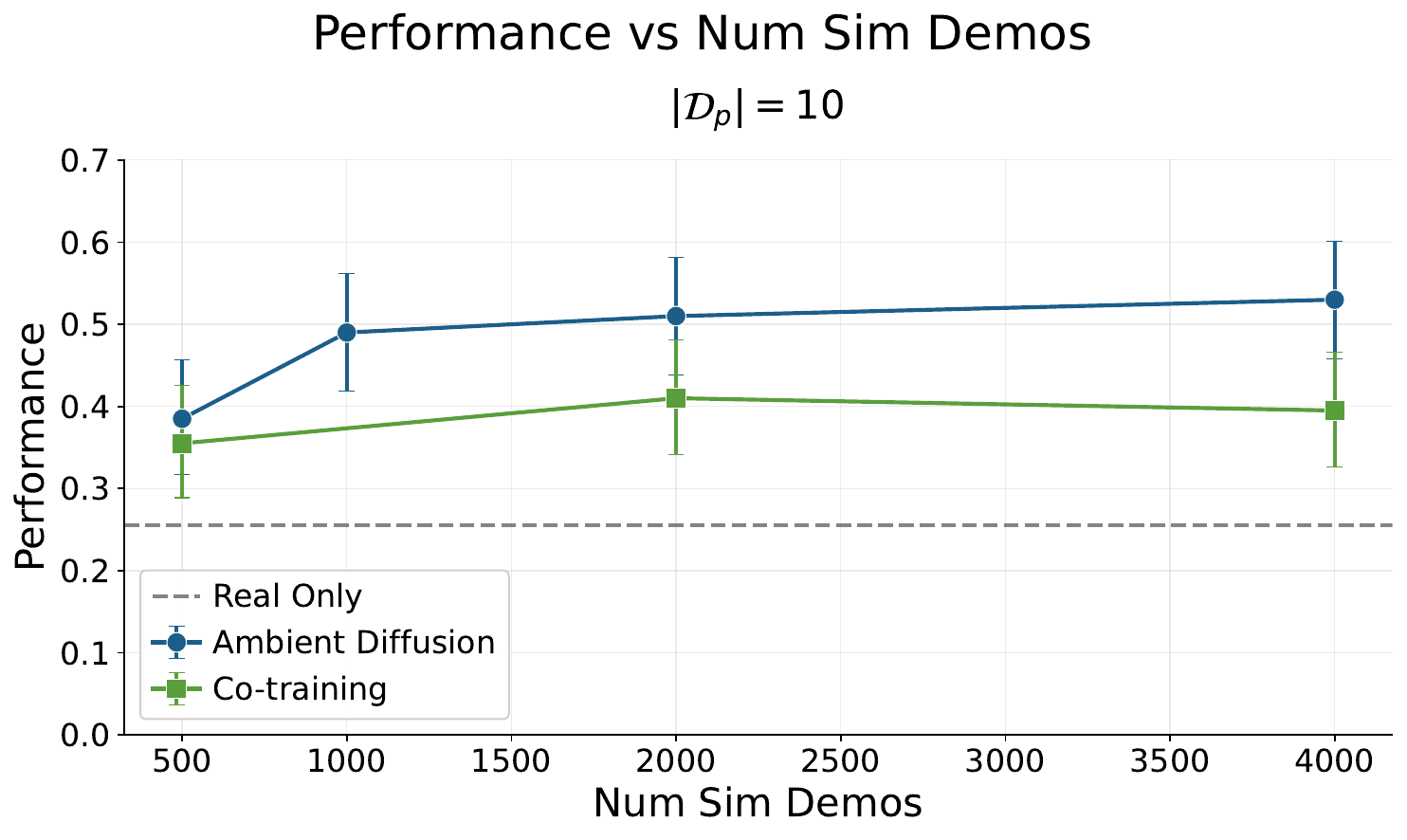}
        \caption{Performance vs $|\mathcal D_q|$ scale ($|\mathcal D_p|=10$).}
        \label{fig:perf_vs_scale_10}
    \end{subfigure}
    \hfill
    \begin{subfigure}[t]{0.49\linewidth}
        \centering
        \includegraphics[width=\linewidth]{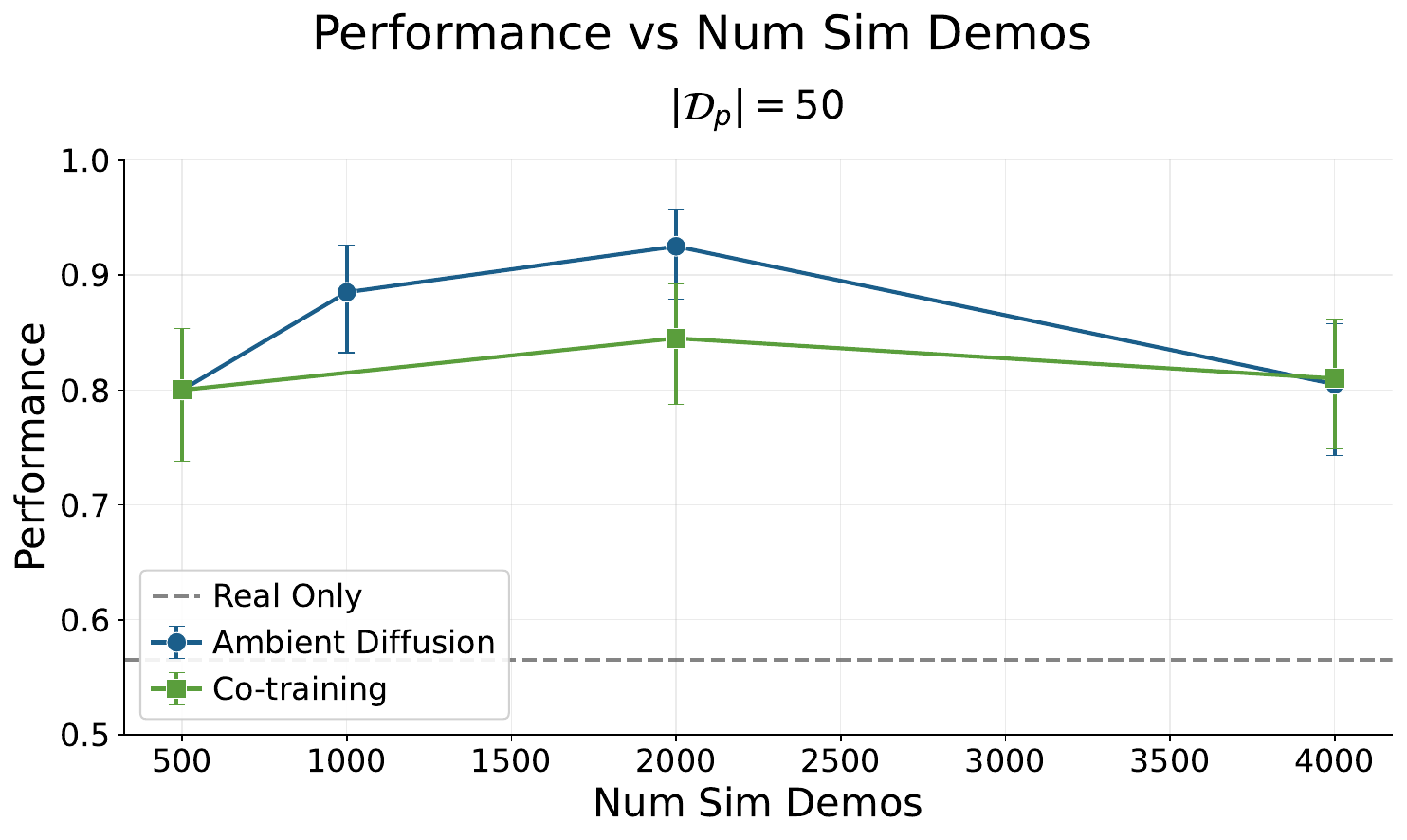}
        \caption{Performance vs $|\mathcal D_q|$ scale ($|\mathcal D_p|=50$).}
        \label{fig:perf_vs_scale_50}
    \end{subfigure}

    \caption{\textbf{Sim-and-real co-training ablations.} Eight ablations covering the key design decisions of Ambient Diffusion Policy. See Appendix~\ref{sec:appendix:sim_and_real} for the per-panel analysis.}
    \label{fig:planar_pushing_ablations}
\end{figure}

\subsection{Task Mismatch: Block Sorting}
\label{sec:appendix:block_sorting}
\textbf{Data Collection:} Both datasets were collected using VR teleoperation.

\textbf{Evaluation:} On each trial, 1-3 red blocks and 1-3 blue blocks are randomly initialized in a pile between the bins. The time limit is 20s per block. This allows trials with more blocks to have a longer time limit. Each policy is evaluated on 200 trials ($\sim$800 blocks).

Figure~\ref{fig:bin_sorting_combined} shows the performance of the Ambient Diffusion Policy on each metric during the $\sigma_{t_{\max}}$ sweep. Figures~\ref{fig:grasping_sigma_max_sweep} and~\ref{fig:time_per_block_sigma_max_sweep} show the average number of grasp attempts and time required per block of the policy's trained with different $\sigma_{t_{\max}}$ values. Both can be viewed as additional \textit{motion} metrics and improve (decrease) with higher $\sigma_{t_{\max}}$.

\textbf{Ablation 1 (Task-Conditioned):} A standard solution to handling the incorrect data is to add a one-hot encoding that indicates the sorting direction. This is analogous to task or quality conditioning \cite{pi07}. Even with task conditioning, the Ambient Diffusion Policy performs the best (Table~\ref{tab:bin_sorting_ablations}a). Note that Ambient Diffusion Policy and task-conditioning are complementary.

\textbf{Ablation 2 ($\mathcal{D}_q$ only at low noise):} Locality implies that the denoisers at $t < t_{\max}$ are insensitive to the global task structure. If so, we should be able to train an effective policy without using any ``correct'' data at low noise. We train a policy (without task-conditioning) that exclusively uses $\mathcal{D}_q$ for $t \in [0, t_{\max})$ and $\mathcal{D}_p$ for $t \in (t_{\max}, T]$. This policy still achieves a logic accuracy of $97.9\%$ and a motion score of $93.8\%$ (Table~\ref{tab:bin_sorting_ablations}b). This reinforces our claim that the backward process factorizes into a planning and a local refinement regime.

\begin{figure}[!htbp]
    \centering
    \begin{subfigure}[t]{0.49\linewidth}
        \centering
        \includegraphics[width=\linewidth]{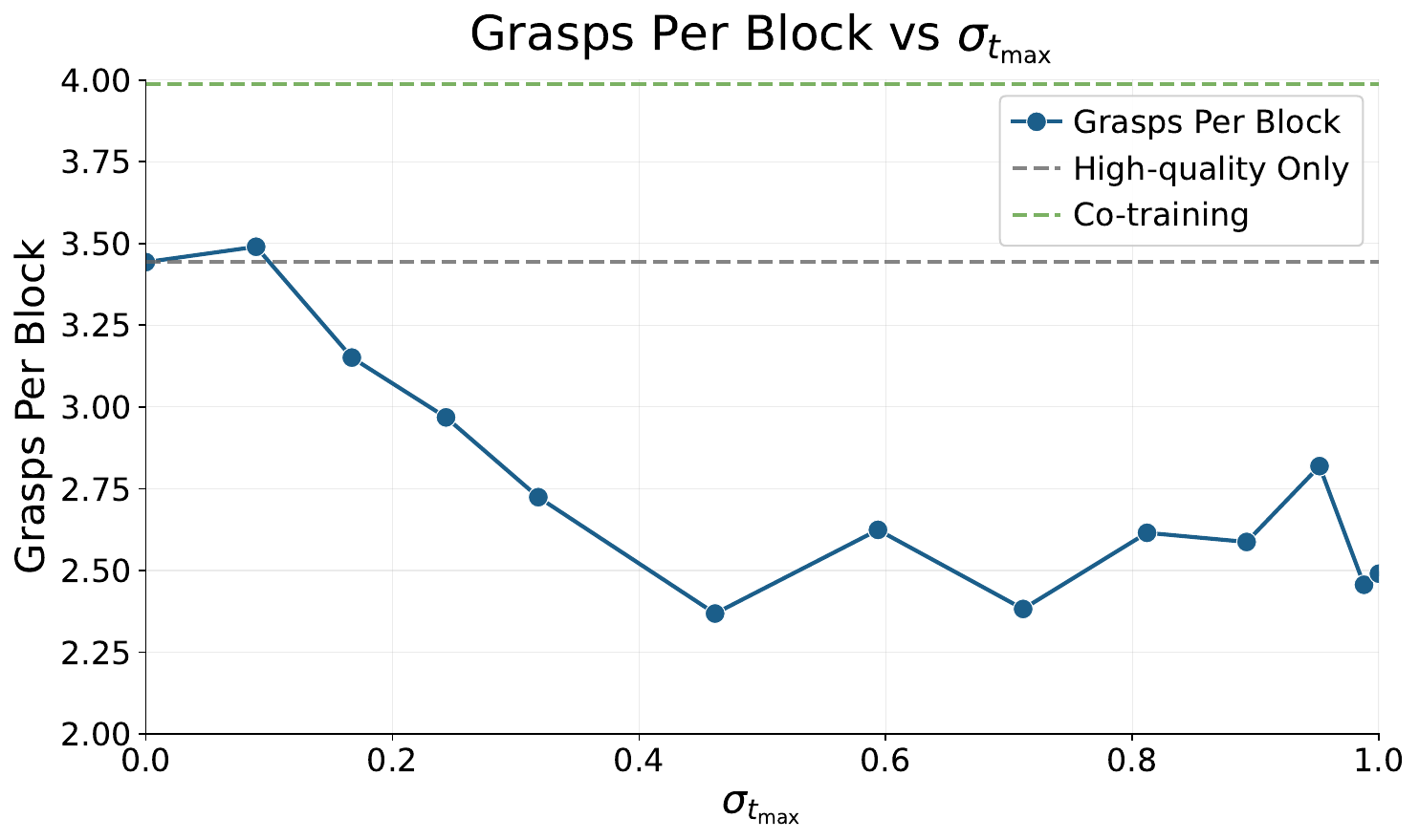}
        \caption{Grasp attempts per block vs $\sigma_{t_{\max}}$.}
        \label{fig:grasping_sigma_max_sweep}
    \end{subfigure}
    \hfill
    \begin{subfigure}[t]{0.49\linewidth}
        \centering
        \includegraphics[width=\linewidth]{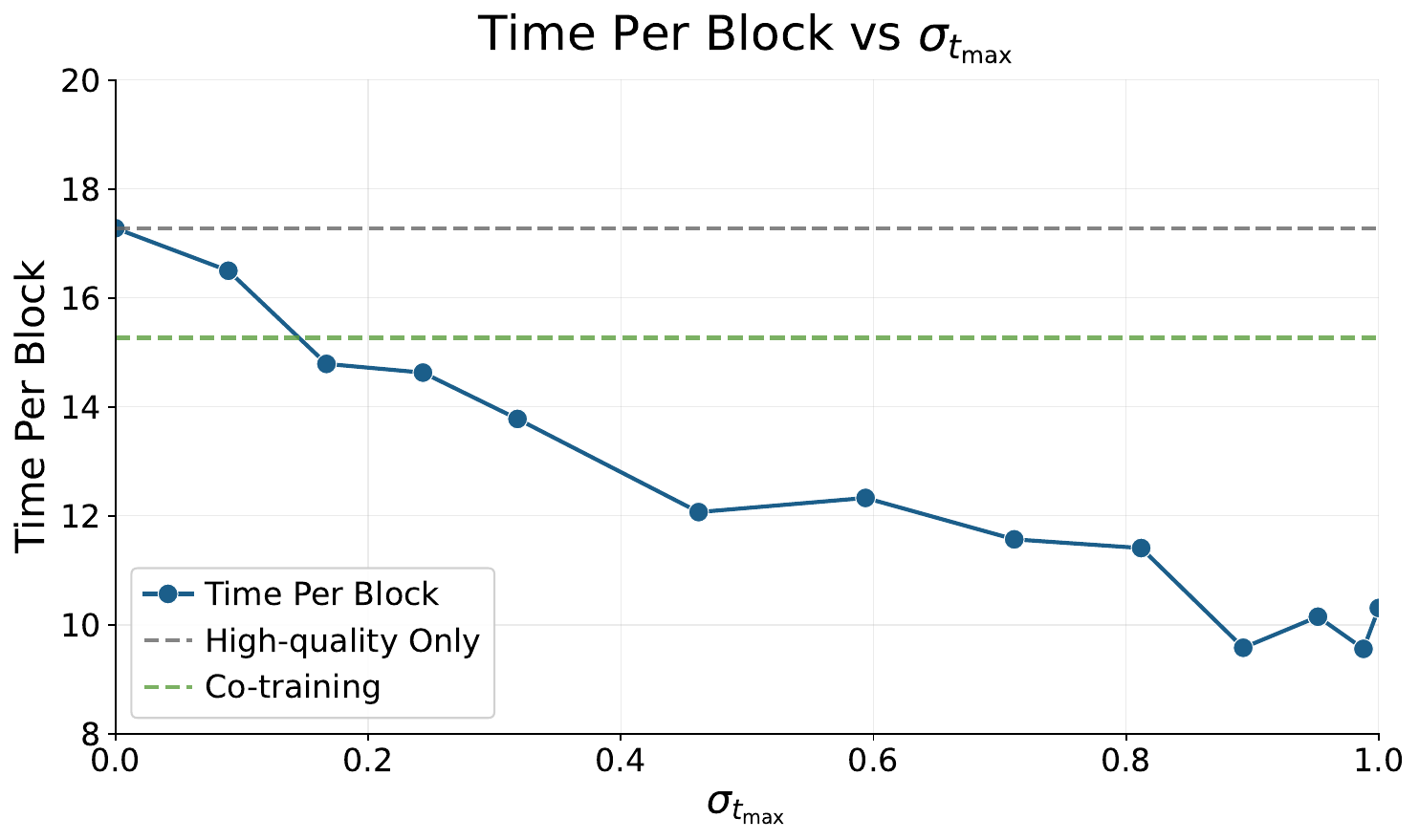}
        \caption{Time per block vs $\sigma_{t_{\max}}$.}
        \label{fig:time_per_block_sigma_max_sweep}
    \end{subfigure}
    \caption{\textbf{Additional motion-level metrics for the block sorting experiment (Section~\ref{sec:controlled_experiments:bin_sorting}).} \textbf{(\subref{fig:grasping_sigma_max_sweep})} Grasp attempts per block improves as $\sigma_{t_{\max}}$ increases but plateaus after $\sigma_{t_{\max}^*}=0.46$. \textbf{(\subref{fig:time_per_block_sigma_max_sweep})} Time per block improves as $\sigma_{t_{\max}}$ increases.}
    \label{fig:bin_sorting_motion_metrics}
\end{figure}

\begin{table}[!htbp]
    \caption{\textbf{Block sorting ablations (200 trials, $\sim$800 blocks).} \textbf{(a)} Ablation 1 (Task-Conditioned): Ambient outperforms co-training. \textbf{(b)} Ablation 2: the Ambient Diffusion Policy performs well despite training on no logically correct samples for $t \in [0,t_{\max}^*)$.}
    \centering
    \begin{tabular*}{\linewidth}{@{\extracolsep{\fill}}l|cc|c@{}}
        \toprule
        & \multicolumn{2}{c|}{\textbf{(a) Ablation 1: Task-Conditioned}} & \textbf{(b) Ablation 2} \\
        Training Algo.        & Co-train         & Locality                          & Locality \\
        (\# Blocks: $\sim$800) & ($\alpha^*=0.5$) & ($\sigma_{t_{\max}^*}=0.89$)      & \makecell[c]{$\mathcal{D}_q$ only below\\$\sigma_{t_{\max}^*}=0.24$} \\
        \midrule
        Logic $\uparrow$        & $\mathbf{98.6^{+0.7}_{-1.1}}$ \% & $98.5^{+0.7}_{-1.1}$ \%          & $97.9^{+0.9}_{-1.3}$ \% \\
        Motion $\uparrow$       & $91.5^{+1.8}_{-2.2}$ \%          & $\mathbf{94.2^{+1.5}_{-1.8}}$ \% & $93.8^{+1.6}_{-1.9}$ \% \\
        \midrule
        Success Rate $\uparrow$ & $90.3^{+2.0}_{-2.3}$ \%          & $\mathbf{92.8^{+1.7}_{-2.0}}$ \% & $91.9^{+1.8}_{-2.1}$ \% \\
        \bottomrule
    \end{tabular*}
    \label{tab:bin_sorting_ablations}
\end{table}

\section{Scaling Experiments: Open X-Embodiment}
\label{sec:appendix:scaling}

\subsection{OXE Experiments}
\label{sec:appendix:oxe}

\textbf{Dataset Curation:} We collected demonstrations for $\mathcal D_p$ with VR teleoperation. For each demonstration, 3-5 objects are placed on the table. These objects are selected among a set of 31 objects in the training distribution. The teleoperator opens the drawer, moves the objects into the drawer, and closes the drawer.

These experiments used two different subsets of the Open X-Embodiment (OXE) for $\mathcal D_q$: Magic Soup++ (MS++) and Custom OXE (COXE). MS++ is a collection of 27 datasets that was taken directly from the OpenVLA paper \cite{kim2024openvla}.

Custom OXE is a collection of 48 datasets. All 27 datasets from MS++ were included in COXE. Afterward, any dataset that met the following criteria was included:
\begin{itemize}
    \item Robot morphology: single arm or mobile manipulator
    \item Action space: EEF position or EEF velocity
    \item Has a known control frequency
    \item Observation space: has proprioception and contains at least one non-depth camera 
    \item Accessible: a few datasets that met the criteria above were not accessible, contained ambiguous features and action-space labels, or bugs. These datasets were excluded.
\end{itemize}

\textbf{Dataset Weights:} Weights for MS++ were taken from OpenVLA \cite{kim2024openvla}. For COXE, we reused weights for the MS++ subset. For the remaining datasets, we assigned weights that were proportional to their size except for the two largest datasets. For those, we halved their relative weight. Following the ``50/50 co-training'' in previous literature \cite{khazatsky2024droid}\cite{nasiriany2024robocasalargescalesimulationeveryday}, we normalize the total weight of $\mathcal D_q$ to be 0.5, and assign $\mathcal D_p$ a weight of 0.5. This is equivalent to assigning $\alpha=0.5$ in the co-training case. In the unweighted case, the weight of each dataset is proportional to its relative size. See Figure~\ref{fig:weighting_visualization} for a visualization.

\textbf{Annotating $t_{\min}$:} For each sub-dataset in COXE (e.g. DROID \cite{khazatsky2024droid}), we train a classifier to discern samples from $\mathcal D_p$ and the sub-dataset. We use this classifier to assign a single $t_{\min}$ value to all action chunks within this sub-dataset. This required us to train 48 classifiers for the 48 sub-datasets in COXE. Although in principle, we could have accomplished this using a single classifier $c_\phi(A_t, t, l)$, where $l\in \{0,1\}^{48}$ is a one-hot encoding indicating the source dataset of $A_t$. Figures~\ref{fig:weighting_visualization} and~\ref{fig:tower_building_weighting_visualization} show the distribution of $t_{\min}$ assigned to the sub-datasets of COXE for both the table cleaning and the tower building task.

\textbf{Annotating $t_{\max}$:} We simply set $t_{\max} = 10$ for the table cleaning and $t_{\max}=5$ for tower building. We did not perform a sweep. These values are unlikely to be optimal, especially since adding locality showed mixed results for table cleaning. Lowering $t_{\max}$ in the tower building experiments appears to help.

\citet{daras2025ambient} present a method to automatically label $t_{\max}$ using a similar classifier-based approach. Their approach assigned $t_{\max}=0$ to nearly all datasets in MS++ and COXE. A full hyperparameter sweep on real hardware would have been prohibitively time-consuming. Hence, we selected values that are small enough to be conservative, yet large enough for a meaningful evaluation of locality.

\begin{figure*}
    \centering
    \includegraphics[width=\linewidth]{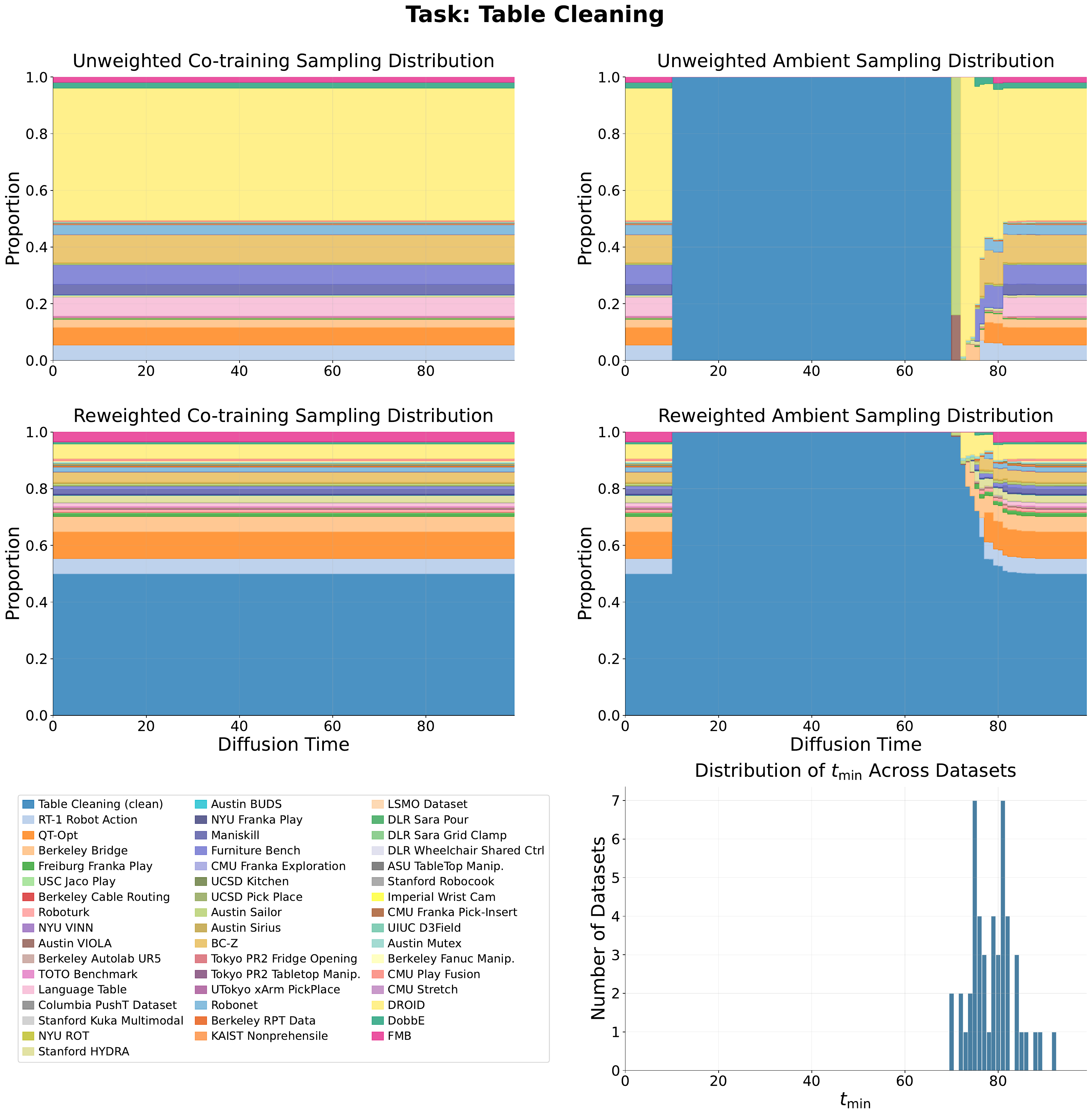}
    \caption{\textbf{Table Cleaning: Dataset re-weighting and $t_{\min}$ visualization.} These plots show the sampling ratio for $\mathcal D_p$ (shown in blue) and each dataset in COXE at different diffusion times for the table cleaning task. \textbf{Left:} Co-training uses the same sampling ratio at all diffusion times. \textbf{Right:} The sampling ratio for each dataset varies at each diffusion time. Notably, at intermediate diffusion times, only data from $\mathcal D_p$ can be used (the ``Donut paradox'' from \citet{daras2025ambient}). The histogram shows the distribution of $t_{\min}$ values assigned by the classifiers to each dataset in COXE.}
    \label{fig:weighting_visualization}
\end{figure*}

\begin{figure*}
    \centering
    \includegraphics[width=\linewidth]{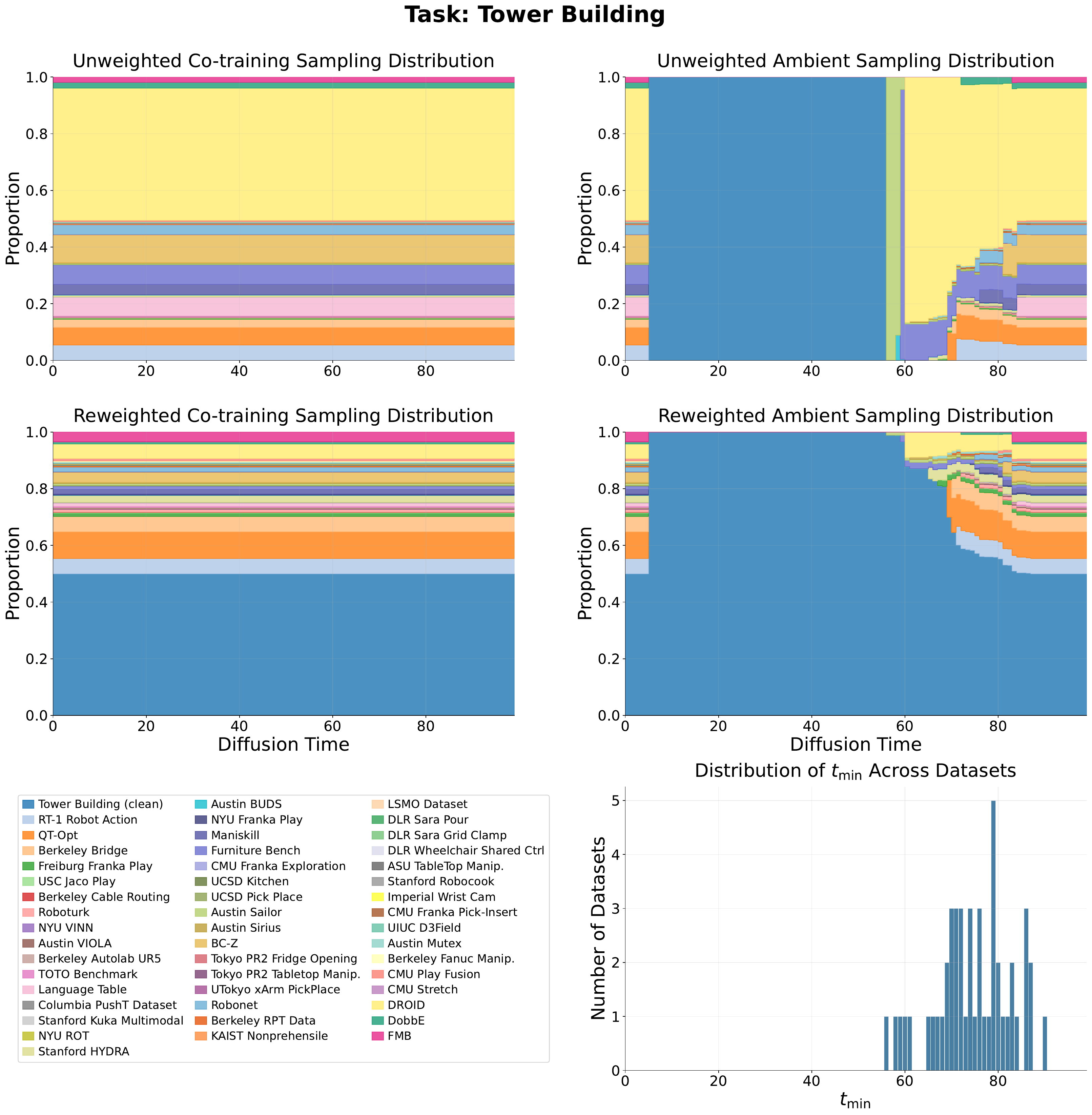}
    \caption{\textbf{Tower Building: Dataset re-weighting and $t_{\min}$ visualization.} See Figure~\ref{fig:weighting_visualization} for the equivalent plot descriptions.}
    \label{fig:tower_building_weighting_visualization}
\end{figure*}

\textbf{Table Cleaning Evaluation:} During each trial, 3 random unseen objects were placed on the table. The robot's task completion is measured according to the rubric in Table~\ref{tab:table_cleaning_rubric}. The time limit is 90s. Each policy was evaluated on 20 trials.

The results include a minor camera resolution bug that \textit{lowers} the policy performance across all table cleaning experiments. The image data was collected at 640 by 320 pixels and resized to 224 by 224 for training. Due to a bug at inference time, the image streams were 128 by 128 and upscaled to 224 by 224 for the policy; thus, the images during inference were effectively lower resolution than the images during training.

\begin{table}[]
\centering
\caption{\textbf{Table cleaning evaluation rubric.} The maximum score per trial is 1.0.}
\label{tab:table_cleaning_rubric}
\begin{tabular}{lc}
    \hline
    \textbf{Subtask} & \textbf{Points} \\
    \hline
    Open drawer & 0.1 \\
    Clean object (per object, 3 total) & $0.8/3 \approx 0.267$ \\
    Close drawer (after table cleaned) & 0.1 \\
    \hline
    \textbf{Max score per trial} & \textbf{1.0} \\
    \hline
\end{tabular}
\end{table}

\textbf{Tower Building Evaluation:} The performance of each trial is the number of correct blocks placed before the robot topples the tower or requires human intervention. Correct indicates that the blocks follow the correct alternating direction and color shown in Figure~\ref{fig:teaser}f. Each policy was evaluated on 20 trials.

The camera resolution bug from the table cleaning experiments was resolved for the tower building experiments.

\subsection{OXE Ablations}
\label{sec:appendix:oxe_ablations}

\subsubsection{Qualitative Performance}
The qualitative gap between the Ambient and co-trained policies is larger than the task completion metric would suggest. The 90s timeout artificially inflates the task completion rates of weak policies: they can fail repeatedly without penalty as long as they eventually succeed. To capture this, we recorded the number of grasp attempts and time required per object for each policy. Tables~\ref{tab:grasps_per_object} and~\ref{tab:time_per_object} show that the Ambient Diffusion Policies require 25-40\% fewer grasps and time per object than their co-trained counterparts.

\begin{table}[]
\centering
\caption{Average grasp attempts per object. The error bars on the \textbf{Average} row report the standard error of the mean across per-object values in each column ($n=10$ objects).}
\label{tab:grasps_per_object}
\begin{tabular}{>{\bfseries}l cccc}
    \toprule
    & \multicolumn{2}{c}{\textbf{MS++}} & \multicolumn{2}{c}{\textbf{COXE}} \\
    \cmidrule(lr){2-3} \cmidrule(lr){4-5}
    \textbf{Object} & Co-train & \makecell{Ambient +\\Locality} & Co-train & \makecell{Ambient +\\Locality} \\
    \midrule
    \rowcolor{rowgray}
    Apple              & 2 & 1 & 2 & 2 \\
    Figurine           & 4 & 2 & 2 & 2 \\
    \rowcolor{rowgray}
    Eraser             & 4 & 2 & 2 & 1 \\
    Play-Doh           & 2 & 2 & 2 & 2 \\
    \rowcolor{rowgray}
    Robot Gripper      & 3 & 2 & 2 & 2 \\
    Yarn               & 2 & 2 & 2 & 2 \\
    \rowcolor{rowgray}
    Camera Mount       & 1 & 1 & 1 & 2 \\
    3D Printed Object  & 2 & 2 & 3 & 1 \\
    \rowcolor{rowgray}
    Mustard            & 3 & 1 & 2 & 2 \\
    Garbage Bag Box    & 1 & 1 & 3 & 1 \\
    \midrule
    \textbf{Average}   & $2.4 \pm 0.3$ & $1.6 \pm 0.2$ & $2.1 \pm 0.2$ & $1.7 \pm 0.2$ \\
    \bottomrule
\end{tabular}
\end{table}

\begin{table}[]
\centering
\caption{Average cleaning time per object. The error bars on the \textbf{Average} row report the standard error of the mean across per-object values in each column ($n=10$ objects).}
\label{tab:time_per_object}
\begin{tabular}{>{\bfseries}l cccc}
    \toprule
    & \multicolumn{2}{c}{\textbf{MS++}} & \multicolumn{2}{c}{\textbf{COXE}} \\
    \cmidrule(lr){2-3} \cmidrule(lr){4-5}
    \textbf{Object} & Co-train & \makecell{Ambient +\\Locality} & Co-train & \makecell{Ambient +\\Locality} \\
    \midrule
    \rowcolor{rowgray}
    Apple              & 21s & 9s  & 30s & 19s \\
    Figurine           & 44s & 20s & 30s & 20s \\
    \rowcolor{rowgray}
    Eraser             & 39s & 16s & 18s & 8s  \\
    Play-Doh           & 20s & 20s & 19s & 20s \\
    \rowcolor{rowgray}
    Robot Gripper      & 31s & 19s & 19s & 20s \\
    Yarn               & 22s & 19s & 20s & 20s \\
    \rowcolor{rowgray}
    Camera Mount       & 12s & 8s  & 8s  & 20s \\
    3D Printed Object  & 19s & 20s & 29s & 8s  \\
    \rowcolor{rowgray}
    Mustard            & 33s & 9s  & 18s & 19s \\
    Garbage Bag Box    & 8s  & 8s  & 30s & 7s  \\
    \midrule
    \textbf{Average}   & $24.9 \pm 3.6$s & $14.8 \pm 1.7$s & $22.1 \pm 2.3$s & $16.1 \pm 1.8$s \\
    \bottomrule
\end{tabular}
\end{table}

The dataset re-weighting ablation is presented in the main body (Section~\ref{sec:oxe_experiments:ablations}).

\section{Finetuning}
\label{sec:appendix:finetuning}

\subsection{Implementation}
\label{sec:appendix:finetuning_implementation}
To finetune the base policies, we reduce the learning rate by 10x to 20x. Checkpoints are saved frequently since finetuning a model for too long can degrade performance. In the simulated experiments, we evaluate every checkpoint and report the performance of the best policy. In the real-world experiments, we perform a preliminary evaluation of each checkpoint with 5-10 trials, and perform a full set of trials on the best checkpoint.

The controlled finetuning comparison is presented in the main body (Section~\ref{sec:controlled_experiments:finetuning}).

\subsection{Table Cleaning (OXE)}
\label{sec:appendix:finetuning_oxe}
\begin{figure}
    \centering
    \includegraphics[width=\linewidth]{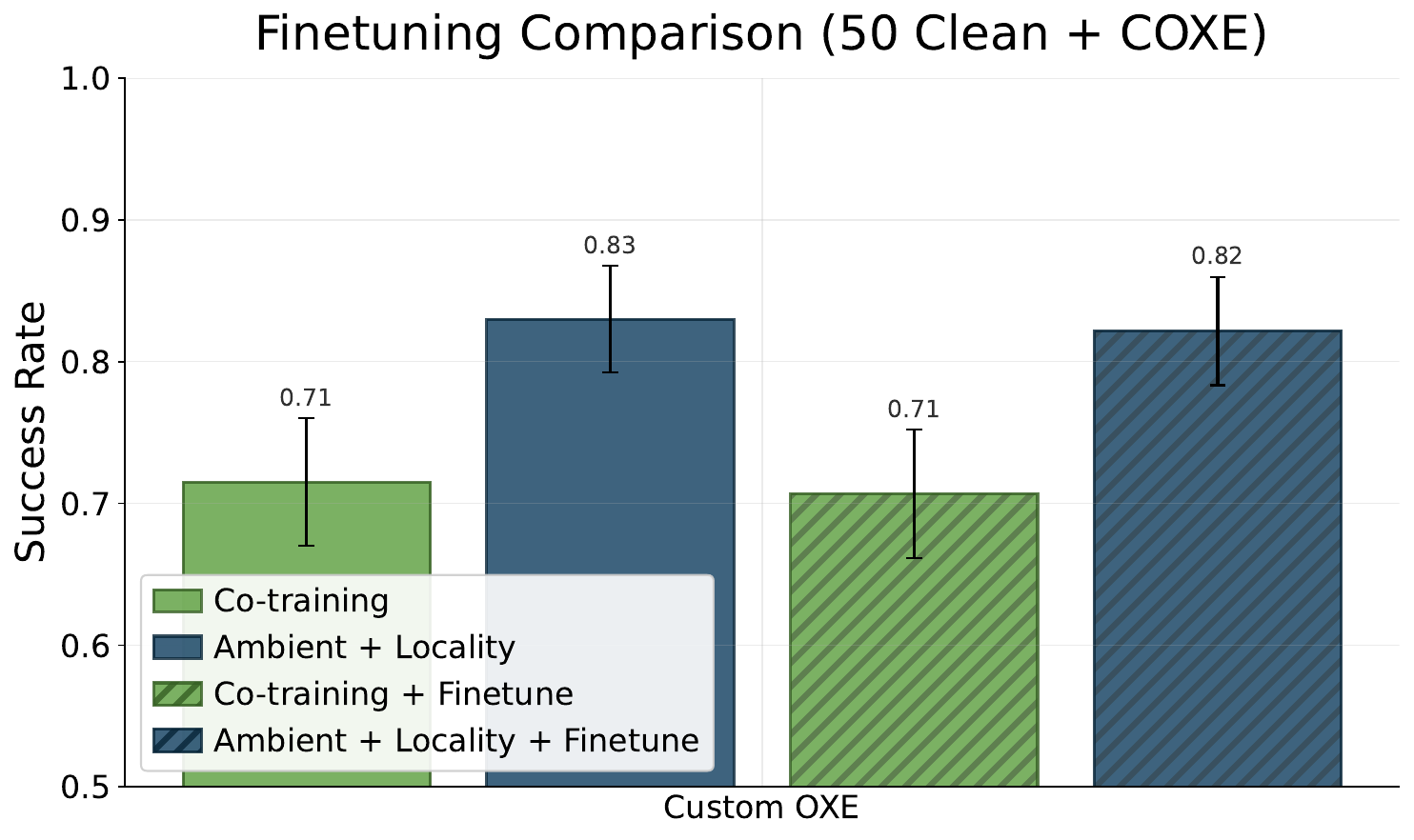}
    \caption{Finetuning resulted in no statistically significant change in policy performance on the table cleaning task.}
    \label{fig:table_cleaning_finetuning_results}
\end{figure}

Figure~\ref{fig:table_cleaning_finetuning_results} compares finetuning the best co-trained and Ambient OXE policies on $\mathcal D_p$ in the table cleaning task; neither produced a statistically significant gain. The results are consistent with the controlled experiments in Section~\ref{sec:controlled_experiments:finetuning}.

\section{Distribution Shifts in the Observation Space}
\label{sec:appendix:obs_shift}
We explore two algorithms for handling suboptimality in the observation space. Other approaches exist -- such as jointly denoising actions and observations. We note that ``suboptimality'' in the observation space can be ill-defined: an image from a different task is not necessarily suboptimal. In fact, it could increase the diversity of the dataset which is desirable. 

Although Ambient Diffusion Policy only addresses action-space distribution shift, its results are already strong. Whether observation-space shift presents a fundamental problem and possible solutions remain open directions.

\subsection{Noising the Observations}
Ambient Diffusion Policy contracts distributional differences in the actions by adding noise. We explore the same solution for the observations. During training, we add noise to the observation input (images, proprioception, etc) for half of the samples using the same noise schedule as the actions. The model takes in the noise level for both the actions and the observations. At inference time, the model performs inference with clean observations. This ablation is conceptually similar to concurrent work that noises observations \cite{hong2026tmrldiffusiontimestepmodulatedpretraining}; however, their goal was to expand the sampling distribution of a base policy for RL finetuning. Figure~\ref{fig:noisy_obs_results} presents the results; it did not appear to help for table cleaning.

\subsection{Classifier-Free Guidance}
In classifier-free guidance, the model learns to denoise the actions both unconditionally (i.e. without observations) and conditionally (with observations) \cite{ho2022classifierfreediffusionguidance}. We use the same training implementation as \cite{ho2022classifierfreediffusionguidance}. At sampling-time, we use the denoiser in \eqref{eq:cfg} where $w$ is the guidance weight. We explored classifier-free guidance as a potential solution since it trains an unconditional denoiser which is unaffected by corruptions in the observations.
\begin{equation}
    \tilde{h}_\theta(A_t, O, t) = (1+w)\cdot h_\theta(A_t, O, t) - w \cdot h_\theta(A_t, \varnothing, t)
\label{eq:cfg}
\end{equation}

Figure~\ref{fig:table_cleaning_cfg_sweep} shows that classifier-free guidance does not improve performance for all swept values of $w$.

\begin{figure}
    \centering
    \begin{subfigure}[t]{0.49\linewidth}
        \centering
        \includegraphics[width=\linewidth]{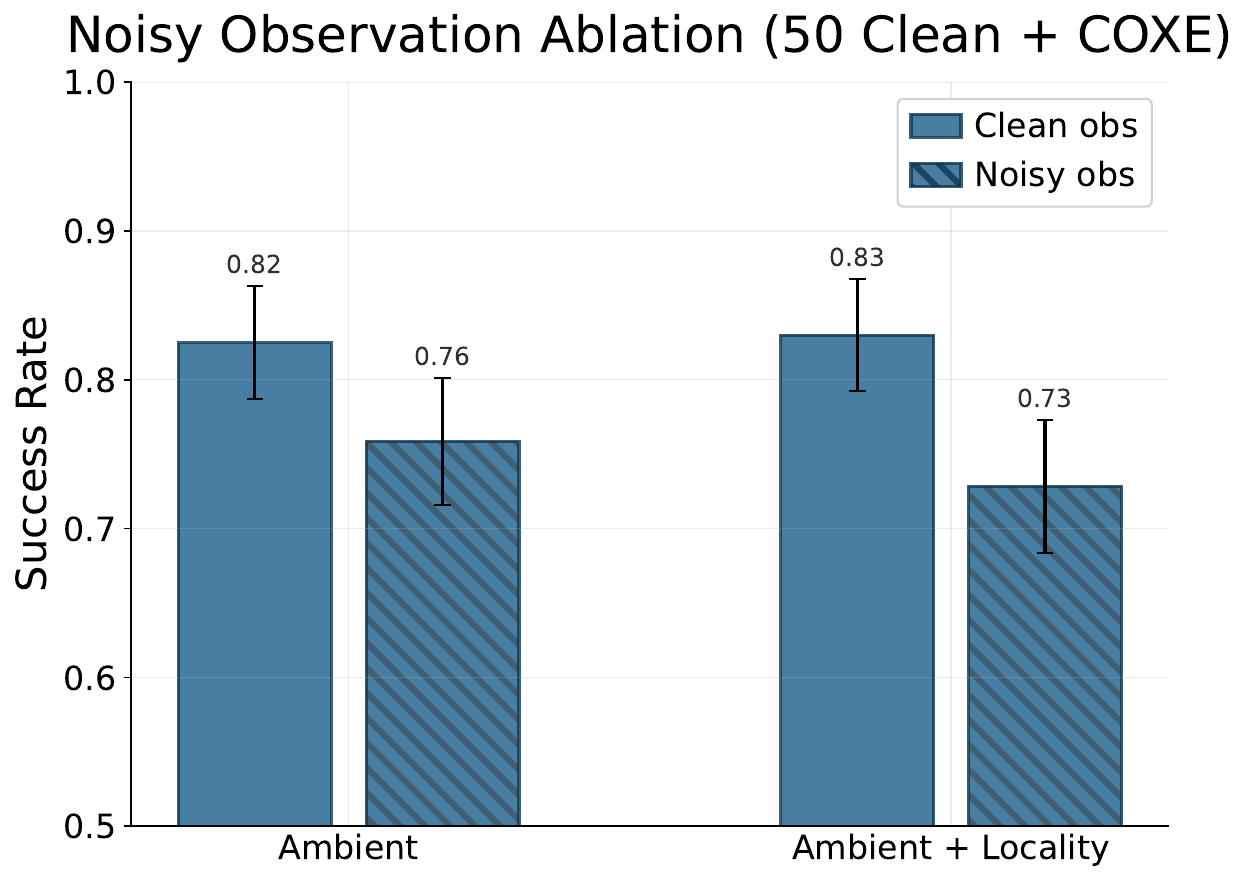}
        \caption{Noising the observations during training.}
        \label{fig:noisy_obs_results}
    \end{subfigure}
    \hfill
    \begin{subfigure}[t]{0.49\linewidth}
        \centering
        \includegraphics[width=\linewidth]{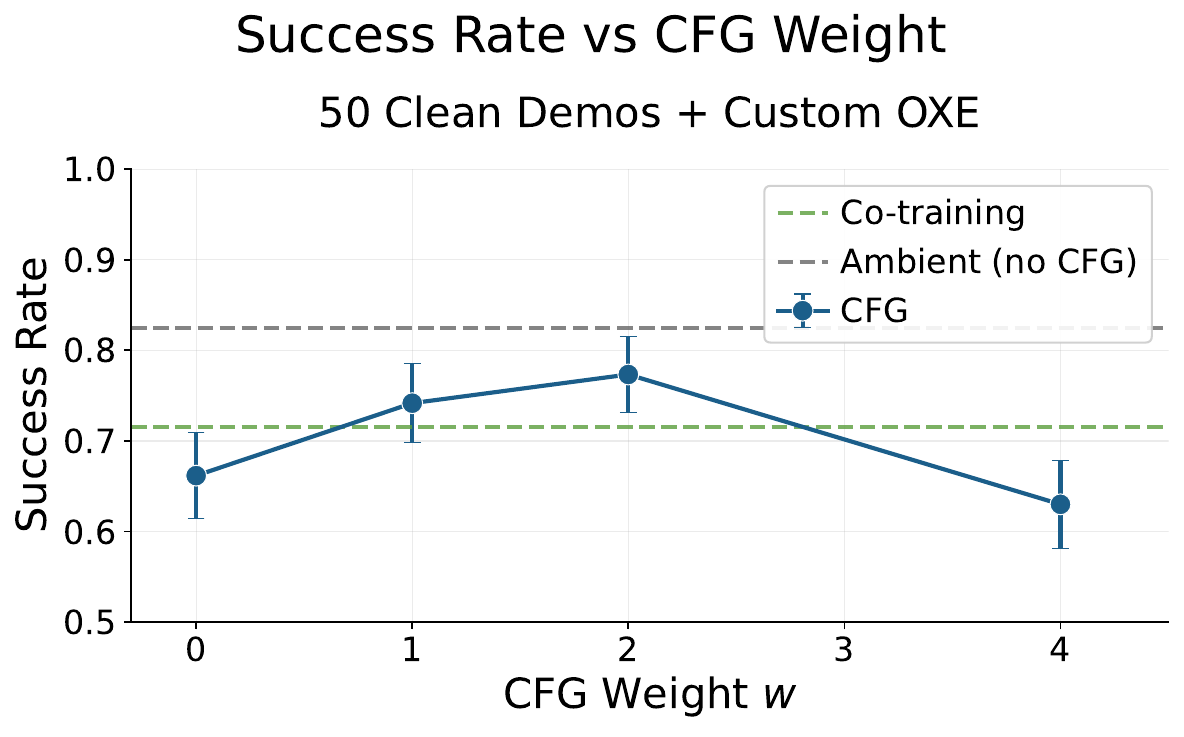}
        \caption{Classifier-free guidance vs.\ guidance weight $w$.}
        \label{fig:table_cleaning_cfg_sweep}
    \end{subfigure}
    \caption{\textbf{Algorithms for handling distribution shift in the observation space.} Neither approach produced stronger policies on the table cleaning task. \textbf{(\subref{fig:noisy_obs_results})} Adding Gaussian noise to the observations during training. \textbf{(\subref{fig:table_cleaning_cfg_sweep})} Classifier-free guidance~\cite{ho2022classifierfreediffusionguidance} swept over guidance weights $w$.}
    \label{fig:obs_shift_ablations}
\end{figure}

\end{document}